\documentclass[10pt,letterpaper]{article}

\usepackage{arxiv}
\usepackage{amsmath,amsfonts,bm}

\def\eqref#1{equation~\ref{#1}}
\def\1{\bm{1}}

\DeclareMathAlphabet{\mathsfit}{\encodingdefault}{\sfdefault}{m}{sl}
\SetMathAlphabet{\mathsfit}{bold}{\encodingdefault}{\sfdefault}{bx}{n}

\newcommand{\E}{\mathbb{E}}

\newcommand{\R}{\mathbb{R}}

\newcommand{\Cov}{\mathrm{Cov}}
\usepackage{url}
\usepackage{amsmath,amssymb,amsthm}

\title{%
  From Modes to Memories: Characterizing the\\
  Scale-Space Dynamics of Diffusion Models
}
\author{%
  Cristina L\'opez Amado\textsuperscript{*},
  Marco Fumero\textsuperscript{*},
  Francesco Locatello\\[0.6em]
  Institute of Science and Technology Austria (ISTA)
}

\correspondingemail{cristina.lopezamado@ist.ac.at, marco.fumero@ist.ac.at}

\usepackage[most]{tcolorbox}

\newtcolorbox{greybox}{enhanced,breakable,colback=blue!6,colframe=blue!6,boxrule=0pt,arc=2pt,left=6pt,right=6pt,top=2pt,bottom=4pt}

\tcbset{graystyle/.style={enhanced,breakable,colback=black!6,colframe=black!6,boxrule=0pt,arc=2pt,left=6pt,right=6pt,top=2pt,bottom=4pt,before skip=8pt,after skip=8pt}}
\newcommand{\graystatements}{%
  \tcolorboxenvironment{theorem}{graystyle}%
  \tcolorboxenvironment{proposition}{graystyle}%
  \tcolorboxenvironment{lemma}{graystyle}%
  \tcolorboxenvironment{corollary}{graystyle}}

\definecolor{darkgreen}{RGB}{0,100,0}

\usepackage{amsmath,amssymb,amsthm}
\usepackage{graphicx}
\usepackage{booktabs}
\usepackage{caption}
\usepackage{xcolor}
\usepackage{natbib}
\usepackage[colorlinks=true,linkcolor=blue,citecolor=blue,urlcolor=blue]{hyperref}
\usepackage{enumitem}
\usepackage{subcaption}
\usepackage{array}

\newcommand{\sigc}{\sigma_{\mathrm{c}}}
\newcommand{\sigemp}{\sigma_{\mathrm{c}}^{\mathrm{emp}}}
\newcommand{\reten}{\kappa}

\newcommand{\D}{D_\sigma}
\newcommand{\M}{M_\sigma}
\newcommand{\Mg}[1]{M_{\sigma}^{(#1)}}
\newcommand{\ps}{p_\sigma}

\theoremstyle{plain}
\newtheorem{theorem}{Theorem}
\newtheorem{proposition}[theorem]{Proposition}
\newtheorem{lemma}[theorem]{Lemma}
\newtheorem{corollary}[theorem]{Corollary}
\theoremstyle{definition}
\newtheorem{definition}[theorem]{Definition}

\theoremstyle{remark}

\usepackage{titlesec}
\titlespacing*{\paragraph}{0pt}{0ex}{0.5em}
\begin{document}

\maketitle
\begin{abstract}
Diffusion models are typically viewed as stochastic processes that transform noise into data. We take a complementary perspective: a diffusion model defines a family of deterministic dynamical systems indexed by noise scale. At each fixed scale $\sigma$, we treat the denoiser as a self-map and study its dynamics. For an exact denoiser, fixed points correspond to critical points of the smoothed data density, while attractors correspond to its modes; as $\sigma$ increases, sample-level modes merge into progressively coarser ones.
This suggests a geometric view of memorization: examples that receive excess probability mass due to duplication or overfitting, as well as outliers, should remain distinguishable under stronger smoothing than ordinary examples. We quantify this persistence by the critical scale $\sigc$, the largest noise scale at which an example is retained by the fixed-scale dynamics. In conditional models, the same construction extends naturally to image--caption pairs.
Experiments in controlled settings and on large-scale models show that $\sigc$ tracks memorization arising from duplication, overfitting, and outliers, and identifies both memorized and partially memorized examples in Stable Diffusion. Moreover, $\sigc$ yields interpretable measures of the image spatial distribution and caption dependence of memorization.
\end{abstract}

\section{Introduction}
\label{sec:intro}

Diffusion models are typically viewed as stochastic processes that transform noise into data
\citep{ho2020ddpm,song2021sde,karras2022edm,rombach2022ldm}. They are trained with a local and static
criterion, the prediction of a clean signal from a noisy one at each noise level, and are used through
a global and dynamic procedure, a sampler that composes these denoisers along a decreasing schedule of
noise levels. What a model has learned is then read off the process as a whole, from its samples or
from its loss along the schedule.

We take a complementary view and interpret a diffusion model as a family of denoisers $\D$, one for each noise level
$\sigma$. Fixing $\sigma$ and iterating the denoiser on its own output, $x_{k+1}=\D(x_k)$, turns a
single trained network into a one-parameter \emph{family of deterministic dynamical systems}, whose fixed
points and basins are properties of the network alone (Figure~\ref{fig:teaser}, left). For an exact denoiser each of these maps is
Gaussian mean shift on the smoothed data density $\ps$%=p*\mathcal N(0,\sigma^{2}I)$
 \citep{fukunaga1975,cheng1995,comaniciu2002}: its \emph{fixed points} are the critical points of $\ps$ and its
\emph{attractors} are the modes of $\ps$. At small $\sigma$ every training example is its own attractor, at
large $\sigma$ a single attractor remains, and in between the attractors merge into progressively
coarser ones (Figures~\ref{fig:portrait}--\ref{fig:scalespace}). Varying $\sigma$ thus traces a scale
space of the attractors of the model.

This scale space gives a geometric view of memorization mechanisms \citep{carlini2023extracting,bonnaire2025why,ross2025geometric}: an example is retained under its dynamics when it carries excess probability mass, as under duplication or overfitting, or when it is isolated from the remaining data, as outliers. We measure this persistence by the \emph{critical scale}
$\sigc$, the largest noise level at which iterating the fixed-scale denoiser from an example keeps the trajectory close to it (Figure~\ref{fig:teaser}). The critical scale requires no labels, applies to conditional
models, where it becomes a property of a caption-image pair, and is interpretable. We test it with interventions that cause
memorization: on \texttt{CIFAR-10} we induce memorization
by duplication, by overfitting and by planting outliers, and on Stable Diffusion we use image--caption
pairs that are known to be memorized, fully or in part.
Our contributions are as follows:
\begin{figure}[t]
\centering
\includegraphics[width=0.85\linewidth,trim={0 0.1cm 0 0.1cm},clip]{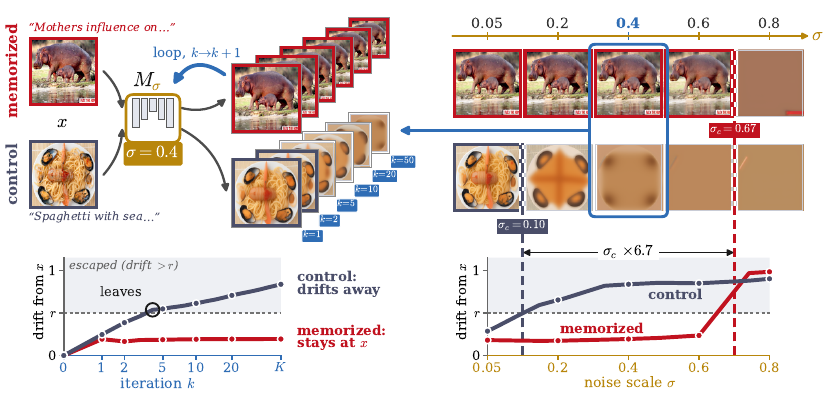}
\caption{\textbf{A diffusion model as a family of dynamical systems.} Fixing the noise level $\sigma$ and
feeding the denoiser its own output defines a self map $\M$ (\emph{left}). A memorized training image is a \emph{fixed point}
of this map: its orbit stays at the image, whereas the orbit of a control image drifts away and 
crosses the escape radius $r$ (\emph{bottom left}). Repeating this retention test at every $\sigma$ (\emph{right}) yields the critical scale
$\sigc$, the largest noise level at which the image is still returned; it is $6.7\times$ larger for the
memorized image, effectively detecting it.}
\label{fig:teaser}
\end{figure}
\begin{itemize}[leftmargin=*,itemsep=1pt,topsep=1pt,after=\vspace{-0.1cm}]

\item We show that a diffusion model can be interpreted as a \emph{family of dynamical systems}
indexed by the noise level: for an exact denoiser the attractors are the modes of the data density smoothed at scale $\sigma$, and varying $\sigma$ traces a scale
space in which these attractors merge.

\item In this framework, we define the \emph{critical scale} $\sigc$ as a memorization measure: the scale up to which
a sample remains an \emph{attractor} of the dynamics. It extends to conditional models, capturing the caption-image interaction. For the optimal case we prove that $\sigc$  depends on the inter-samples distances and their multiplicity, reflecting memorization induced by data duplication and outliers.

\item In controlled experiments we show that
$\sigma_c$ detects memorization induced by duplication, by overfitting and by planting outliers, transferring the mechanism from the theory to trained networks.

\item  We show that the critical scale is an effective detector on large-scale conditional models such as Stable Diffusion, for both fully and partially memorized examples and that the measures is
\emph{interpretable}: being able to characterize whether its retention comes from
the caption, from the image, or their interaction and where in the image the memorized content localizes.
\end{itemize}

\section{Diffusion models as a family of dynamical systems}
\label{sec:framework}
Let $x_0\in\R^{d}$ be drawn from a data distribution $p$, and let $y=x_0+\sigma\varepsilon$, with
$\varepsilon\sim\mathcal N(0,I)$, be its observation at noise level $\sigma>0$. Following
\citet{karras2022edm}, a diffusion model is a family of denoisers $\{\D:\R^{d}\to\R^{d}\}_{\sigma>0}$
trained by minimising $\E\,\|\D(x_0+\sigma\varepsilon)-x_0\|^{2}$ over a range of noise levels. Its
minimiser is the posterior mean $\D(y)=\E[x_0\mid y]$. Noise-prediction models such as DDPM
\citep{ho2020ddpm} and Stable Diffusion \citep{rombach2022ldm} reduce to this form: a network
$\epsilon_\theta$ trained to predict $\varepsilon$ gives $\D(y)=y-\sigma\epsilon_\theta(y)$, up to a
rescaling of the input (see Appendix~\ref{app:details}). 
\label{sec:map}

\begin{figure}[t]
\centering
\includegraphics[width=0.9\linewidth,trim={0 0.2cm 0 0.2cm},clip]{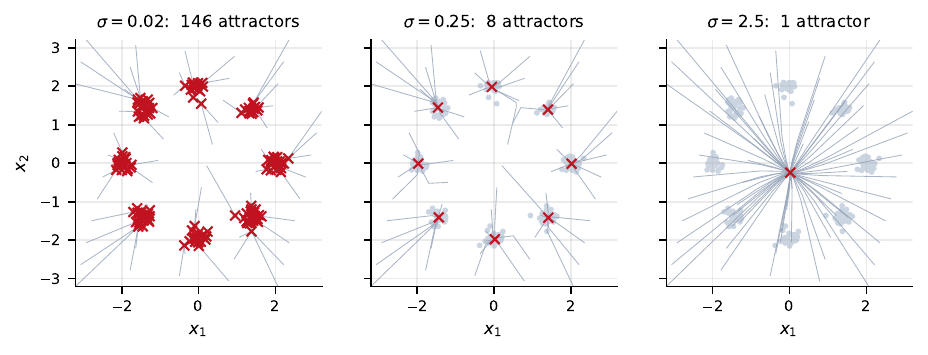}
\caption{\textbf{One model, a family of dynamical systems.} Orbits of the exact map $\M$
(Eq.~\ref{eq:meanshift}) for samples from an 8-mode mixture of gaussians. \emph{Left:} at small $\sigma$ each datum carries its own attractor. \emph{Middle:} at
intermediate $\sigma$ the attractors are the 8 population modes. \emph{Right:} at large $\sigma$ a
single attractor remains, at the data mean.}
\label{fig:portrait}
\end{figure}

We study each denoiser by fixing its noise level and iterating it on its own output, so that each noise level defines a deterministic dynamical system on $\R^{d}$ (Figure~\ref{fig:teaser}, left).
\begin{greybox}
\begin{definition}[Denoising dynamics at scale $\sigma$]
\label{def:map}
Fix $\sigma>0$ and let $\M:=\D$. The \emph{denoising dynamics at scale $\sigma$} is the deterministic
dynamical system on $\R^{d}$
\[
x_{k+1}=\M(x_k),\qquad k=0,1,2,\dots
\]
For $x\in\R^{d}$, the \emph{orbit} of $x$ is the sequence $(\M^{k}(x))_{k\ge0}$, where $\M^{k}$
denotes the $k$ times composition of $\M$. A point $x^\ast\in\R^{d}$ is a \emph{fixed point} if
$\M(x^\ast)=x^\ast$. Its \emph{basin of attraction} is
\[
\mathcal{B}_\sigma(x^\ast)=\big\{x\in\R^{d}:\ \lim_{k\to\infty}\M^{k}(x)=x^\ast\big\},
\]
and $x^\ast$ is an \emph{attractor} if $\mathcal{B}_\sigma(x^\ast)$ contains an open neighbourhood of
$x^\ast$. Letting $\sigma$ vary, a diffusion model defines the one-parameter family of dynamical
systems $\{\M\}_{\sigma>0}$.

\end{definition}
\end{greybox}
Throughout the paper our goal will be to characterize the dynamics of $\M$ as a function of the noise scale $\sigma$: see Figures~\ref{fig:portrait} and~\ref{fig:scalespace}, top row, for a 2D example of the map $M$ at different $\sigma_s$ and its dynamics.
Figure~\ref{fig:sdloop}
(Appendix~\ref{app:sdloop}) shows the loop on Stable Diffusion.
\label{sec:theory}

When the denoiser is exact, i.e. it predicts the posterior mean $\D(y)=\E[x_0\mid y]$,  the dynamics of $\M$ admit an explicit description: 

\begin{proposition}[
formal: Proposition~\ref{prop:meanshift-formal}]
\label{prop:meanshift}
For the exact denoiser, one step of the dynamics is a gradient-ascent step on the smoothed density,
$\M(x)=x+\sigma^{2}\nabla\log\ps(x)$. For a training set $p=\frac1N\sum_i\delta_{x_i}$, $M$ is
Gaussian mean shift with bandwidth $\sigma$,
\begin{equation}
\M(x)=\sum_i w_i(x)\,x_i,\qquad
w_i(x)=\frac{\exp\!\big(-\|x-x_i\|^{2}/2\sigma^{2}\big)}{\sum_j\exp\!\big(-\|x-x_j\|^{2}/2\sigma^{2}\big)}.
\label{eq:meanshift}
\end{equation}
Its fixed points are the critical points of $\ps$, its attractors are the modes of $\ps$, and every
orbit increases $\ps$ and converges to a fixed point.
\end{proposition}

The first identity is Tweedie's formula \citep{robbins1956empirical,efron2011tweedie}, and the
iteration of \eqref{eq:meanshift} is mean shift \citep{fukunaga1975,cheng1995,comaniciu2002}, the
classical mode-seeking clustering algorithm, run at bandwidth $\sigma$. The attractors of the
denoising dynamics are therefore the modes of the data blurred at scale $\sigma$. In the limit
$\sigma\to0$ every training example is an attractor of its own, whereas once $\sigma$ exceeds half the
diameter of the data, a single attractor remains, close to the data mean (Proposition~\ref{prop:ends});
at intermediate scales the attractors merge progressively, nearby data into cluster modes and cluster
modes into coarser ones. Figure~\ref{fig:portrait} shows the three regimes, and
Figure~\ref{fig:scalespace} traces the merging as a tree, in which the branch of an example terminates
where the example ceases to be an attractor of its own.

\subsection{Memorization and the critical scale}
\label{sec:sigc}

As $\sigma$ increases, a training example is absorbed into a coarser attractor, and the model reproduces only its cluster. We regard an example
as \emph{memorized} to the extent that it resists this absorption: in Figure~\ref{fig:scalespace},
non-duplicated examples lose their branch at small $\sigma$, whereas the duplicated example retains
it until merging with entire groups. To measure the scale of absorption without locating
the attractors, we initialize the dynamics at the example and test whether the orbit remains close to it.

\begin{figure}[t]
\centering
\includegraphics[width=0.9\linewidth,trim={0 0.2cm 0 0.2cm},clip]{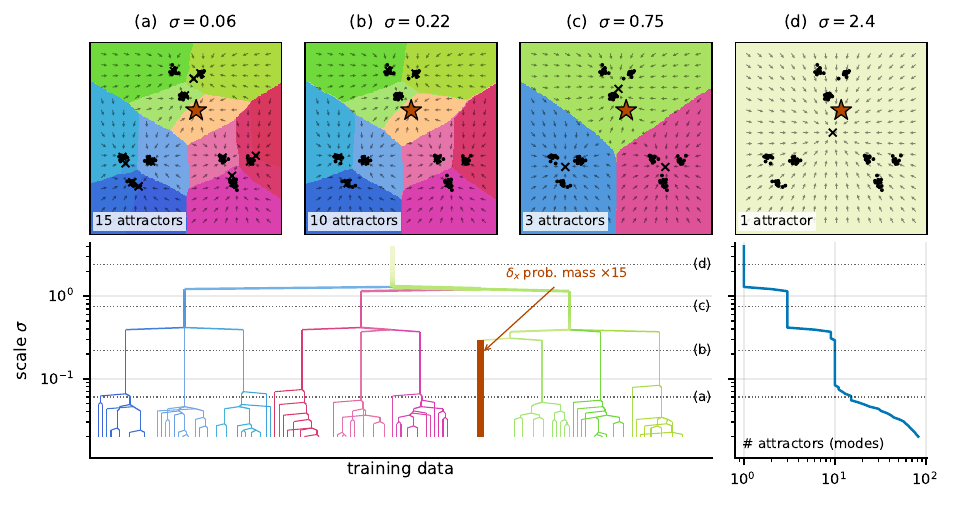}
\caption{\textbf{Scale space of attractors.} Exact map $\M$ on a 2-D training set: three groups of
three sub-clusters (black dots), plus one datum repeated $15$ times (red star). \emph{Top:} basins of
attraction at four scales, coloured by the reached attractor ($\times$); arrows are the drift
$\M(x)-x$. \emph{Bottom left:} merge tree of the attractors reached from the data (width $=$ number
of data collected). Single data lose their attractor at small $\sigma$; the duplicated datum (red)
keeps it until merging with the neighbouring sub-clusters. \emph{Bottom right:} number of attractors.
A denoiser trained on the same data reproduces this scale space, basins, merge tree and counts alike
(Appendix~\ref{app:remarks}, Fig.~\ref{fig:scalespacelearned}).}
\label{fig:scalespace}
\end{figure}

\begin{greybox}
\begin{definition}[Critical scale]
\label{def:sigc}
Fix an escape radius $r>0$ and an iteration budget $K$. The probe $x$ is \emph{retained at scale
$\sigma$} if $\|\M^{k}(x)-x\|\le r\|x\|$ for every $k\le K$. The critical scale of $x$ is the largest
scale at which it is retained,
\[
\sigc(x)\;=\;\sup\big\{\sigma>0:\ x\ \text{is retained at scale }\sigma\big\}.
\]
\end{definition}
\end{greybox}

The critical scale is thus the largest noise level at which the model, iterated on its own output,
still returns $x$: below $\sigc$ the example lies in a basin of its own, and above it the orbit is drawn towards an attractor possibly shared with other data. Figure~\ref{fig:teaser} illustrates the definition on a
memorized and a control image on Stable Diffusion model: the memorized image is returned over a range
of scales an order of magnitude wider than the control, and $\sigc$ marks the upper end of this range (the full grid of scales and iterations is in Figure~\ref{fig:sigcpic}, Appendix~\ref{app:sd}).
Being expressed in units of noise, $\sigc$ is comparable across models and noise schedules.

\paragraph{What the critical scale captures.}
\label{sec:sigc:raise}

For the exact denoiser, the critical scale $\sigc$ responds to the probability mass carried by an example and to its
distance from the rest of the data.

\begin{theorem}[formal: Theorem~\ref{thm:mass-formal}]
\label{thm:mass}
Increasing the probability mass on an example, for instance by duplicating it, reduces the displacement
$\|\M(x)-x\|$ at the example at every noise level and can therefore only increase its critical scale.
\end{theorem}

The map is pulled away from the example only by the other data, and additional mass on the example
dilutes this pull. The following theorem makes the rate explicit.

\begin{theorem}[formal: Thm.~\ref{thm:twoatom-formal}]
\label{thm:twoatom}
Consider an example of mass $\theta\le\tfrac12$ whose only neighbour, carrying the remaining mass,
lies at distance $d$. The example keeps an attractor of its own exactly below a critical scale, which
increases with $\theta$. For small $\theta$ this scale grows linearly in the distance and only
logarithmically in the mass, as $d/\sqrt{2\log(1/\theta)}$.
\end{theorem}

Doubling the distance to the neighbour therefore doubles $\sigma_c$, whereas doubling the copies raises it by a few percent. The critical scale thus captures both duplicated examples
and outliers. At $\sigc$ the attractor disappears in a saddle-node bifurcation, which the
escape test detects (Proposition~\ref{prop:escape}).

\paragraph{Conditional models and the caption gap.}
\label{sec:cond}

A text-to-image model defines one family of dynamics $\M(\cdot\,;c)$ for each caption $c$, with
$c=\varnothing$ the unconditional model. Memorization is then a property of a caption--image
\emph{pair}, and we measure the change that the caption induces in the $\sigma_c$ of the image.

\begin{greybox}
\begin{definition}[Caption gap]
\label{def:capgap}
With $\sigc(x;c)$ the critical scale computed with $\M(\cdot\,;c)$, the \emph{caption gap} of the pair
$(x,c)$ is $\Delta\log\sigc(x;c)=\log\sigc(x;c)-\log\sigc(x;\varnothing)$.
\end{definition}
\end{greybox}

Subtracting the unconditional value removes the retention that the image exhibits on its own. For the
exact denoiser, the caption enters the mean shift \eqref{eq:meanshift} only through the prior weights
of the training images, like a duplication count, and Theorem~\ref{thm:mass} gives the following.

\begin{corollary}[formal: Corollary~\ref{cor:capgap-formal}]
\label{cor:capgap}
For a single iteration, and when the caption leaves the relative weights of the other examples
unchanged, the caption gap of $(x_i,c)$ is positive if and only if the caption raises the model's mass
on \emph{that} example.
\end{corollary}

\paragraph{From the exact denoiser to trained networks.}
\label{sec:transfer}

The results above concern the exact denoiser, of which a trained network is only an approximation,
and we do not expect a trained network to reproduce their values. We expect the mechanism to carry
over: an example retains a basin of its own when it carries mass or lies far from the rest of the
data, and $\sigc$ measures the width of that basin. The experiments of Section~\ref{sec:evidence} verify if memorization from duplicated data and outliers transfers to real data scenarios. On the
exact denoiser of a real training set the theory holds quantitatively (Appendix~\ref{app:exact}), and
Appendix~\ref{app:coef} examines how closely trained networks approach it.
\section{Controlled experiments}
\label{sec:evidence}

We first evaluate the critical scale in a controlled setting, with DDPM~\citep{ho2020ddpm} models on
the \texttt{CIFAR-10} dataset~\citep{krizhevsky2009learning}. 
An example can be memorized for different reasons, but these mechanisms are confounded in pretrained
models. We therefore isolate them through controlled training interventions and test whether
critical scale detects the memorization induced by each.
In all experiments $\sigc$ is computed by bisection in $\log\sigma$ on the escape event of
Definition~\ref{def:sigc}. Training and evaluation details, and additional results are given in Appendices~\ref{app:details} and
~\ref{app:cifar}.
\begin{itemize}[leftmargin=*,itemsep=1pt,topsep=0.5pt,after=\vspace{-0.2cm}]
\item \emph{Duplication} (\S\ref{sec:cifar:dup}). An example is repeated $m$ times in the training
set, so the empirical distribution places mass $m/N$ on it.
\item \emph{Rarity} (\S\ref{sec:cifar:out}). An atypical example is added a few times. It has no
training neighbours with which to share a basin, but it also receives almost no probability mass.
\item \emph{Overfitting} (Appendix~\ref{app:cifar:n}). The training set is reduced at fixed model
capacity and training budget, moving the model from generalization to memorization. $\sigc$ tracks
this transition and detects memorization at training-set sizes where sampling no longer finds copies.
\end{itemize}

\begin{figure}[t]
\centering
\includegraphics[width=\linewidth, trim={0cm 0.2cm 0cm 0.25cm},clip]{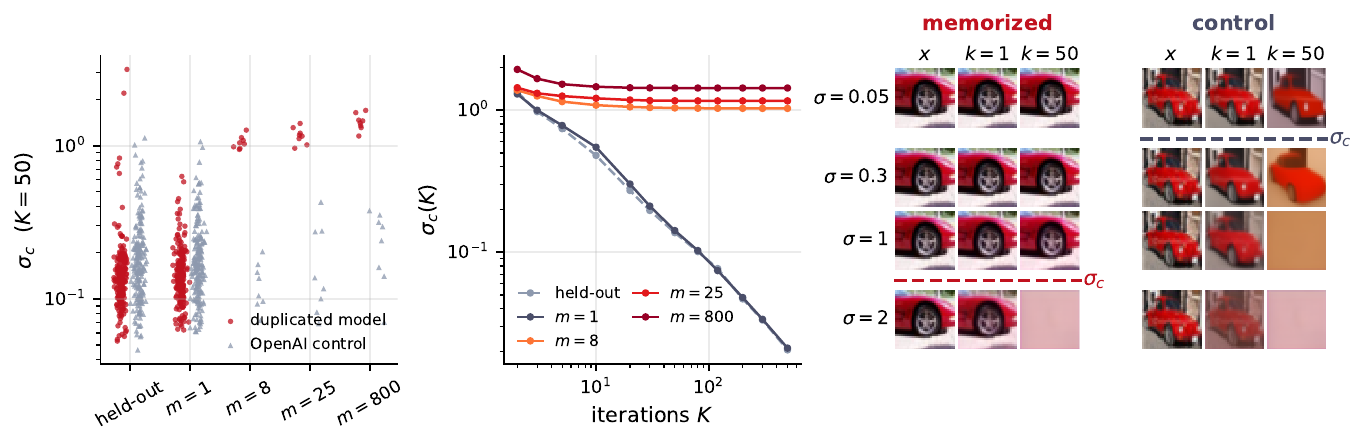}
\caption{\textbf{Duplication on full CIFAR-10.} \emph{(a)} $\sigc$ of each image, grouped by
multiplicity, in the model trained with duplicates (circles) and in the released model trained on the
same images without duplicates (triangles). \emph{(b)} $\sigc$ as a function of the iteration budget
$K$ (median per group): constant for images with $m\ge8$, which are fixed points, and decreasing as
$1/K$ for held-out and $m=1$ images, which are transients. \emph{(c)} Orbits $\M^{k}(x)$ of a memorized
training image ($m=400$) and of a training image seen once ($m=1$), with $\sigma$ fixed along each row
and $k$ along the columns; the dashed line marks the $\sigc$ of each image ($K=50$, $r=0.5$; $1.62$
and $0.14$). The memorized image is still reproduced after $50$ iterations at $\sigma=1$, where the
other one has already vanished.}
\label{fig:cifar}
\end{figure}

\subsection{Memorization from data duplication}
\label{sec:cifar:dup}

\paragraph{Setup.}
We study memorization induced by duplication, and ask whether $\sigc$ detects it and how it depends
on the number of copies. To isolate the effect of duplication from that of membership, we compare two
models trained on the same images: one with duplicates and one without. We train the
improved-diffusion CIFAR-10 model \citep{nichol2021improved} ($50$M parameters, $500$k steps) on the
full training set, in which eight images are repeated at each multiplicity
$m\in\{2,8,25,50,100,200,400,800\}$. As control we use the publicly released checkpoint of the same
model, trained on the same data without duplicates. We evaluate $\sigc$ on the $64$ duplicated
images, on $200$ training images that appear once ($m=1$), and on $200$ held-out images. Duplication
does induce copying: $16\,605$ of $50\,000$ generated samples are copies of a duplicated image.

\paragraph{Results.}
Figure~\ref{fig:cifar}a shows that $\sigc$ responds to the intervention rather than to the images. In
the model trained with duplicates, the median $\sigc$ is $0.137$ on held-out images and $0.142$ at
$m=1$, against $1.04$--$1.43$ for $m\ge8$. In the control model the same images have $\sigc$ between
$0.12$ and $0.27$ at every multiplicity, and the rank correlation with $m$ drops from $\rho=0.70$ to
$0.19$. Used as a detector, $\sigc$ separates $m\ge8$ from $m=1$ with Area Under the Curve (AUC) $1.000$. The duplicated images are fixed points: at $m=800$ the median
$\sigc$ is the same for $K=50$ and $K=500$ iterations, while for held-out and $m=1$ images it
decreases by more than two orders of magnitude over the same range (Figure~\ref{fig:cifar}b). Above
$m=8$, $\sigc$ increases by only a factor $1.38$ while $m$ increases a hundredfold, the saturation
expected of a width in which mass enters logarithmically (Appendix~\ref{app:cifar:dup}).

\paragraph{Takeaway.}
The critical scale captures memorization from data duplication, which turns transient states in the image orbits into stable fixed points. Additional copies increase $\sigc$ only slightly.

\vspace{-0.2cm}

\subsection{Memorization of rare samples}
\label{sec:cifar:out}
\vspace{-0.2cm}

\begin{figure}[h]
\centering
\includegraphics[width=0.93\linewidth,trim={0cm 0.1cm 0 0.1cm},clip]{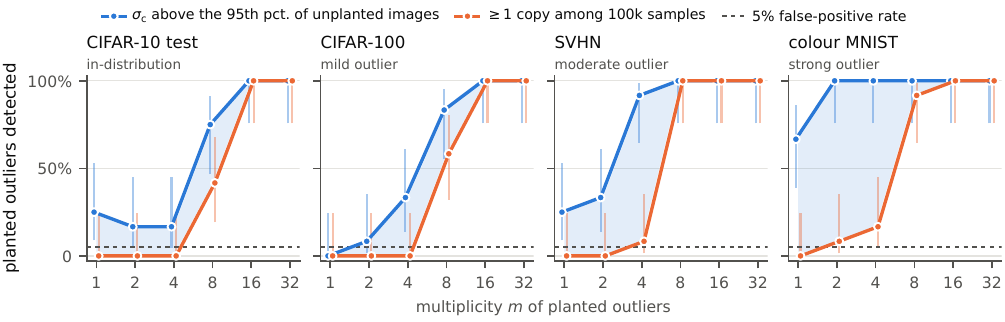}
\caption{\textbf{Rarity: detection thresholds of $\sigc$ and of sampling.} Fraction of added images
detected as a function of their multiplicity $m$, for sources of increasing atypicality. \emph{Orange}: the
image appears at least once among $10^5$ samples; the threshold is $m\approx8$ for every source. \emph{Blue}:
$\sigc$ exceeds the $95$th percentile of the images of the same source not added to the model; the
threshold decreases to $m=1$ as the images become more atypical. The shaded region is memorization
detected by $\sigc$ but not by sampling.}
\label{fig:outliers}
\end{figure}

\paragraph{Setup.}
We study memorization of outliers, which Theorem~\ref{thm:twoatom} predicts to be retained
with few copies because of their distance to the other samples. We train two models, A and B, on the
same subset of $N=10^4$ \texttt{CIFAR-10} images, to which we add images from 4 datasets of increasing atypicality: \texttt{CIFAR-10} test images, \texttt{CIFAR-100} classes with no \texttt{CIFAR-10} counterpart, \texttt{SVHN} digits, and
colourised \texttt{MNIST} digits. The added images statistics are matched to the \texttt{CIFAR-10}. From each source, $36$ images are added to A, $36$ other images to B, and $100$ images to
neither, with six images at each multiplicity $m\in\{1,2,4,8,16,32\}$. Each image is thus memorized in
one model and not in the other, and the images added to neither model serve as a reference
distribution. We draw $10^5$ samples from each model and count near-copies of the added images.

\paragraph{Results.}
Sampling and $\sigc$ have different detection thresholds (Figure~\ref{fig:outliers}). Sampling detects added images only from $m\approx8$, for every source; below this multiplicity no image is generated.
We flag an image with $\sigc$ when its value exceeds the $95$th percentile non-added images from
the same source, which gives a false-positive rate of $5\%$ without using
labels.  The resulting threshold
decreases with atypicality:  $m=8$ for \texttt{CIFAR-10} and
\texttt{CIFAR-100}, $m=4$ for \texttt{SVHN} and $m=1$ for colour \texttt{MNIST}. The paired design confirms that the effect is due
to training on the image and not to the image itself. 
Among the added images that were never sampled in the $10^5$ drawns, $\sigc$ flags
$88.2\%$ (\texttt{MNIST}), $48.6\%$ (\texttt{SVHN}), $25.6\%$ (\texttt{CIFAR-10}), and $22.0\%$ (\texttt{CIFAR-100}).
Additional analysis on how $\sigc$ depends on distance
and on multiplicity in these models is in Appendix~\ref{app:cifar:out}.

\paragraph{Takeaway.}
The critical scale can detect memorized samples that are unlikely to be generated: $\sigma_c$ depends also on distance, generation only on mass. An isolated example (outlier) can be memorized, since it has room for its own basin, yet not be sampled, since it carries little probability mass.
\section{Detecting Memorization in Stable Diffusion}\label{sec:sd}

\begin{table}[t]
\centering\small

\caption{\textbf{Memorization-detection on Stable Diffusion v1.4.}
AUC for distinguishing MV and TV examples from controls, caption-swapped pairs, and MV from TV. Our global scores use $r=0.25$, and local $|\Delta\log\sigc|$ uses $r=0.1$. For both variants of \citet{wen2024detecting}, we report its best-performing configuration (Appendix~\ref{app:details}).}
\label{tab:sd}
\begin{tabular}{@{}lccccc@{}}
\toprule
Score & MV vs.\ control & MV vs.\ swap & TV vs.\ control & TV vs.\ swap  & MV vs.\ TV\\
\midrule
$\|\epsilon_c-\epsilon_\varnothing\| $ \footnotesize{\cite{wen2024detecting}} & $\mathbf{0.999}$ & $0.606$ & $\mathbf{0.999}$ & $0.184$ & $0.902$ \\
$\|\epsilon_c(x_0)-\epsilon_\varnothing(x_0)\|$ & $\mathbf{0.999}$ & $0.887$ & $\mathbf{0.996}$ & $0.696$ & $0.868$\\
\midrule
$\sigc$, $K=16$                & $0.917$ & $0.904$ & $0.695$ & $0.803$  & $0.909$ \\
$\Delta\log\sigc$, $K=16$                & $0.8569$ & $\mathbf{0.908}$ & $0.824$ & $\mathbf{0.939}$  & $0.820$ \\
 $|\Delta\log\sigc|$, $K=16$       & $0.984$ & $0.860$ & $0.879$ &  $0.354$ & $0.905$\\
local $|\Delta\log\sigc|$, $K=16$         &    $\mathbf{0.997}$  & $0.865$ &  $\mathbf{0.993}$ & $0.458$  & $\mathbf{0.935}$ \\
\bottomrule
\end{tabular}\\[2pt]
\end{table}

We evaluate Stable Diffusion v1.4 \citep{rombach2022ldm} on memorized \texttt{LAION} training examples
identified by \citet{webster2023reproducible}: \emph{matching verbatim} (MV) examples, whose
generations reproduce the paired training image, and \emph{template verbatim} (TV) examples, whose
generations preserve only part of its content or structure. After removing duplicates and mismatched
image--caption pairs (Appendix~\ref{app:details}), we retain $70$ MV and $24$ TV examples. As controls
we use $94$ training images from \texttt{LAION-MI}\citep{dubinski2024membership}, drawn from LAION Aesthetics v2 5+\citep{schuhmann2022laion}, so that all images are training members
and differ only in their degree of memorization. Each memorized image is evaluated with its own
caption, a \emph{swap} caption (captions randomly permuted across memorized images) and the empty
caption; each control with its own and the empty caption.

\paragraph{Scores.}
From these runs we compute the critical scale $\sigc$ under the own caption and the caption gap
$\Delta\log\sigc$ of Definition~\ref{def:capgap}, which measures how much the caption increases the
retention of that image (Figures~\ref{fig:traj-mv-control}--\ref{fig:sweep} in
Appendix~\ref{app:examples}). Both are global and do not say \emph{where} in the image memorization
occurs. We therefore also apply the retention test separately at each latent location $u\in U$, which
gives a local gap $\Delta\log\sigc^{u}$ and the local score
$\frac{1}{|U|}\sum_{u\in U}|\Delta\log\sigc^{u}|$ (see Appendix~\ref{app:details}). To separate the benefit
of spatial resolution from that of discarding the sign, we also report the absolute global gap
$|\Delta\log\sigc|$. We first evaluate critical-scale scores for detecting memorized examples
 (\S\ref{sssec:sd-detection}), then examine whether local gaps identify memorized regions
 (\S\ref{sssec:sd-localization}), and finally use mismatched captions to identify images memorized in
 the unconditional branch (\S\ref{sssec:sd-image-driven}).

\subsection{Critical scale detects memorization}\label{sssec:sd-detection}

\paragraph{Setup.}
We report AUC for separating: (i) MV and TV examples from controls, (ii) caption-swapped pairs, and (iii) MV from TV. We compare our scores with the prompt-level memorization
score $\|\epsilon_\theta(x_t,c)-\epsilon_\theta(x_t,\varnothing)\|$ of
\citet{wen2024detecting}. Since this measure depends only on the caption, we also consider an image-dependent variant: $x_t$ is obtained by adding noise to the candidate image $x_0$.
This provides a stronger baseline for testing if a particular image-caption pair is memorized.
\paragraph{Results.}
Table~\ref{tab:sd} reports the detection results (see Appendix~\ref{app:sd-ablation} for an ablation on $K$ and $r$). All scores separate memorized examples from control, with local $|\Delta\log\sigma_c|$ performing best among our variants. It achieves AUCs of $0.997$ for MV and $0.993$ for TV, nearly matching the prompt-level baseline ($0.999$). Its largest gain occurs for TV, where it raises the AUC from $0.879$ to $0.993$, consistent with partial memorization being localized. It also best separates MV from TV ($0.935$).

For caption swaps, the signed gap performs best because a mismatched caption can reduce conditional retention below unconditional retention, reversing the gap’s sign (Figure~\ref{fig:sd-histograms}, Appendix~\ref{app:hist}). Taking the absolute value discards this directional information. This advantage is most pronounced for TV, where $\Delta\log\sigma_c$ achieves $0.939$ AUC versus $0.696$ for the image-dependent baseline. 

\paragraph{Takeaway.}
Critical scale is an effective global measure of memorization at scale, and computing it per spatial location improves detection. Comparing unconditional retention with retention under a mismatched caption further reveals whether an image was memorized with a particular caption.
\subsection{Local critical scale localizes memorized regions}\label{sssec:sd-localization}
\begin{figure}[t]
    \centering

    \begin{subfigure}[t]{0.38\linewidth}
        \centering
        \raisebox{0.4cm}{%
            \includegraphics[width=\linewidth]{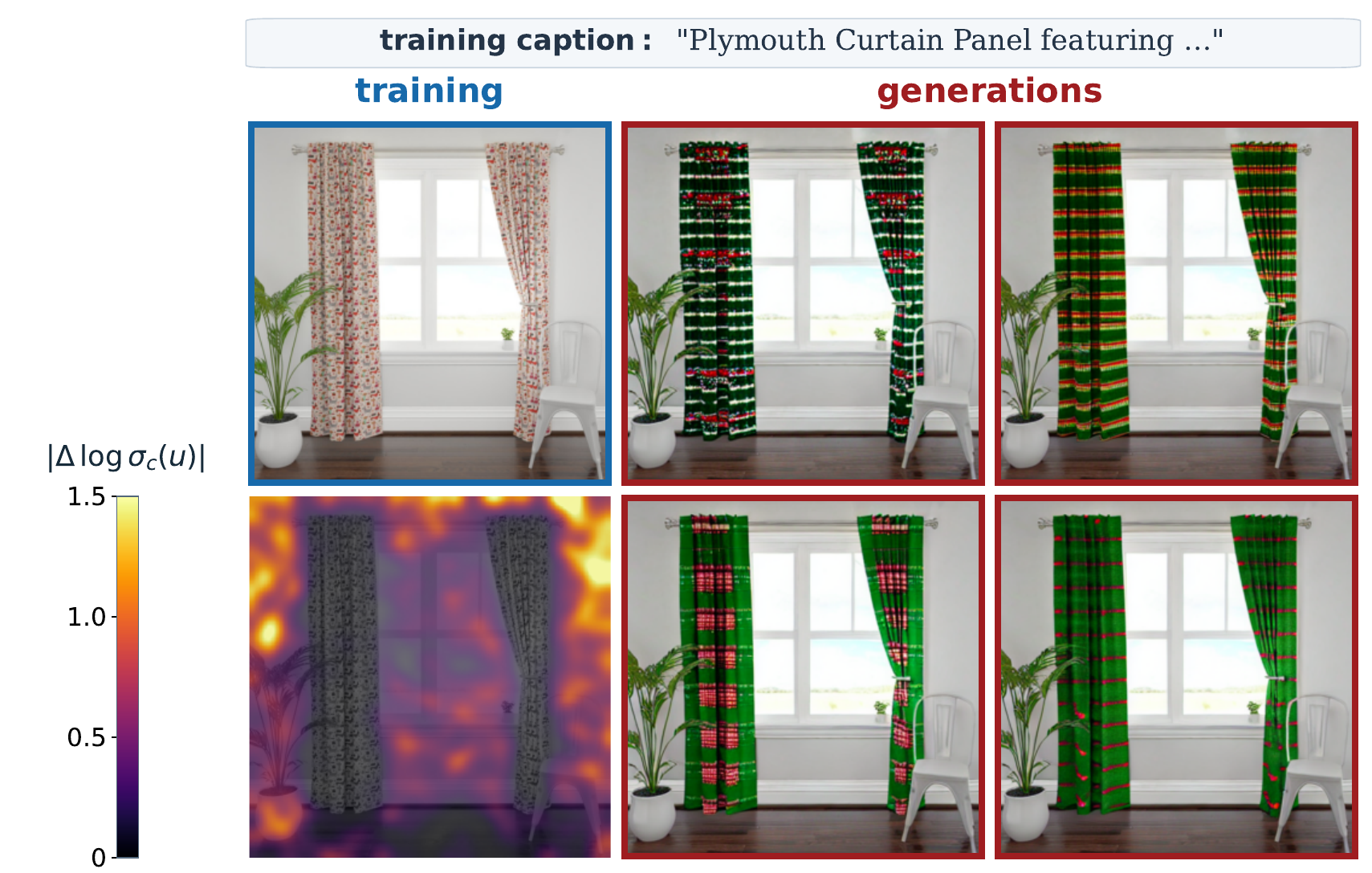}%
        }
    \end{subfigure}
    \hfill
    \begin{subfigure}[t]{0.60\linewidth}
        \centering
        \includegraphics[width=\linewidth]{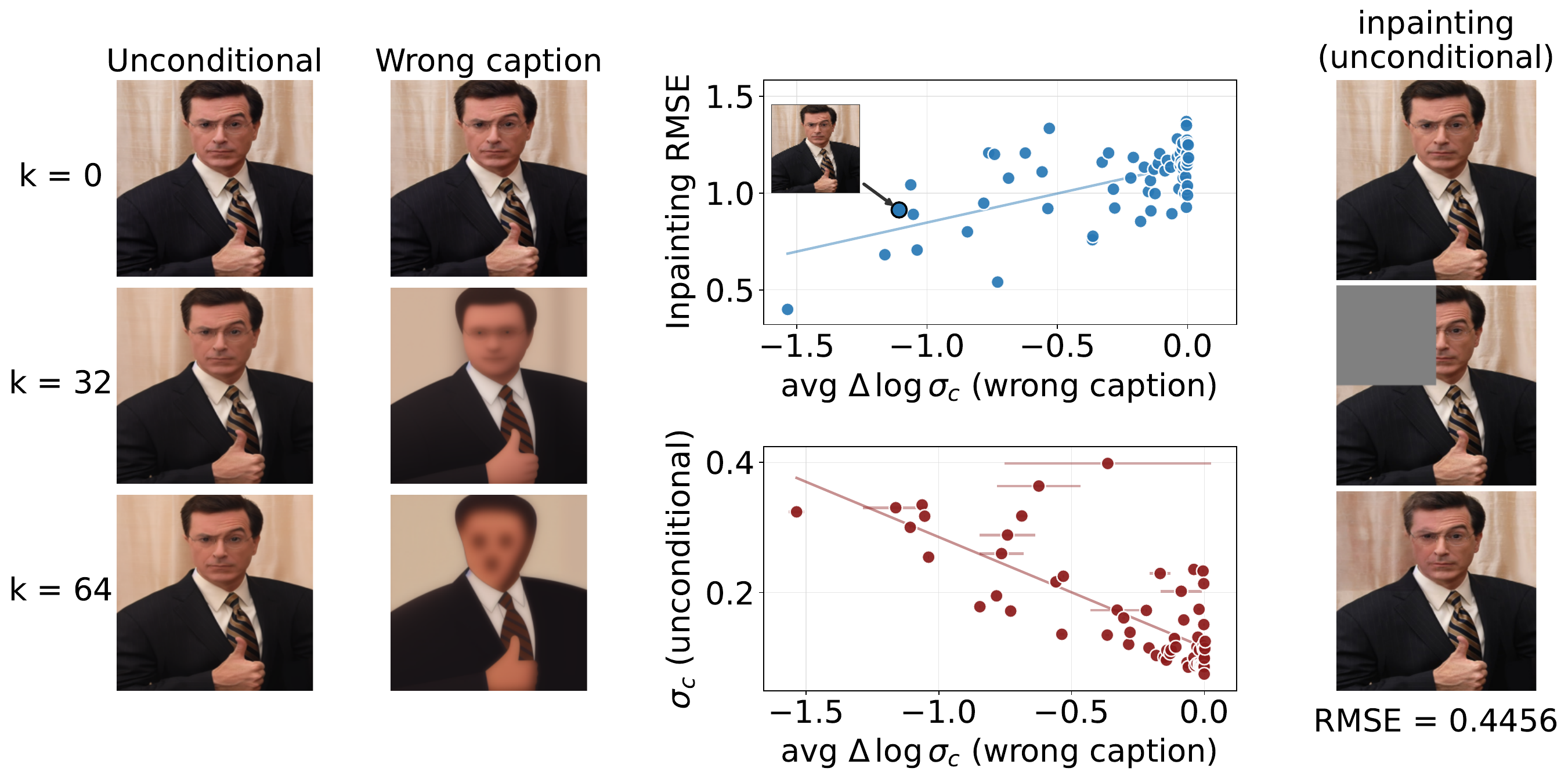}
    \end{subfigure}

    \caption{\textbf{Critical-scale gaps reveal spatial and image-driven memorization.}
\textit{Left:} A TV example: the training image, generations from its training caption, and the
local gap $|\Delta\log\sigma_c^u|$. The gap is large where all generations copy the training image
(the room) and near zero where they vary (the curtain pattern).
\textit{Right:} A wrong caption exposes images already retained by the unconditional model.
Center: wrong-caption gap $\Delta\log\sigma_c$ against unconditional inpainting RMSE (top) and
unconditional $\sigma_c$ (bottom); bars include std deviation over 5 captions. Side: a negative-gap image is retained by unconditional dynamics but lost when conditioning on a wrong caption;
unconditional inpainting restores the masked region.
}
    \label{fig:sd-interpretability}
\end{figure}
\paragraph{Setup.}
We qualitatively assess if the local gap $|\Delta\log\sigma_c^u|$ identifies memorized regions
in TV examples. As a reference, we generate five samples from each training caption: consistently
reproduced regions indicate memorized content, whereas variable regions indicate non-memorized
content. 
\paragraph{Results.}
Figure~\ref{fig:sd-interpretability} (left) shows that large values of
$|\Delta\log\sigma_c^u|$ concentrate in regions reproduced
consistently across generations. By contrast, regions that vary across samples generally have gaps close to zero. Thus, the spatial variation of the critical-scale gap aligns with the apparent boundary between memorized and non-memorized content. The trajectories in Figures~\ref{fig:traj-tv-control} and~\ref{fig:traj-mv-tv}  provide
complementary evidence: in partially memorized images, non-memorized content disappears at smaller
noise scales, whereas memorized regions persist longer. Additional qualitative examples are provided in Figure~\ref{fig:additional-local-gap} in Appendix~\ref{app:local-sigma-gap}.
\paragraph{Takeaway.}
The local critical-scale gap provides an \emph{interpretable} spatial measure of where memorized content
appears within an image.
\subsection{Critical-Scale Gap Identifies Image-Driven Memorization}\label{sssec:sd-image-driven}
\paragraph{Setup.}
We aim to find images that the unconditional branch has memorized by itself, and not only through their captions. A strongly negative $\Delta\log\sigma_c$ indicates stronger retention under unconditional
than conditional dynamics. However, since unconditional memorization is rare, searching the full training set for such examples is computationally prohibitive.
We therefore use the MV examples under five unrelated captions drawn from LAION-MI, and compute $\Delta\log\sigma_c$ for each image-caption pair. This allows us to rank images by evidence of unconditional retention.
Since finding an image through unconditional sampling is not feasible, we evaluate this ranking through inpainting. We mask each image quadrant in turn, reconstruct it from the remaining three, and report the mean latent RMSE to the ground-truth region. For each image, we generate 10 samples per quadrant. Low reconstruction error provides evidence that the image is retained independently of its caption.
\paragraph{Results.}
Figure~\ref{fig:sd-interpretability} (right) shows an example with a strongly negative gap. Its
unconditional dynamics retain the image more strongly than the mismatched-caption dynamics, and
unconditional inpainting accurately reconstructs the missing content. This provides qualitative evidence that the image is memorized by the unconditional branch. This reconstruction is consistent with the memorized image acting as an attractor and the masked image lying within its basin of attraction (see Figure~\ref{fig:sdloop} in Appendix~\ref{app:sdloop}).
Across images, strongly negative gaps are associated with lower inpainting RMSE (top center) and
larger unconditional $\sigma_c$ (bottom center), supporting their use as evidence of unconditional
retention. Gaps near zero are less informative because they can arise when either both dynamics
retain the image or neither does. Figure~\ref{fig:near-zero-gap} in Appendix~\ref{app:capt-swap} illustrates both near-zero regimes, while Figure~\ref{fig:negative-gap-examples} provides additional examples with strongly negative gaps.
\paragraph{Takeaway.}
A strongly negative critical-scale gap provides evidence that an image is retained independently of
its caption, identifying cases in which the image itself contributes to memorization and can be
reconstructed by the unconditional branch.

\section{Related work}
\label{sec:related}
\paragraph{Dynamical and multiscale views of diffusion models.}
Mean shift turns gradients of kernel density estimates into mode-seeking dynamics
\citep{fukunaga1975,cheng1995,comaniciu2002}; with Gaussian kernels, it equals 
posterior-mean denoising of the empirical distribution, whose smoothed-density critical points are
fixed points and modes are attractors. 
The link between denoising and the score, established for denoising autoencoders \citep{alain2014autoencoders}, underlies score-based diffusion via reverse-time SDEs/ODEs \citep{song2021sde}. These dynamics reveal how structure emerges across noise scales:
global high-variance features precede low-variance details \citep{wang2023diffusion}, a speciation transition marks the emergence of broad structure \citep{biroli2024dynamical}, and high- and low-level features evolve differently, reflecting the data hierarchy \citep{sclocchi2025phase}. These works follow
generation trajectories: instead, we fix the noise scale and iterate the denoiser, obtaining a family of dynamical systems whose attractors and persistence across
scales describe the learned structure.

\paragraph{Memorization and generalization in diffusion models.}
The empirical optimum of denoising score matching reproduces the training data: \citet{gu2025memorization} study when learned models approach it, and \citet{baptista2025memorization} how explicit and implicit regularization in the reverse dynamics prevent it. 
 \citet{kadkhodaie2024generalization} attribute generalization to
geometry-adaptive inductive biases, \citet{shah2025generation} tie memorization mainly to low-noise
denoising, and \citet{bonnaire2025why} show that good generation precedes memorization during
training. 
Closest to us, \citet{fumero2025navigating} iterate an autoencoder's map, whose latent fixed points and attractors reflect memorization and generalization. A diffusion model instead induce a family of dynamical systems, so attractor persistence can be tracked along a scale axis. \citet{pham2025memorization} view diffusion models as associative memories, 
and \citet{jain2025classifier} argue that classifier-free guidance steers reverse trajectories into the basins of memorized images. Both study the reverse  generation trajectory, while our fixed-scale dynamics measure how strongly a given image-caption pair is retained across noise scales.

\paragraph{Detecting memorization in diffusion models.}
Early work detected replication by matching generations to training data \citep{somepalli2023forgery} or extracting training examples \citep{carlini2023extracting,webster2023reproducible}. Later methods
avoid exhaustive generation. \citet{wen2024detecting} detect memorized prompts from the conditional--unconditional noise-prediction gap. \citet{jeon2025sharpness} use probability-landscape sharpness and \citet{ross2025geometric} the local geometry of model generations to detect memorization.
\citet{jiang2025imagelevel} address image-level detection without the paired prompt and \citet{chen2025exploring} localize memorized regions.  Our score instead depends jointly on a
candidate image--caption pair, testing whether that pair is memorized rather than whether the caption alone elicits memorization. Comparing conditional and empty-caption critical scales further detects memorization in the unconditional branch. Membership inference \citep{carlini2023extracting,hu2023membership,zhai2024clid,duan2023membership,matsumoto2023membership} instead asks whether an example was used in training, not whether it is retained strongly enough to be reproduced.

\section{Conclusions}
\label{sec:discussion}

In this paper we proposed to interpret a diffusion model as a \emph{family of dynamical systems} indexed by the
noise level, whose attractors form a scale space of the data. Memorization then becomes the persistence of an example as an attractor, measured by the critical scale $\sigc$. 
On Stable Diffusion, $\sigma_c$ detects fully and partially memorized image--caption pairs and offers two forms of interpretation: comparing conditional and unconditional critical scales isolates the caption-image interaction in memorization, while local variants of $\sigc$ can  identify which regions
within an image are memorized. Further analysis is needed to quantify the precision and generality of this localization, which in principle can be extended also to the caption parts.
Our theory holds for the exact denoiser. While the mechanism transfers to trained networks (see Section \ref{sec:evidence}), closing this gap formally--for instance via analytic models of the inductive biases of trained diffusion models \cite{kamb2024analytic}--is a natural direction for future work. 
Beyond memorization, the orbits and basins of the denoising maps may reveal broader structure in what a model has learned \cite{sclocchi2025phase}, and locating where individual examples are retained in the weights \cite{hintersdorf2024nemo} suggests a route towards targeted unlearning.

\section*{Acknowledgements}
M.F. was supported by the project “Building Energy Systems on causal reasoning (BOSS)“, funded within the “Technologies and Innovations for the Climate-Neutral City” (TIKS) Programme
of the Austrian Research Promotion Agency (FFG).
\clearpage
\subsection*{AI use statement}
We used generative AI tools as coding assistants when implementing and modifying the codebase, and
for language editing. They were not used to produce research results on their own or to replace
the authors' scientific judgment. The authors designed the methodology and experiments, and
reviewed and validated all AI-assisted code and text.

\subsection*{Ethics statement}
This work is primarily theoretical and methodological, focusing on the dynamics of denoising maps
and on memorization in diffusion models. All datasets and models are publicly available (CIFAR-10,
MNIST, SVHN, LAION-5B \citep{schuhmann2022laion} and Stable Diffusion
\citep{rombach2022ldm}, and the memorized prompts we study were identified in prior
work \citep{webster2023reproducible}. No human subjects are involved and no new data is collected.
Our detectors probe which training images a model has memorized. In principle, such insights
could be misused to recover training data from pretrained models. However, our experiments score
images that are already public, we release no new private or copyrighted content, and the
techniques should be applied responsibly, for instance to audit and reduce memorization. We hope
our findings contribute to a better understanding of how diffusion models generalize and
memorize.

\subsection*{Reproducibility statement}
All formal statements and proofs are reported in Appendix~\ref{app:proofs}. We implement all
experiments in PyTorch \citep{paszke2019pytorch} and the HuggingFace \texttt{diffusers}
library \citep{vonplaten2022diffusers}. We use only publicly released checkpoints, and the CIFAR-10
models we train follow the public recipe of \citet{nichol2021improved}. Models, datasets, training
recipes, noise schedules, estimator hyperparameters and the filtering of the memorization set are
listed in Appendix~\ref{app:details}. We will open-source our codebase for all experiments upon
acceptance.

\bibliographystyle{plainnat}
\bibliography{refs}

\newpage
\appendix
\numberwithin{theorem}{section}
\graystatements

\section{Denoising dynamics in Stable Diffusion}
\label{app:sdloop}

Figure~\ref{fig:sdloop} shows Definition~\ref{def:map} on Stable Diffusion~v1.4. The image is encoded
once into the latent space, where the map is iterated; the decoder is used only to display the iterates.

\begin{figure}[ht]
\centering
\includegraphics[width=\linewidth]{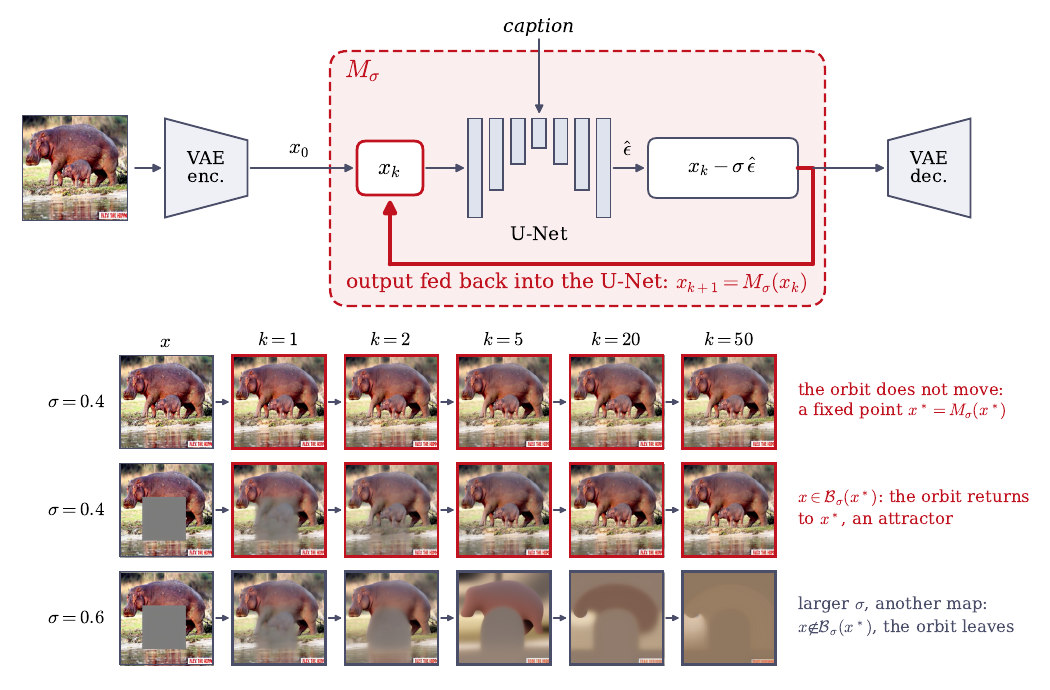}
\caption{\textbf{Definition~\ref{def:map} on Stable Diffusion.} \emph{Top:} the map $\M$. The image is
encoded once into the latent $x_0$; the U-Net is evaluated with the caption at noise level $\sigma$, and its
output $x_k-\sigma\hat\epsilon$ is fed back as the next input. \emph{Bottom:} each row is an orbit
$(\M^{k}(x))_{k\ge0}$, decoded for display, of a memorized image under its own caption ($w=1$).
\emph{First row:} started at the image, the orbit does not move; the image is a fixed point $x^\ast$.
\emph{Second row:} started with part of the image replaced by a grey square, the orbit restores it and
converges to the same $x^\ast$. The occluded image lies in the basin $\mathcal B_\sigma(x^\ast)$, so
$x^\ast$ is an attractor. \emph{Third row:} the same start at a larger $\sigma$, which defines a different
map in the family $\{\M\}_{\sigma>0}$. There the start is no longer in the basin, and the orbit leaves.}
\label{fig:sdloop}
\end{figure}

\section{Proofs}
\label{app:proofs}

\paragraph{Setting and notation.}
Throughout, $p$ is a probability measure supported in a compact convex set $S\subset\R^{d}$ of
diameter $R<\infty$. This covers image data in $[-1,1]^{d}$ and, taking $S$ to be its convex hull,
any finite training set. Write $\varphi_\sigma(u)=(2\pi\sigma^{2})^{-d/2}e^{-\|u\|^{2}/2\sigma^{2}}$.
The smoothed density $\ps(x)=\int\varphi_\sigma(x-x_0)\,p(dx_0)$ is $C^{\infty}$ and strictly
positive, and derivatives may be taken under the integral because $S$ is bounded. The posterior over
clean signals given $x$ is
$\pi_x(dx_0)=\varphi_\sigma(x-x_0)\,p(dx_0)/\ps(x)$. $\E[\,\cdot\mid x]$ and $\Cov[\,\cdot\mid x]$
denote its mean and covariance, so $\D(x)=\E[x_0\mid x]$. A fixed point $x^\ast$ of $\M$ is an
\emph{attractor} if every eigenvalue of $\nabla\M(x^\ast)$ has modulus less than $1$. A \emph{mode} is
a local maximum of $\ps$, and it is \emph{nondegenerate} if $\nabla^{2}\log\ps$ is negative definite
there. Two elementary facts are used throughout.

\begin{lemma}
\label{lem:facts}
For every $x\in\R^{d}$: \emph{(F1)} $\D(x)\in S$; \emph{(F2)}
$0\preceq\Cov[x_0\mid x]\preceq\frac{R^{2}}{4}I$.
\end{lemma}
\begin{proof}
(F1): $\D(x)$ is the mean of a probability measure on the convex set $S$. (F2): for a unit vector
$v$, $v^{\top}x_0$ lies in an interval of length at most $R$, and a random variable confined to an
interval of length $\ell$ has variance at most $\ell^{2}/4$.
\end{proof}

%----------------------------------------------------------------------
\subsection{The exact map (Section~\ref{sec:theory})}
\label{app:exactmap}

\paragraph{Intuition.}
The exact denoiser averages the training data, weighting each clean signal by how plausible it is
as the source of $x$. Moving to that average moves $x$ uphill on the smoothed density: the step is
the score (Tweedie). How sharply the output reacts to the input is the spread of the plausible
sources, i.e.\ the posterior covariance. So a point is held in place when the posterior at it is
narrower than the noise, which is exactly a mode of $\ps$.

\begin{proposition}[Exact denoisers are mean shift; formal version of Proposition~\ref{prop:meanshift}]
\label{prop:meanshift-formal}
Let $\D(x)=\E[x_0\mid x]$ be the exact denoiser. Then $\M$ is Gaussian mean shift with bandwidth
$\sigma$ on the smoothed density $\ps$:
\begin{enumerate}\itemsep1pt
\item[(i)] $\M(x)=x+\sigma^{2}\nabla\log\ps(x)$, and for $p=\frac1N\sum_i\delta_{x_i}$ it is given
by~\eqref{eq:meanshift}.
\item[(ii)] The fixed points of $\M$ are the critical points of $\ps$. A fixed point $x^\ast$ is an
attractor iff $\Cov[x_0\mid x^\ast]\prec\sigma^{2}I$, i.e.\ iff it is a nondegenerate mode of $\ps$.
\item[(iii)] Every step increases $\ps$. The map has no cycles, and every orbit converges to a fixed
point when the fixed points are isolated.
\end{enumerate}
\end{proposition}

Proposition~\ref{prop:meanshift-formal} combines four statements, which we prove in turn: Tweedie's
identities (Proposition~\ref{prop:tweedie}), the fixed points and their stability
(Proposition~\ref{prop:stab}), the ascent property (Lemma~\ref{lem:ascent}) and its consequence for
orbits (Corollary~\ref{cor:orbits}).

\begin{proposition}[Tweedie]
\label{prop:tweedie}
$\D(x)=x+\sigma^{2}\nabla\log\ps(x)$ and $\nabla\D(x)=\sigma^{-2}\Cov[x_0\mid x]$.
\end{proposition}

\begin{proof}[Proof of Proposition~\ref{prop:tweedie}]
This is Tweedie's formula \citep{robbins1956empirical,efron2011tweedie}. Since
$\nabla_x\varphi_\sigma(x-x_0)=\sigma^{-2}(x_0-x)\varphi_\sigma(x-x_0)$, integrating against $p$ gives
$\nabla\ps(x)=\sigma^{-2}\ps(x)(\D(x)-x)$, which is the first identity. For the second, note that
$\nabla_x\log[\varphi_\sigma(x-x_0)/\ps(x)]=\sigma^{-2}(x_0-\D(x))$. Differentiating
$\D(x)=\int x_0\,\pi_x(dx_0)$ under the integral therefore gives
$\nabla\D(x)=\sigma^{-2}\E[x_0(x_0-\D(x))^{\top}\mid x]=\sigma^{-2}\Cov[x_0\mid x]$.
\end{proof}

Combining the two identities of Proposition~\ref{prop:tweedie},
\begin{equation}
\nabla^{2}\log\ps(x)=\sigma^{-2}\big(\nabla\D(x)-I\big)=\sigma^{-4}\Cov[x_0\mid x]-\sigma^{-2}I .
\label{eq:tweedie2}
\end{equation}
For a finite $p=\sum_i\mu_i\delta_{x_i}$, where $\mu_i\propto m_i$ and $m_i$ is the multiplicity of
$x_i$, the posterior is discrete and its mean is
\begin{equation}
\M(x)=\sum_iw_i(x)\,x_i,\qquad
w_i(x)=\frac{m_i\,e^{-\|x-x_i\|^{2}/2\sigma^{2}}}{\sum_jm_j\,e^{-\|x-x_j\|^{2}/2\sigma^{2}}} .
\label{eq:meanshift-m}
\end{equation}
This is~\eqref{eq:meanshift} when all $m_i=1$. The weighted form is the one used for $\sigemp$.

\begin{proposition}[Fixed points and stability]
\label{prop:stab}
$x^\ast$ is a fixed point of $\M$ iff $\nabla\ps(x^\ast)=0$. The Jacobian $\nabla\M$ is symmetric
positive semidefinite everywhere, so the linearised map neither rotates nor oscillates. A fixed point
is an attractor iff $\Cov[x_0\mid x^\ast]\prec\sigma^{2}I$, equivalently iff $x^\ast$ is a
nondegenerate mode of $\ps$. Saddles of $\ps$ carry the basin boundaries.
\end{proposition}

\begin{proof}[Proof of Proposition~\ref{prop:stab}]
\emph{Fixed points.} By Proposition~\ref{prop:tweedie}, $\M(x)-x=\sigma^{2}\nabla\log\ps(x)$. Since
$\ps>0$, $\M(x^\ast)=x^\ast$ iff $\nabla\ps(x^\ast)=0$.

\emph{Jacobian.} $\nabla\M=\sigma^{-2}\Cov[x_0\mid x]$ is symmetric positive semidefinite. Its
eigenvalues are therefore real, so the linearised map has no rotation, and non-negative, so it has no
sign-reversing (oscillating) direction.

\emph{Stability.} Because the eigenvalues are non-negative, they all lie in $[0,1)$ iff
$\Cov[x_0\mid x^\ast]\prec\sigma^{2}I$. By \eqref{eq:tweedie2} this holds iff
$\nabla^{2}\log\ps(x^\ast)\prec0$. At a critical point, that is a nondegenerate local maximum.

\emph{Saddles.} If $\nabla^{2}\log\ps(x^\ast)$ has a positive eigenvalue, \eqref{eq:tweedie2} gives
$\nabla\M(x^\ast)$ an eigenvalue greater than $1$, so $x^\ast$ repels along that direction. A saddle
attracts along some directions and repels along others. By Corollary~\ref{cor:orbits} below, a point
in no attractor's basin converges to such a non-maximal critical point. Generically this is a saddle,
and the sets converging to saddles separate the basins.
\end{proof}

The linearisation is local. The global behaviour comes from the classical ascent property of Gaussian
mean shift \citep{comaniciu2002}: each step climbs $\ps$ by at least its own squared
length.

\begin{lemma}[Ascent]
\label{lem:ascent}
For every $x$, $\ \log\ps(\M(x))-\log\ps(x)\ge\|\M(x)-x\|^{2}/2\sigma^{2}$.
\end{lemma}
\begin{proof}
For any $y$,
$\ps(y)/\ps(x)=\E\big[\exp\{(\|x-x_0\|^{2}-\|y-x_0\|^{2})/2\sigma^{2}\}\mid x\big]$. Jensen's
inequality bounds its logarithm below by the posterior mean of the exponent. That mean equals
$\big(\|x\|^{2}-\|y\|^{2}-2(x-y)^{\top}\D(x)\big)/2\sigma^{2}$, which at $y=\D(x)$ is
$\|x-\D(x)\|^{2}/2\sigma^{2}$.
\end{proof}

\begin{corollary}[Orbits of the exact map]
\label{cor:orbits}
Let $x_{k+1}=\M(x_k)$. Then (i) $\M$ has no periodic orbit other than fixed points; (ii) every limit
point of $(x_k)$ is a fixed point, and if the fixed points are isolated, $(x_k)$ converges; (iii)
every attractor is a local maximum of $\ps$.
\end{corollary}
\begin{proof}
(i) By Lemma~\ref{lem:ascent}, $\ps$ strictly increases along any step that moves, so an orbit
cannot return to its starting point. (ii) Summing Lemma~\ref{lem:ascent} over $k$ and using
$\ps\le(2\pi\sigma^{2})^{-d/2}$ gives $\sum_k\|x_{k+1}-x_k\|^{2}<\infty$, so the steps tend to $0$.
By (F1) the orbit is bounded. If $x_{k_j}\to z$, then also $x_{k_j+1}\to z$, and continuity of $\M$
gives $\M(z)=z$. The limit set of a bounded sequence whose steps tend to $0$ is connected, and a
connected set of isolated points is a single point. (iii) If every orbit started in a neighbourhood
$U$ of $x^\ast$ converges to $x^\ast$, then $\ps(x)\le\ps(x^\ast)$ for all $x\in U$, because $\ps$
increases along these orbits.
\end{proof}

\begin{proof}[Proof of Proposition~\ref{prop:meanshift-formal}]
(i) is the first identity of Proposition~\ref{prop:tweedie}, and~\eqref{eq:meanshift} is
\eqref{eq:meanshift-m} with all $m_i=1$. (ii) is Proposition~\ref{prop:stab}. (iii) is
Lemma~\ref{lem:ascent} together with Corollary~\ref{cor:orbits}(i)--(ii).
\end{proof}

\paragraph{Both ends of the scale axis.}
\label{app:ends}
At large $\sigma$ every posterior is nearly the whole data distribution, so the map barely depends
on $x$ and collapses everything to one point. At small $\sigma$ the posterior near a datum is almost
entirely that datum, so the map is nearly constant on a small ball around it, which then holds one
attractor. Both statements are contraction arguments.

\begin{proposition}
\label{prop:ends}
(a) If $\sigma>R/2$, $\M$ has exactly one fixed point, and every orbit converges to it. As
$\sigma\to\infty$ this fixed point tends to $\E_p[x_0]$.
(b) Let $p=\sum _i\mu_i\delta_{x_i}$ with pairwise distances at least $\delta$, let
$B_i=\{x:\|x-x_i\|\le\delta/4\}$, and let
$\epsilon_i=\frac{1-\mu_i}{\mu_i}e^{-\delta^{2}/4\sigma^{2}}$. If $\epsilon_iR\le\delta/4$ and
$\epsilon_iR^{2}<\sigma^{2}$, which holds for all small enough $\sigma$, then $B_i$ contains exactly
one fixed point. It is an attractor, and it lies within $\epsilon_iR$ of $x_i$.
\end{proposition}
\begin{proof}
(a) By (F2), $\|\nabla\M(x)\|\le R^{2}/4\sigma^{2}<1$ everywhere, and by (F1) $\M$ maps $S$ into
itself. By Banach's fixed-point theorem, $\M$ has a unique fixed point $x^\ast_\sigma$ in $S$. Every
orbit enters $S$ after one step, so every orbit converges to it, and by (F1) there are no fixed
points outside $S$. As $\sigma\to\infty$, $\pi_x\to p$ uniformly in $x\in S$, so
$x^\ast_\sigma=\E[x_0\mid x^\ast_\sigma]\to\E_p[x_0]$.

(b) Let $x\in B_i$ and $j\ne i$. Then $\|x-x_j\|\ge3\delta/4$ and $\|x-x_i\|\le\delta/4$, so
$\|x-x_j\|^{2}-\|x-x_i\|^{2}\ge\delta^{2}/2$ and
\begin{equation}
\frac{w_j(x)}{w_i(x)}\le\frac{\mu_j}{\mu_i}\,e^{-\delta^{2}/4\sigma^{2}},
\qquad\text{hence}\qquad
\sum_{j\ne i}w_j(x)\le\frac{1-\mu_i}{\mu_i}\,e^{-\delta^{2}/4\sigma^{2}}=\epsilon_i .
\end{equation}
Two bounds follow, using $\Cov[x_0\mid x]\preceq\E[(x_0-x_i)(x_0-x_i)^{\top}\mid x]$ for the second:
\begin{align}
\|\M(x)-x_i\|&=\Big\|\sum_{j\ne i}w_j(x)(x_j-x_i)\Big\|\le\epsilon_iR\le\delta/4,\\
\|\nabla\M(x)\|&\le\sigma^{-2}\Big\|\sum_{j\ne i}w_j(x)(x_j-x_i)(x_j-x_i)^{\top}\Big\|\le\frac{\epsilon_iR^{2}}{\sigma^{2}}<1 .
\end{align}
The first gives $\M(B_i)\subset B_i$, so $\M$ is a contraction of the convex set $B_i$ into itself,
and Banach's theorem gives a unique fixed point there. Its Jacobian has norm below $1$, so it is an
attractor, and by the first bound it lies within $\epsilon_iR$ of $x_i$.
\end{proof}

%----------------------------------------------------------------------
\subsection{The critical scale and mass monotonicity (Sections~\ref{sec:sigc}--\ref{sec:sigc:raise})}
\label{app:critscale}

\begin{proposition}[Critical scale at $K=1$]
\label{prop:k1}
For the exact denoiser, the critical scale with a single iteration is
\begin{equation}
\sigc^{(1)}(x)=\sup\big\{\sigma>0:\ \sigma^{2}\,\|\nabla\log\ps(x)\|\le r\|x\|\big\}.
\label{eq:k1}
\end{equation}
\end{proposition}

\begin{proof}[Proof of Proposition~\ref{prop:k1}]
At $K=1$ the probe is retained at $\sigma$ iff $\|\M(x)-x\|\le r\|x\|$, and by
Proposition~\ref{prop:tweedie} $\|\M(x)-x\|=\sigma^{2}\|\nabla\log\ps(x)\|$. Take the supremum over
the retained scales.
\end{proof}

Two direct consequences of Definition~\ref{def:sigc} are worth recording. $\sigc$ is nonincreasing in
$K$, since retention for $K+1$ steps implies retention for $K$. The retention set need not be an
interval, which is why we verify it empirically (Appendix~\ref{app:remarks}).

\begin{theorem}[Mass monotonicity; formal version of Theorem~\ref{thm:mass}]
\label{thm:mass-formal}
Let $p_\theta=(1-\theta)q+\theta\delta_{x^\ast}$ put mass $\theta$ on $x^\ast$, with $q$ any
distribution of bounded support. Then $\|\D(x^\ast)-x^\ast\|$ is strictly decreasing in $\theta$, and
$\sigc^{(1)}(x^\ast)$ is nondecreasing in $\theta$.
\end{theorem}

Duplicating a datum $m$ times among $N$ others corresponds to $\theta=m/(N+m)$.

\paragraph{Intuition for Theorem~\ref{thm:mass-formal}.}
At the probe, the posterior splits its vote between the probe's own atom and the rest of the data.
Only the rest pulls the probe away, and it pulls with the same force whatever the probe's mass. More
mass on the probe only shrinks the share of the vote the rest receives. The proof makes this exact.

\begin{proof}[Proof of Theorem~\ref{thm:mass-formal}]
Let $D^{q}$ and $\pi^{q}_x$ be the exact denoiser and the posterior of $q$, and set
$c=\varphi_\sigma(0)$ and $B=\int\varphi_\sigma(x^\ast-x_0)\,q(dx_0)>0$. At $x^\ast$ the posterior of
$p_\theta$ is a mixture of the atom and the posterior of $q$:
\begin{equation}
\pi_{x^\ast}=\lambda_\theta\,\delta_{x^\ast}+(1-\lambda_\theta)\,\pi^{q}_{x^\ast},
\qquad
\lambda_\theta=\frac{\theta c}{\theta c+(1-\theta)B}.
\end{equation}
Taking means,
\begin{equation}
\D^{\theta}(x^\ast)-x^\ast=(1-\lambda_\theta)\big(D^{q}(x^\ast)-x^\ast\big).
\label{eq:massdecomp}
\end{equation}
The vector $D^{q}(x^\ast)-x^\ast$ does not depend on $\theta$, and $1-\lambda_\theta$ is strictly
decreasing in $\theta$. Hence the displacement
\begin{equation}
f_\theta(\sigma):=\|\D^{\theta}(x^\ast)-x^\ast\|=(1-\lambda_\theta)\,\|D^{q}(x^\ast)-x^\ast\|
\end{equation}
is strictly decreasing in $\theta$ at every $\sigma$ with $D^{q}(x^\ast)\ne x^\ast$, and zero at every
other $\sigma$. For $\theta<\theta'$ this gives $f_{\theta'}\le f_\theta$, so the $K=1$ retention sets
are nested,
\begin{equation}
\{\sigma:\ f_\theta(\sigma)\le r\|x^\ast\|\}\ \subseteq\ \{\sigma:\ f_{\theta'}(\sigma)\le r\|x^\ast\|\},
\end{equation}
and their suprema, $\sigc^{(1)}(x^\ast)$ by Proposition~\ref{prop:k1}, are nondecreasing in $\theta$.
\end{proof}

\begin{lemma}[Strict mass monotonicity]
\label{lem:strict}
In Theorem~\ref{thm:mass-formal}, suppose $x^\ast\ne0$ and $\bar\sigma=\sigc^{(1)}(x^\ast)$ under $\theta$
is finite. Then $\sigc^{(1)}(x^\ast)$ under any $\theta'>\theta$ is strictly larger than $\bar\sigma$.
\end{lemma}
\begin{proof}
Let $\tau=r\|x^\ast\|>0$. The retention set is closed because $f_\theta$ is continuous, and every
$\sigma>\bar\sigma$ escapes, so
\begin{equation}
f_\theta(\bar\sigma)\le\tau
\quad\text{and}\quad
f_\theta(\sigma)>\tau\ \ (\sigma>\bar\sigma)
\qquad\Longrightarrow\qquad
f_\theta(\bar\sigma)=\tau>0 .
\end{equation}
In particular $D^{q}(x^\ast)\ne x^\ast$ at $\bar\sigma$, and \eqref{eq:massdecomp} gives
\begin{equation}
f_{\theta'}(\bar\sigma)=\frac{1-\lambda_{\theta'}}{1-\lambda_\theta}\,f_\theta(\bar\sigma)<\tau .
\end{equation}
By continuity $f_{\theta'}<\tau$ on an interval above $\bar\sigma$.
\end{proof}

%----------------------------------------------------------------------
\subsection{Two atoms (Theorem~\ref{thm:twoatom})}
\label{app:twoatom}

\begin{theorem}[Two atoms; formal version of Theorem~\ref{thm:twoatom}]
\label{thm:twoatom-formal}
Let $p=\theta\delta_{+a}+(1-\theta)\delta_{-a}$ with $a=d/2$ and $\theta\le\tfrac12$: the probe
carries mass $\theta$ and its only neighbour sits at distance $d$. The probe carries its own attractor
iff $\sigma<\sigc(\theta)$, where $\sigc$ is strictly increasing on $(0,\tfrac12]$,
$\sigc(\tfrac12)=d/2$, and
\[
\sigc(\theta)=\frac{d}{\sqrt{2\log(1/\theta)}}\,(1+o(1))\qquad\text{as }\theta\to0 .
\]
\end{theorem}

\paragraph{Intuition.}
With two atoms the map lives on the line through them. There it is a sigmoid of the position, and the
sigmoid's slope is the separation over the noise, $\beta=a^{2}/\sigma^{2}$, while its offset is the
log-mass ratio $h$. A steep sigmoid crosses the diagonal three times: one attractor at each atom and
a saddle between them. As $\sigma$ grows the sigmoid flattens, and the crossing near the lighter atom
disappears by merging with the saddle. The lighter atom survives while the slope beats the offset,
which is why separation (through $\beta$) enters linearly and mass (through $h$) only logarithmically.

\begin{proof}[Proof of Theorem~\ref{thm:twoatom-formal}]
Place the atoms at $\pm ae$ for a unit vector $e$, with the probe at $+ae$ carrying mass
$\theta\le\frac12$.

\emph{Step 1 (reduction to a line).} By (F1) all fixed points lie on the segment $[-ae,ae]$. Write
$x=aue+x_\perp$ with $x_\perp\perp e$. Then
\begin{equation}
\|x\mp ae\|^{2}=\|x_\perp\|^{2}+a^{2}(u\mp1)^{2},
\qquad
\|x+ae\|^{2}-\|x-ae\|^{2}=4a^{2}u ,
\end{equation}
so the factor depending on $x_\perp$ cancels from the weights of \eqref{eq:meanshift-m}, and
\begin{equation}
\frac{w_+(x)}{w_-(x)}=\frac{\theta}{1-\theta}\,e^{2a^{2}u/\sigma^{2}}=e^{2(\beta u+h)},
\qquad
\M(x)=a\,(w_+-w_-)\,e=a\tanh(\beta u+h)\,e,
\end{equation}
with
\begin{equation}
\beta=\frac{a^{2}}{\sigma^{2}},\qquad h=\frac12\log\frac{\theta}{1-\theta}\le0 .
\end{equation}
$\M$ depends on $x$ only through $u$, and $\nabla\M$ vanishes orthogonally to $e$. Fixed points and
their stability are therefore those of the scalar map $g(u)=\tanh(\beta u+h)$ on $[-1,1]$.

\emph{Step 2 (fixed points).} For $|u|<1$,
\begin{equation}
u=g(u)\iff H(u)=h,\qquad H(u)=\operatorname{arctanh}u-\beta u,\qquad H'(u)=\frac{1}{1-u^{2}}-\beta ,
\end{equation}
and $H\to\mp\infty$ as $u\to\mp1$. If $\beta\le1$, $H$ is increasing and there is exactly one fixed
point. If $\beta>1$, let $u_\beta=\sqrt{1-1/\beta}$ and
\begin{equation}
h^\ast(\beta):=\beta u_\beta-\operatorname{arctanh}u_\beta>0 .
\end{equation}
Then $H$ increases on $(-1,-u_\beta)$ up to $h^\ast(\beta)$, decreases on $(-u_\beta,u_\beta)$ down to
$-h^\ast(\beta)$, and increases on $(u_\beta,1)$.

\emph{Step 3 (stability).} At a fixed point $\tanh(\beta u+h)=u$, so
\begin{equation}
g'(u)=\beta(1-u^{2})\ge0,\qquad g'(u)<1\iff H'(u)>0 .
\end{equation}
Fixed points on the increasing branches of $H$ are attractors, and a fixed point on the decreasing
branch is repelling: it is the saddle between the two basins.

\emph{Step 4 (the probe's attractor).} On $(u_\beta,1)$, $H$ takes every value in $(-h^\ast,\infty)$
exactly once. Since $h\le0$, the probe carries its own attractor iff
\begin{equation}
\beta>1\quad\text{and}\quad h^\ast(\beta)>|h| .
\label{eq:twoatom-cond}
\end{equation}
Otherwise the unique fixed point has $u\le0$, on the neighbour's side or at the midpoint.

\emph{Step 5 (the critical scale).} Using $1/(1-u_\beta^{2})=\beta$,
\begin{equation}
\frac{dh^\ast}{d\beta}=u_\beta+\beta u_\beta'-\frac{u_\beta'}{1-u_\beta^{2}}=u_\beta>0,
\qquad h^\ast(1)=0,\qquad h^\ast(\infty)=\infty .
\end{equation}
So $\sigma\mapsto h^\ast(a^{2}/\sigma^{2})$ decreases strictly from $\infty$ at $\sigma\to0$ to $0$ at
$\sigma=a$, and \eqref{eq:twoatom-cond} holds iff $\sigma<\sigc(\theta)$, where $\sigc(\theta)$ is the
unique solution of
\begin{equation}
h^\ast\!\big(a^{2}/\sigc(\theta)^{2}\big)=|h(\theta)|=\frac12\log\frac{1-\theta}{\theta}.
\end{equation}
The right-hand side decreases strictly on $(0,\frac12]$, so $\sigc(\theta)$ increases strictly, and
$h(\frac12)=0$ gives $\sigc(\frac12)=a=d/2$.

\emph{Step 6 (asymptotics).} As $\beta\to\infty$,
\begin{align}
u_\beta&=1-\frac{1}{2\beta}+O(\beta^{-2}),\qquad
\operatorname{arctanh}u_\beta=\frac12\log\frac{1+u_\beta}{1-u_\beta}=\frac12\log(4\beta)+O(\beta^{-1}),
\\
h^\ast(\beta)&=\beta-\tfrac12\log(4\beta)-\tfrac12+O(\beta^{-1}).
\label{eq:hstar}
\end{align}
As $\theta\to0$, $|h|=\frac12\log(1/\theta)+O(\theta)\to\infty$, so $\beta_{\mathrm c}\to\infty$ and
\eqref{eq:hstar} gives
\begin{equation}
\beta_{\mathrm c}=\tfrac12\log(1/\theta)\,(1+o(1)),
\qquad
\sigc=\frac{a}{\sqrt{\beta_{\mathrm c}}}=\frac{d}{\sqrt{2\log(1/\theta)}}\,(1+o(1)).
\qedhere
\end{equation}
\end{proof}

\paragraph{Remarks.}
(i) For $\theta>\frac12$ the atoms swap roles, and bistability ends at the same scale:
$\sigc(\theta)=\sigc(1-\theta)$. (ii) The $\log\beta$ term in \eqref{eq:hstar} makes the exact
$\sigc$ smaller than the leading-order formula: solving Step~5 numerically gives $0.74$, $0.83$ and
$0.87$ of the formula's value at $\theta=10^{-2},10^{-4},10^{-6}$. So the formula of Theorem~\ref{thm:twoatom} gives how
$\sigc$ scales (linearly in $d$, logarithmically in $\theta$), not its value. (iii) At $\sigc$ the
probe's attractor and the saddle merge at $u_{\beta_{\mathrm c}}$ and both disappear: a saddle-node.
The attractor is still at distance $a(1-u_{\beta_{\mathrm c}})$ from the probe when this happens.

Theorem~\ref{thm:twoatom-formal} locates the scale at which an attractor ceases to exist, while
Definition~\ref{def:sigc} thresholds a displacement. The next proposition says the two agree
whenever the attractor is still within the escape radius when it vanishes.

\begin{proposition}[Escape test in the two-atom model]
\label{prop:escape}
In the setting of Theorem~\ref{thm:twoatom-formal}, let $K=\infty$ and let the escape radius
$\tau=r\|x\|$ satisfy $\tau<a$. Then the probe is retained exactly for
$\sigma\in(0,\sigma_{\mathrm{esc}}]$, with $\sigma_{\mathrm{esc}}\le\sigc(\theta)$. Equality holds iff
$a(1-u_{\beta_{\mathrm c}})\le\tau$, where $\beta_{\mathrm c}=a^{2}/\sigc(\theta)^{2}$.
\end{proposition}
\begin{proof}
$g$ is increasing and $g(1)<1$, so the orbit started at $u=1$ decreases monotonically to the largest
fixed point $u^\ast(\sigma)\le1$. Its largest displacement from the probe is its limit, so
\begin{equation}
\sup_k\|\M^{k}(x)-x\|=a\big(1-u^\ast(\sigma)\big),
\qquad
\text{the probe is retained}\iff a\big(1-u^\ast(\sigma)\big)\le\tau .
\end{equation}
For $\sigma\le\sigc$, $u^\ast$ is the probe's attractor, the root of $H(u)=h$ on the rising branch.
As $\sigma$ grows, $H$ increases pointwise on $u>0$, so this root moves continuously left, down to
$u_{\beta_{\mathrm c}}$ at $\sigc$. For $\sigma>\sigc$, $u^\ast\le0$ and the displacement is at least
$a>\tau$. So the retention set is
\begin{equation}
\big\{\sigma\le\sigc:\ a\big(1-u^\ast(\sigma)\big)\le\tau\big\},
\end{equation}
an interval, which reaches $\sigc$ iff the displacement at $\sigc$ is at most $\tau$.
\end{proof}

Since $1-u_{\beta_{\mathrm c}}\approx1/(2\beta_{\mathrm c})\approx1/\log(1/\theta)$, the condition
holds for any fixed radius once $\theta$ is small. The escape test then returns the saddle-node
scale. This is a statement about the exact map, not about a trained network.

%----------------------------------------------------------------------
\subsection{From two atoms to a training set}
\label{app:balance}

\paragraph{Intuition.}
Remark~(ii) and~\eqref{eq:hstar} say that, at leading order, the lighter atom keeps its attractor
while it wins the posterior vote at its own location, $\beta>|h|$. For a whole training set the same
criterion reads: the probe's $m$ copies must outweigh the kernel mass of all other rows at the probe,
\begin{equation}
S(\sigma):=\textstyle\sum_jm_j\,e^{-\|x-x_j\|^{2}/2\sigma^{2}}<m ,
\label{eq:balance}
\end{equation}
the sum running over the other rows. Each term of $S$ increases with $\sigma$, so the balance
holds on an interval $(0,\sigma_{\mathrm{bal}})$. This criterion is a heuristic. It is leading-order,
it looks only at the probe's own location, and it has no escape threshold. What can be proved is how
the closed forms of Section~\ref{sec:sigc:raise} relate to it.

\begin{proposition}
\label{prop:balance}
Let $n>m$ be the number of other rows, $d_{\mathrm{nn}}$ the distance to the nearest of them, and
$\sigma_1$ the solution of $S(\sigma_1)=1$. Set $d_{\mathrm{kern}}=\sigma_1\sqrt{2\log n}$. Then
\begin{equation}
\sigma_{\mathrm{bal}}\ \ge\ \frac{d_{\mathrm{nn}}}{\sqrt{2\log(n/m)}}
\qquad\text{and}\qquad d_{\mathrm{kern}}\ \ge\ d_{\mathrm{nn}},
\end{equation}
with equality in both iff all other rows are at distance $d_{\mathrm{nn}}$. Moreover, a single atom of
$n$ rows at distance $d_{\mathrm{kern}}$ has the same kernel mass as the actual rows at $\sigma_1$.
\end{proposition}
\begin{proof}
Every other row is at distance at least $d_{\mathrm{nn}}$, so
\begin{equation}
S(\sigma)\le n\,e^{-d_{\mathrm{nn}}^{2}/2\sigma^{2}},
\end{equation}
with equality iff all of them are at distance $d_{\mathrm{nn}}$. The three claims follow from this
bound:
\begin{align}
n\,e^{-d_{\mathrm{nn}}^{2}/2\sigma^{2}}<m
&\iff\sigma<\frac{d_{\mathrm{nn}}}{\sqrt{2\log(n/m)}}
&&\Longrightarrow\ \ \sigma_{\mathrm{bal}}\ge\frac{d_{\mathrm{nn}}}{\sqrt{2\log(n/m)}},\\
1=S(\sigma_1)\le n\,e^{-d_{\mathrm{nn}}^{2}/2\sigma_1^{2}}
&\iff d_{\mathrm{nn}}\le\sigma_1\sqrt{2\log n}=d_{\mathrm{kern}},\\
n\,e^{-d_{\mathrm{kern}}^{2}/2\sigma_1^{2}}&=1=S(\sigma_1). &&\qedhere
\end{align}
\end{proof}

So~\eqref{eq:ceiling}, with $n\approx N$, is a lower bound on the balance scale: lumping every row
at the nearest distance overcounts the competition. Lumping them at $d_{\mathrm{kern}}$ instead gives
$d_{\mathrm{kern}}/\sqrt{2\log(n/m)}$, which is exact at $m=1$ and approximate otherwise. In words,
$d_{\mathrm{kern}}$ is the distance at which one atom carrying all other rows would weigh as much as
the actual neighbours do against a single copy. Because the balance itself is only a heuristic, the
ceiling $\sigemp$ is not taken from either formula. It is computed by running \eqref{eq:meanshift-m}
through the network's own escape test (Appendix~\ref{app:exact}).

%----------------------------------------------------------------------
\subsection{Conditional models and guidance (Section~\ref{sec:cond})}\label{app:cond}

No argument above uses that $p$ is a marginal distribution, so every result holds for
$p(\cdot\mid c)$. For $p(x_0\mid c)=\sum_iw_i(c)\delta_{x_i}$, \eqref{eq:meanshift-m} with
$\mu_i=w_i(c)$ gives the caption-reweighted mean shift
\begin{equation}
\M(x;c)=\sum_i\tilde w_i(x,c)\,x_i,\qquad
\tilde w_i(x,c)\propto w_i(c)\exp\!\big(-\|x-x_i\|^{2}/2\sigma^{2}\big).
\label{eq:condmeanshift}
\end{equation}
The Gaussian kernel, which carries $\sigma$, is unchanged, and the caption enters only through the
prior weights $w_i(c)$.

\begin{corollary}[Caption gap; formal version of Corollary~\ref{cor:capgap}]
\label{cor:capgap-formal}
At $K=1$, and when the caption changes the weight of $x_i$ but not the relative weights of the other
examples, $\Delta\log\sigc(x_i;c)>0$ iff $w_i(c)>w_i(\varnothing)$.
\end{corollary}

\paragraph{Intuition for Corollary~\ref{cor:capgap-formal}.}
A caption acts on the empirical mean shift only through the prior weights, exactly as a duplication
count does. If it raises the weight of $x_i$ and leaves the rest of the data as it was, it is a
duplication of $x_i$, and Theorem~\ref{thm:mass-formal} applies.

\begin{proof}[Proof of Corollary~\ref{cor:capgap-formal}]
By assumption, $w_j(c)/w_j(\varnothing)$ is the same for every $j\ne i$, so the normalized remainder
is the same under both conditions:
\begin{equation}
p(\cdot\mid c')=w_i(c')\,\delta_{x_i}+\big(1-w_i(c')\big)\,q_i,\quad c'\in\{c,\varnothing\},
\qquad
q_i=\frac{\sum_{j\ne i}w_j(\varnothing)\,\delta_{x_j}}{1-w_i(\varnothing)} .
\end{equation}
These are the family $p_\theta$ of Theorem~\ref{thm:mass-formal} at $\theta=w_i(c)$ and at
$\theta=w_i(\varnothing)$. By Lemma~\ref{lem:strict}, applied in whichever direction the weight moves,
\begin{equation}
\sigc^{(1)}(x_i;c)>\sigc^{(1)}(x_i;\varnothing)\iff w_i(c)>w_i(\varnothing),
\end{equation}
when both are finite.
\end{proof}

Without the assumption, \eqref{eq:massdecomp} still describes the step, but the caption also changes
the pull of the remainder $q_i$.

\paragraph{Guidance.}
With the guided noise prediction $\epsilon_w=\epsilon_\varnothing+w(\epsilon_c-\epsilon_\varnothing)$,
the map becomes
\begin{equation}
\Mg{w}(x)=x+\sigma^{2}\nabla\log\Big[\ps(x)^{1-w}\ps(x\mid c)^{w}\Big].
\label{eq:cfg}
\end{equation}
It climbs the tilted density $\ps(x)\big(\ps(x\mid c)/\ps(x)\big)^{w}$, whose critical points for
$w>1$ are not those of $\ps(\cdot\mid c)$. Only $\Mg{1}$ describes the fixed points of the
conditional model.

\begin{lemma}[Guided map]
\label{lem:cfg}
The guided map is~\eqref{eq:cfg}, and its Jacobian is
\begin{equation}
\nabla\Mg{w}(x)=\frac{(1-w)\Cov[x_0\mid x]+w\Cov[x_0\mid x,c]}{\sigma^{2}} .
\end{equation}
For $0\le w\le1$ it is positive semidefinite; for $w>1$ it is a difference of positive semidefinite
matrices and may be indefinite.
\end{lemma}

\begin{proof}[Proof of Lemma~\ref{lem:cfg}]
Both denoisers are affine in the noise prediction, $\D(x)=x-\sigma\epsilon_\varnothing$ and
$\D(x;c)=x-\sigma\epsilon_c$, so guidance with $\epsilon_w=(1-w)\epsilon_\varnothing+w\epsilon_c$ gives
\begin{equation}
\Mg{w}=(1-w)\,\D(\cdot)+w\,\D(\cdot\,;c).
\end{equation}
Applying Proposition~\ref{prop:tweedie} to each term gives the map and its Jacobian. For
$0\le w\le1$ the Jacobian is a convex combination of positive semidefinite matrices; for $w>1$ the
coefficient $1-w$ is negative.
\end{proof}

The consequence for the dynamics: for $0\le w\le1$ everything in this appendix carries over to the
tilted density $\ps^{1-w}\ps(\cdot\mid c)^{w}$. For $w>1$ a negative eigenvalue lets the map
overshoot and oscillate, and a maximum of the tilted density is an attractor only if every eigenvalue
of its Jacobian stays above $-1$.

\section{Remarks on the critical scale}
\label{app:remarks}

\begin{figure}[t]
\centering
\includegraphics[width=\linewidth]{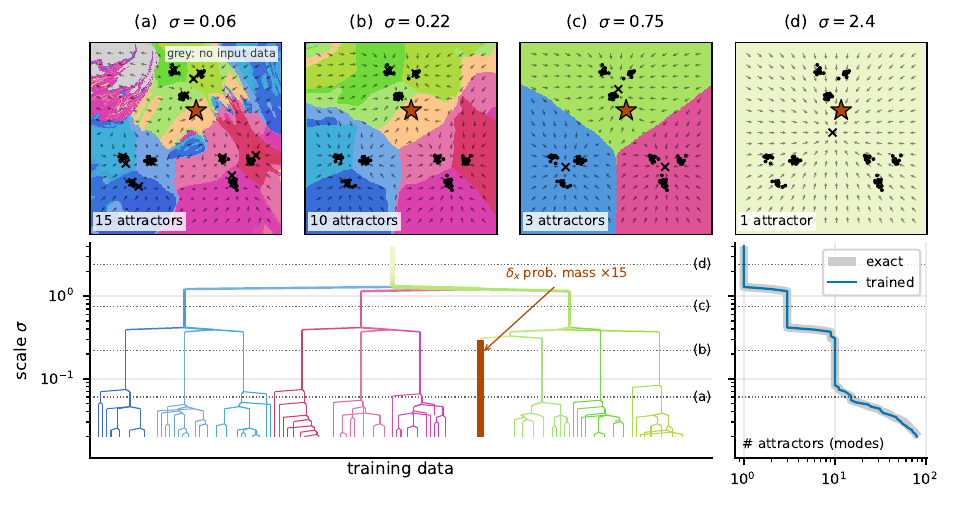}
\caption{\textbf{A trained denoiser reproduces the scale space of Fig.~\ref{fig:scalespace}.} Same
training set and same panels as Fig.~\ref{fig:scalespace}, with the exact map $\M$ replaced by
$\M(x)=\D(x)$ for an EDM-preconditioned MLP trained on the $141$ points. \emph{Top:} basins at the
four scales of Fig.~\ref{fig:scalespace}, coloured by the attractor ($\times$) reached from the data;
grey marks the grid points that reach none of them, in the corners of the domain where the network
saw no training data. \emph{Bottom left:} merge tree of the trained map. \emph{Bottom right:} its
attractor count against that of the exact map. The two maps agree: the counts are equal at $80\%$ of
the $90$ scales, the partitions of the training set they induce have adjusted Rand index $0.99$ on
average ($0.76$ at worst, at the smallest scales where single data split), and the duplicated datum
keeps its attractor over the same range of scales in both.}
\label{fig:scalespacelearned}
\end{figure}

\paragraph{Computing the critical scale.}
We compute $\sigc$ by bisection in $\log\sigma$ on the escape event of Definition~\ref{def:sigc}.
The definition is most natural when the set of retained scales is an interval, i.e.\ when
the escape event is monotone in $\sigma$; in that case bisection returns $\sigc$
exactly. We do not assume monotonicity but verify it on the trained models, and observe no violation
on either model (on Stable Diffusion, none for $16$ images over the noise grid). For $K>1$ the escape
event depends on the whole orbit, and $\sigc$ has no closed form analogous to
Proposition~\ref{prop:k1}.

\paragraph{Related iterations.}
The map of Definition~\ref{def:map} feeds the clean estimate back to the denoiser without adding
noise, so that it is deterministic and noise cannot create spurious endpoints. Two related
iterations have different fixed points and are not equivalent to $\M$. In the renoising iteration
$x_{k+1}=\D(x_k+\sigma\varepsilon_k)$ the endpoints depend on the noise realisation, and the round trip
$x_{k+1}=\mathrm{Sample}_{\sigma\to0}(x_k+\sigma\varepsilon_k)$ is a generative operator.

%======================================================================
\section{The exact denoiser and the retention coefficient}
\label{app:exact}

The theorems of Section~\ref{sec:sigc:raise} describe the exact denoiser of a fixed distribution.
The exact denoiser of a training set, $\M^{\mathrm{emp}}$, is the multiplicity-weighted mean shift
\eqref{eq:meanshift-m} over the number of samples: the map of a model that has stored its training set exactly.
It costs one softmax over the training set per step, so it can be iterated from the same images and
scored with the same escape test as the network (settings in Appendix~\ref{app:details}). We write
\begin{equation}
\sigemp(x)\;=\;\sigc(x)\ \text{computed with }\M^{\mathrm{emp}}\text{ in place of the network, at the same }K\text{ and }r .
\label{eq:sigemp}
\end{equation}
Section~\ref{app:exact:theory} checks the theory on this map; Section~\ref{app:coef} compares trained
networks with it.

\subsection{The theory on the exact denoiser}
\label{app:exact:theory}

\paragraph{A bifurcation, not a threshold.}
For this map a training image is an exact fixed point at every scale below the transition, so its
critical scale is the saddle-node of Theorem~\ref{thm:twoatom}: on the
duplication training set $\sigemp$ is unchanged between escape radii $r=0.1$ and $r=0.5$ and between
$K=20$ and $K=500$ iterations (largest group-median difference $1\%$).

\paragraph{Distance and multiplicity.}
With the nearest training image as the only neighbour, Theorem~\ref{thm:twoatom} gives, for an
example of multiplicity $m$ among $N$ rows,
\begin{equation}
\sigemp(x)\;\approx\;\frac{d_{\mathrm{nn}}(x)}{\sqrt{2\log(N/m)}},\qquad
d_{\mathrm{nn}}(x)=\min_{x_i\ne x}\|x-x_i\| ,
\label{eq:ceiling}
\end{equation}
refined by the balance \eqref{eq:balance} of Appendix~\ref{app:balance}. The critical scale should thus
grow linearly with distance and only logarithmically with multiplicity. The exact denoiser follows
this law without fitting (Figure~\ref{fig:outexact}): doubling the kernel distance doubles $\sigemp$,
whereas $32$ copies increase it by a factor $1.26$. The log--log slope in distance is $1.03$
$[1.00,1.05]$ for one copy and $0.88$ $[0.79,0.93]$ for $32$ copies, against the predicted $1$, and a
joint fit over all $288$ images of the rarity experiment gives $1.95$ per doubling of distance against
$1.06$ per doubling of multiplicity, where the law predicts $2$ and $1.05$.

\paragraph{Which distance the theory means.}
Solving the balance \eqref{eq:balance} image by image predicts $\sigemp$ at rank correlation
$0.99$--$1.00$ at every multiplicity of the rarity experiment, with an error of $7$--$10\%$, and on the
duplication training set at $\rho=0.99$ with a constant factor $0.83$. The nearest-neighbour formula
\eqref{eq:ceiling} reaches $\rho=0.87$--$0.95$ with an error of $43\%$, and lies below the computed
value, as Proposition~\ref{prop:balance} requires. The gaussian kernel distance $d_{\mathrm{kern}}$, is the notion of distance the theory uses.

\begin{figure}[t]
\centering
\includegraphics[width=0.52\linewidth]{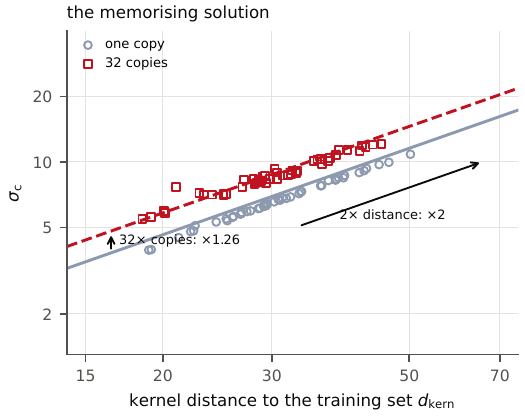}
\caption{\textbf{The exact denoiser follows the law.} $\sigemp$ of the added images of the rarity
experiment against their kernel distance $d_{\mathrm{kern}}$ to the training set
(Eq.~\ref{eq:balance}), for images added once (grey) and $32$ times (red), under the escape test with
$K=1$, $r=0.2$. The lines are the law $\sigc=d_{\mathrm{kern}}/\sqrt{2\log(N/m)}$ of
Theorem~\ref{thm:twoatom}, drawn without fitting. Doubling the distance doubles the critical scale,
whereas $32$ copies increase it by a factor $1.26$.}
\label{fig:outexact}
\end{figure}

\subsection{Trained networks against the exact denoiser}
\label{app:coef}

\paragraph{The retention coefficient.}
A network should not hold an image longer than the memorizing solution does, which suggests reporting
\begin{equation}
\reten(x)\;=\;\frac{\sigc(x)}{\sigemp(x)} ,
\label{eq:coef}
\end{equation}
the fraction of that solution's retention the network reaches at $x$: near $0$ where the model
generalizes through $x$, approaching $1$ where it has stored $x$. Both terms are measured with the
same test and the same threshold $r\|x\|$. A held-out image has no atom in the training set and
leaves the exact map at the first step at every $\sigma$, so $\reten$ is defined for training images
only.

\paragraph{The bound holds.}
We use the $1047$ training images of Section~\ref{sec:evidence}: $264$ from the duplication experiment,
$495$ from the overfitting experiment (five subsets of $100$, excluding five images that escape at the
lower end of the interval), and $288$ from the rarity experiment. No image has $\reten>1$. The largest
values are $0.78$ (duplication), $0.82$ (overfitting) and $0.65$ (rarity). The memorizing solution
therefore bounds every trained model we measure, and $\sigc$ can be read as a position between a model
that generalizes through the image and one that has stored it.

\paragraph{What moves, and what does not.}
Across the three interventions the ceiling $\sigemp$ stays within a factor $4.2$ per image and between
$4.9$ and $10.3$ in group median, whereas $\sigc$ varies by a factor $335$ (Table~\ref{tab:coef}).
Each intervention raises $\reten$, from $0.01$--$0.06$ for images that are not memorized to
$0.2$--$0.6$ for images that are. The training set fixes the scale on which retention is measured, and
training decides how far along it the model goes; this is why the overfitting sweep moves $\sigc$ by
two orders of magnitude while the exact denoiser of its subsets barely moves.

\begin{table}[t]
\centering\small
\caption{The three interventions on a common scale. $\sigemp$ is the critical scale of the exact
empirical denoiser of the training set of each model under the same escape test,
~\eqref{eq:sigemp}; $\reten=\sigc/\sigemp$ is the fraction of this solution that the network
implements. Group medians, training images only: the ceiling varies little, and the effect is carried
by $\reten$.}
\label{tab:coef}
\begin{tabular}{@{}llccc@{}}
\toprule
intervention & group & $\sigemp$ & $\sigc$ & $\reten$ \\
\midrule
duplication & $m=1$                     & $4.99$ & $0.142$ & $0.028$ \\
 & $m=8$                                & $5.60$ & $1.036$ & $0.194$ \\
 & $m=800$                              & $9.91$ & $1.431$ & $0.145$ \\
 & $m=800$, control model               & $9.91$ & $0.269$ & $0.024$ \\
\midrule
overfitting & $N=5\cdot10^4$            & $4.93$ & $0.056$ & $0.012$ \\
 & $N=2000$                             & $6.11$ & $0.243$ & $0.043$ \\
 & $N=100$                              & $7.84$ & $4.674$ & $0.598$ \\
\midrule
rarity & model not trained on the image & $7.68$ & $0.451$ & $0.058$ \\
 & CIFAR-10 added, $m=1$                & $7.28$ & $0.447$ & $0.076$ \\
 & colour MNIST added, $m=1$            & $6.85$ & $1.459$ & $0.182$ \\
 & colour MNIST added, $m=32$           & $10.30$ & $4.024$ & $0.432$ \\
\bottomrule
\end{tabular}\\[2pt]
{\footnotesize The rarity rows use $K=1$, $r=0.2$ and the others the settings of
Appendix~\ref{app:details}, so $\reten$ is comparable within a block but not across blocks.}
\end{table}

\paragraph{The two maps release an image for different reasons.}
On the exact denoiser a training image is a fixed point until the saddle-node, so its critical scale
does not depend on the iteration budget. On a network it does, and the dependence itself separates the
groups: at $m=800$ the median $\sigc$ is $1.431$ at both $K=50$ and $K=500$, whereas for $m=1$ and
held-out images it falls by factors of $165$ and $193$ over the same range. A memorized image is
retained because it is a fixed point of the network's map; an ordinary image is retained only
transiently, and a longer iteration releases it. Consequently $\reten$ compares two escape scales
measured the same way, not two maps: at a fixed budget it is a conservative reading of how much of the
memorizing solution the network implements, and for images without a basin it is an upper bound rather
than an estimate.

%======================================================================
\section{CIFAR-10: additional results}
\label{app:cifar}

This appendix reports additional results for the experiments of Section~\ref{sec:evidence}. Training and
evaluation details are given in Appendix~\ref{app:details}.

\subsection{Memorization from data duplication (Section\ref{sec:cifar:dup})}
\label{app:cifar:dup}

\paragraph{Copying and sample quality.}
With the copy criterion of \citet{carlini2023extracting}, $16\,605$ of $50\,000$ samples are copies.
The FID increases from $3.90$ of the control model  to $4.70$ for non  duplicated samples (when this are counted in the FID computation FID goes to $26.46$). Therefore aside for the planted memorized images the retained model still perform well on quality of generations.

\paragraph{Membership.}
$\sigc$ separates training images that appear once from held-out (test set) images with AUC $0.528$, i.e.\ it does
not detect membership in this case: this is due to the fact that not accounting for the memorized content the models have similar FID and fit well the CIFAR data distribution therefore there is few difference between training and test distribution. This is consistent with a model that copies neither group, and agrees with the
overfitting experiment at $N=5\cdot10^{4}$. Figure~\ref{fig:orbits3} illustrates this on one image
per group.

\begin{figure}[t]
\centering
\includegraphics[width=\linewidth]{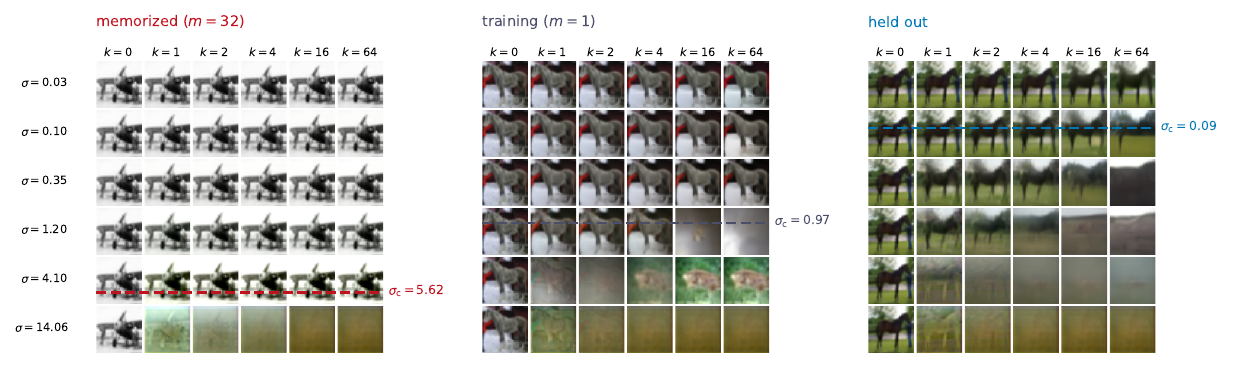}
\caption{\textbf{Orbits across scales, one image per group.} Orbits of the model trained with
duplicates, with $\sigma$ fixed along each row and the iteration index $k$ along each column, for a
memorized image ($m=800$), a training image with $m=1$ and a held-out image, each chosen at the median
$\sigc$ of its group. The dashed line marks the $\sigc$ of the image ($K=64$, $r=0.5$). The memorized
image is unchanged after $64$ iterations at $\sigma=1.43$, whereas the other two converge to a uniform
colour at $\sigma=0.24$, with nearly identical critical scales ($0.122$ and $0.118$).}
\label{fig:orbits3}
\end{figure}

\paragraph{Dependence on multiplicity.}
The median $\sigc$ relative to its value at $m=8$ is:
\begin{center}\small
\begin{tabular}{@{}lccccccc@{}}
\toprule
$m$ & $8$ & $25$ & $50$ & $100$ & $200$ & $400$ & $800$\\
\midrule
network                   & $1.00$ & $1.12$ & $1.18$ & $1.34$ & $1.24$ & $1.26$ & $1.38$\\
exact empirical denoiser  & $1.00$ & $1.11$ & $1.08$ & $1.23$ & $1.19$ & $1.29$ & $1.77$\\
\bottomrule
\end{tabular}
\end{center}
A hundredfold increase in the number of copies changes the critical scale by less than a factor $1.4$,
and the memorizing solution of the same training set saturates in the same way: mass enters the width
of a basin only logarithmically, so once an image has a basin, further copies cannot widen it much.
Copies still change what the sampler does, as the copy rate shows. Accordingly $\reten$ does not order
the duplicated images by multiplicity ($0.028$ at $m=1$, $0.19$ at $m=8$, $0.15$ at $m=800$;
$\rho=-0.01$), but it separates duplicated from non-duplicated images, and in the control model it lies
between $0.014$ and $0.037$ at every $m$.

\subsection{From memorization to generalization: memorization from overfitting}
\label{app:cifar:n}

\begin{figure}[ht]
\centering
\includegraphics[width=\linewidth]{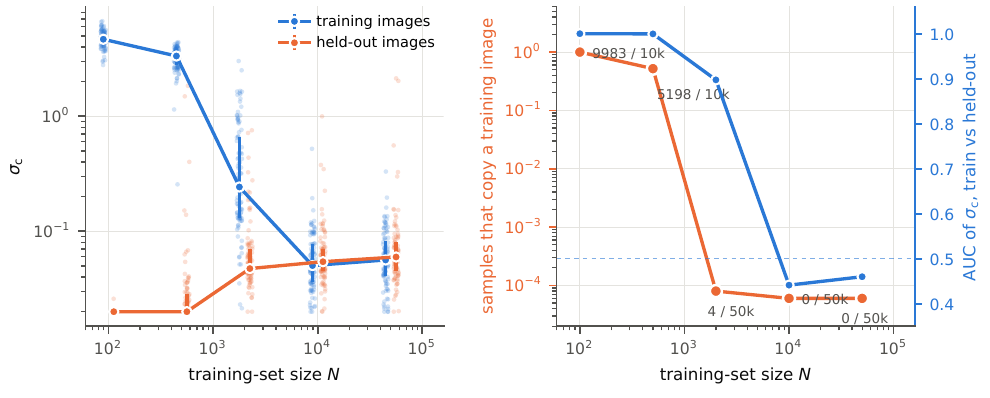}
\caption{\textbf{Overfitting: the same images across training-set sizes.} \emph{(a)} $\sigc$ of the
$100$ training images shared by all models (blue) and of $100$ held-out images (orange): medians with
interquartile range, individual images in the background. Images that escape at the lower end of the bisection interval are drawn
at that end, so the held-out medians for $N\le500$ are upper bounds. \emph{(b)} Fraction of generated
samples that are copies (orange, left axis; when no copy is found the point is drawn at the $95\%$ upper bound)
and AUC of $\sigc$ between training and held-out images (blue, right axis). At $N=2000$ sampling no
longer detects memorization while $\sigc$ still does.}
\label{fig:sigman}
\end{figure}

\paragraph{Setup.}
We study memorization induced by training on a small dataset, and compare $\sigc$ with memorization
measured by sampling. We train five models with identical architecture, optimiser and training
budget on nested subsets of CIFAR-10 of size $N\in\{100,\,500,\,2000,\,10^4,\,5\cdot10^4\}$ inducing a transition from a memorization (overfitting) to generalization phase  (\cite{fumero2026navigating,kadkhodaie2024generalization,pham2025memorization}). Since
only $N$ varies, a smaller $N$ means more passes over each image. The first $100$ images of the
permutation belong to every subset, so $\sigc$ is evaluated on the same $100$ training images in every
model, together with $100$ held-out test images. Independently of $\sigc$, we measure memorization as
the fraction of generated samples that are near-copies of a training image. We also evaluate the
public $50$M improved-diffusion model on the same images.

\paragraph{Results.}
The median $\sigc$ of the training images decreases by two orders of magnitude, from $4.67$ at $N=100$
to $0.056$ at $N=5\cdot10^4$, where it coincides with that of the held-out images ($0.060$;
Figure~\ref{fig:sigman}a). The decrease is not a power law: $\sigc$ is nearly constant up to
$N=500$, collapses between $N=500$ and $N=2000$, and is flat afterwards. At both ends $\sigc$ agrees
with sampling. For $N\le500$, $52$--$100\%$ of the samples are copies and $\sigc$ separates training
from held-out images with AUC $1.00$; for $N\ge10^4$ and for the public model, no sample is a copy and
$\sigc$ is at the held-out level. The two measures differ at $N=2000$: only $4$ of $5\cdot10^4$
samples are copies, yet $\sigc$ separates training from held-out images with AUC $0.897$
(Figure~\ref{fig:sigman}b).

Every image here appears once, so the training set alone changes little across the sweep: what
changes by two orders of magnitude is how closely the trained model reproduces the memorizing
solution at these images (Appendix~\ref{app:coef}).

\paragraph{Summary.}
Retention is lost later than generation: a model can keep its training images as attractors after
its sampler has stopped producing them. Since $\sigc$ measures retention, it detects memorization in a
regime in which the copy rate is already zero.

\paragraph{Comparison to full OpenAi checkpoint.}
The public model behaves like our models with $N\ge10^4$ (median $0.068$ against $0.072$, AUC $0.45$).
This is consistent with sampling: an exhaustive search finds no extractable training image, and the
number of copies remains at the held-out level up to $3\cdot10^5$ samples. The images that sampling
does recover from this model belong to clusters of near-duplicates in CIFAR-10, for which $\sigc$ is at
the held-out level.

\paragraph{Exact denoiser.}
Every image has $m=1$, so \eqref{eq:twoatom-cond} depends on $N$ only through $\log N$ and predicts
$\sigc(100)/\sigc(5\cdot10^4)=1.53$, against a measured ratio of $83$. The exact empirical denoiser of
each subset agrees with the formula: its critical scale decreases from $7.84$ to $4.93$, a factor of
$1.6$. Held-out images, which have no atom, leave the exact map immediately ($91$ to $100$ of $100$ at
every $N$), so their retention by the network is due to generalization.

\subsection{Memorization of rare samples (Section\ref{sec:cifar:out})}
\label{app:cifar:out}

\paragraph{Detection by sampling.}
No added image with $m\le4$ is copied more than $0.25$ times per $10^5$ samples for any source. At
$m=8$ the number of copies per image is between $1.3$ and $13.6$, and for $m\ge16$ every added image is
generated. For colour MNIST, $\sigc$ flags $8$ of $12$ images with a single copy and all $12$ with two.

\paragraph{Paired comparison.}
The median ratio between $\sigc$ in the model trained on the image and $\sigc$ in the other model
increases from $1.02$ (CIFAR-10, $m=1$) through $2.5$ (SVHN, $m=4$) and $3.0$ (MNIST, $m=1$) to between
$3.4$ and $8$ for $m\ge8$. For the two CIFAR sources with $m\le4$, the unpaired AUC against images of the
same source is between $0.36$ and $0.66$: these images cannot be detected in a single model, but they
are detected by the paired comparison. Figure~\ref{fig:crossover} shows three added images that are
never sampled, one per source, chosen at the median of their group.

\begin{figure}[t]
\centering
\includegraphics[width=\linewidth]{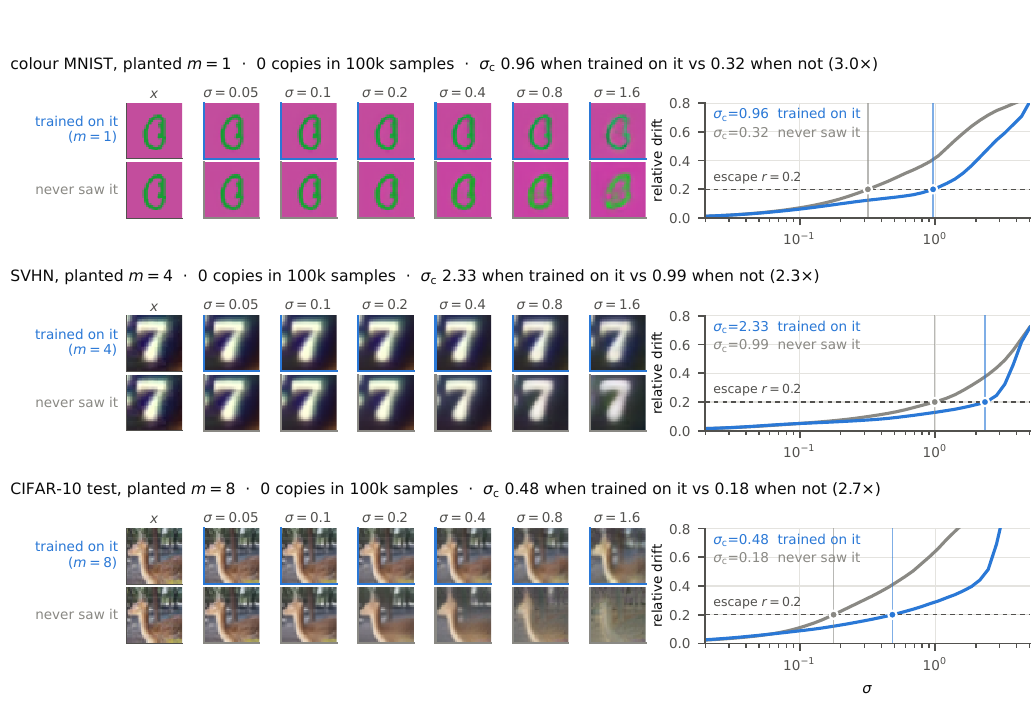}
\caption{\textbf{Paired comparison on three images that are never sampled.} One added image per
source, chosen at the median of its group. Each row shows $\M(x)$, one denoising step applied to the
clean image at noise level $\sigma$, in the model trained on the image (top, blue) and in the model not
trained on it (bottom, grey), together with the relative displacement $\|\M(x)-x\|/\|x\|$ whose
crossing of $r$ defines $\sigc$. The model not trained on the image moves it towards what it has
learned (blur, generic texture) at a much smaller $\sigma$.}
\label{fig:crossover}
\end{figure}

\paragraph{Isolation.}
Within a source and multiplicity, the increase of $\sigc$ caused by adding an image is correlated with
its distance to the nearest training image (Spearman $\rho=0.33$, $p=5\cdot10^{-5}$, $n=144$ for
$m\le4$). An isolated image does not share its basin at low noise with any training image, so a single
copy suffices to create it; a typical image is absorbed into the solution that the model generalizes
from its neighbours unless $m$ is large. The probability that the sampler reaches the basin from pure
noise, by contrast, grows with $m$ and not with isolation, which explains why the two detection
thresholds differ.

\paragraph{Distance and multiplicity in the trained models.}
The design spans five doublings of multiplicity and less than one doubling of distance, since the
sources were matched in norm and total variation, so the two factors are compared per doubling. A
joint fit over the $288$ added images gives a factor $4.25$ $[3.4,5.3]$ per doubling of kernel
distance against $1.49$ $[1.44,1.54]$ per doubling of multiplicity: in the trained models, as in the
memorizing solution (Appendix~\ref{app:exact:theory}), distance is the \emph{stronger} factor. The network
responds to multiplicity more than the exact denoiser does, which is expected, since each copy also adds training steps on the image. Pixel distances are correlated with $\|x\|$ (see \emph{Norm} below),

\paragraph{Agreement with the exact denoiser.}
Within a source and multiplicity the network orders the images as the exact map does ($\rho=0.78$).
The ordering of the sources is not an effect of the norm: at equal $\|x\|$, $\reten$ is larger than for
CIFAR-100 by a factor $3.1$ for colour MNIST and $2.2$ for SVHN ($t=13.7$ and $9.8$), and by a factor
$1.13$ for CIFAR-10 ($t=1.5$).

\paragraph{Escape radius.}
The radius $r=0.02$ selected in advance on the public model is not suitable for this noise schedule:
the escape occurs at $\sigma\approx0.025$, below the informative range, and all AUCs except that of
colour MNIST are close to $0.5$. The results of Section~\ref{sec:cifar:out} use $r=0.2$, selected from a
sweep over $(K,r)$. The separation is stable for $r\in[0.1,0.5]$ and $K\in[1,20]$.
The paired design determines the direction of the effect independently of
this choice, since the same image is evaluated with the same $r$ in two models.

\section{Experimental details}
\label{app:details}

\paragraph{Noise-prediction models as denoisers.}
DDPM and Stable Diffusion corrupt the data as $x_t=\sqrt{\bar\alpha_t}\,x_0+\sqrt{1-\bar\alpha_t}\,\varepsilon$
at discrete steps $t$ and train $\epsilon_\theta(x_t,t)$ to predict $\varepsilon$. Dividing by
$\sqrt{\bar\alpha_t}$ gives $y=x_0+\sigma\varepsilon$ with $\sigma(t)=\sqrt{(1-\bar\alpha_t)/\bar\alpha_t}$,
and the denoiser is $\D(y)=y-\sigma\,\epsilon_\theta(\sqrt{\bar\alpha_t}\,y,t)$.

\paragraph{CIFAR-10 duplication (Section\ref{sec:cifar:dup}).}
The model with duplicates is trained with the improved-diffusion CIFAR-10 configuration
\citep{nichol2021improved}, using the model and diffusion definitions of the original implementation:
$50$M parameters, $T=4000$ steps with a cosine schedule, learned variances, dropout $0.3$, $128$
channels at three resolutions, learning rate $10^{-4}$, batch size $128$, EMA $0.9999$, no horizontal
flips, and $500$k steps. The control is the released checkpoint \texttt{cifar10\_uncond\_50M\_500K.pt}.
Eight images are duplicated at each $m\in\{2,8,25,50,100,200,400,800\}$ in addition to the $50\,000$
training images, giving $62\,616$ training examples. The multiplicities are chosen so that $m=800$
corresponds to a per-image mass of $1.28\%$. Images in the lowest $5\%$ of nearest-neighbour distance
within CIFAR-10, i.e.\ the near-duplicates already present in the dataset, are excluded from both the
duplicated images and the $m=1$ images. $\sigc$ is computed on the $64$ duplicated images, $200$
training images and $200$ held-out images by bisection on $[0.02,12]$ with $r=0.5$. Copies are counted
with the criterion of \citet{carlini2023extracting} on $50\,000$ DDIM samples with $100$ steps, and FID
is computed against the training set with $50\,000$ samples per model. The two models are not exactly
matched: the held-out loss of the model with duplicates is $1$--$4.5\%$ higher than that of the
released checkpoint at every $t$, and its gap between training and held-out loss is about $5\%$, against
$0.8\%$ for the released model, although it sees each non-duplicated image $20\%$ less often. The
released checkpoint should therefore be regarded as a reference trained without duplicates rather than
as an exactly matched control.

\paragraph{CIFAR-10 overfitting (Section\ref{app:cifar:n}).}
Each model is a $25.8$M-parameter DDPM trained for $15$k steps at batch size $128$, so the number of
passes over the data decreases from $19\,200$ at $N=100$ to $38$ at $N=5\cdot10^4$. The subsets are
nested prefixes of a single permutation, so the subset with $N=100$ is contained in all others. $\sigc$
is computed by bisection on $[0.02,20]$ with $12$ steps, $K=120$ and $r=0.5$, on the first $100$ images
of the permutation and on $100$ CIFAR-10 test images. For the copy rate we draw DDIM samples with $100$
steps, $10^4$ per model for $N\le500$ and $5\cdot10^4$ otherwise. A sample is a copy when the ratio of
its distances to the nearest and second-nearest training images is below $1/3$; the same criterion
applied to a held-out set of the same size finds no copies at any $N$. The public model is the
improved-diffusion checkpoint \texttt{cifar10\_uncond\_50M\_500K}, evaluated on the same images with
the same settings.

\paragraph{CIFAR-10 rarity (Section\ref{sec:cifar:out}).}
The base set consists of $N=10^4$ images from the same permutation, and the models use the same
architecture and $15$k training steps. The added images bring each training set to $11\,512$
examples. Since $\sigc$ on CIFAR-10 depends strongly on $\|x\|$ and on typicality, unmatched images
would obtain high values for trivial reasons; the added images are therefore matched to the base set
in RMS norm and total variation, and images with a near-duplicate in the CIFAR-10 training set are
excluded. Colour MNIST cannot be matched exactly (RMS $0.56$ against $0.49$, total variation $0.038$
against $0.056$) and should be regarded as the extreme of the atypicality range. We draw $10^5$ DDIM
samples per model ($\eta=1$, $100$ steps). A sample is a copy when its nearest image among the CIFAR-10
training set and all images of the four sources is at RMS distance below $0.07$ on $[0,1]$, and below
$0.45$ times the mean distance to the next $50$ images; no image that was not added to a model was
copied in $2\cdot10^5$ samples. $\sigc$ is computed with $K=1$ and $r=0.2$ by bisection on the range of
the noise schedule up to $\sigma=20$, with $14$ steps.
\paragraph{Filtering the memorization dataset.}
For the Stable Diffusion experiments, we use the matching-verbatim (MV) and template-verbatim (TV)
examples released by \citet{webster2023reproducible}. The collection contains many duplicate or
near-duplicate entries, particularly among the TV examples, where the same memorized template
appears with different visual patterns. To prevent these repeated templates from being
overrepresented in our evaluation, we manually deduplicate the collection, retaining one MV copy
per image and one TV example per template.
We additionally remove a small number of mismatched image--caption pairs. The released collection
associates each memorized image with its training caption, but does not guarantee that sampling from
that caption reproduces that particular image. In several cases, the stored caption instead
reproduces a different memorized image. Because our method evaluates memorization at the level of
image--caption pairs, we exclude these mismatches. Figure~\ref{fig:dataset-filtering} shows
representative examples removed by both filtering steps.
\begin{figure}[t]
    \centering
    \setlength{\tabcolsep}{1.5pt}
    \renewcommand{\arraystretch}{1.05}

    \begin{minipage}[t]{0.48\linewidth}
        \vspace{0pt}
        \centering
        \begin{tabular}{@{}cccc@{}}
        \scriptsize Duplicated 1 & \scriptsize Duplicated 2  &
            \scriptsize Duplicated 3  & \scriptsize Duplicated 4  \\[1pt]
            \includegraphics[width=0.23\linewidth]{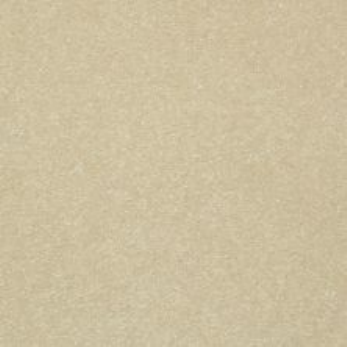} &
            \includegraphics[width=0.23\linewidth]{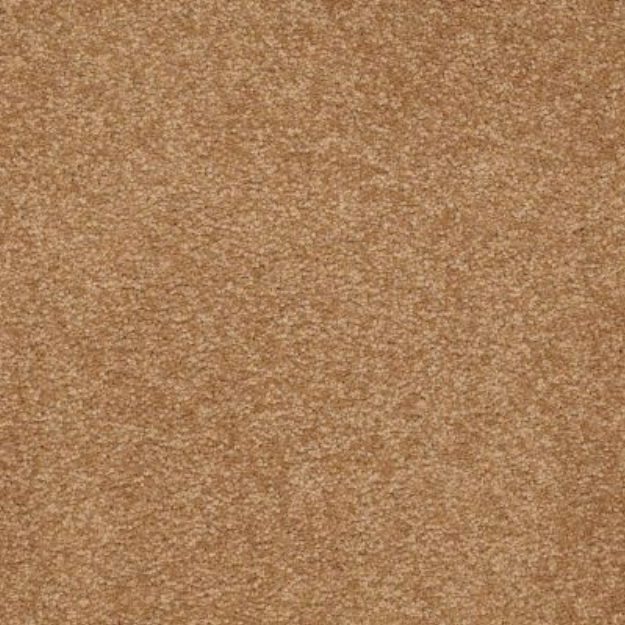} &
            \includegraphics[width=0.23\linewidth]{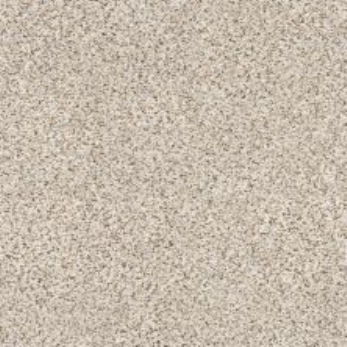} &
            \includegraphics[width=0.23\linewidth]{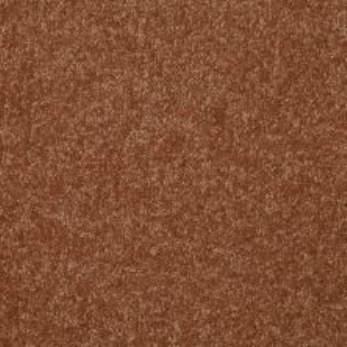} \\

            \includegraphics[width=0.23\linewidth]{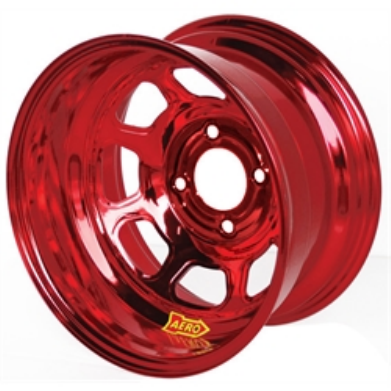} &
            \includegraphics[width=0.23\linewidth]{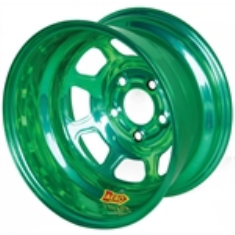} &
            \includegraphics[width=0.23\linewidth]{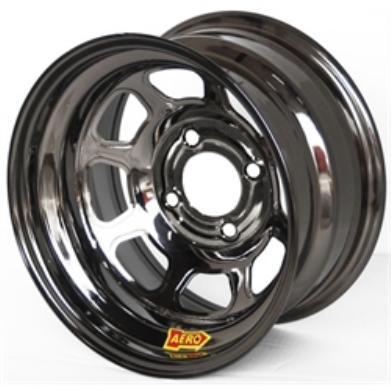} &
            \includegraphics[width=0.23\linewidth]{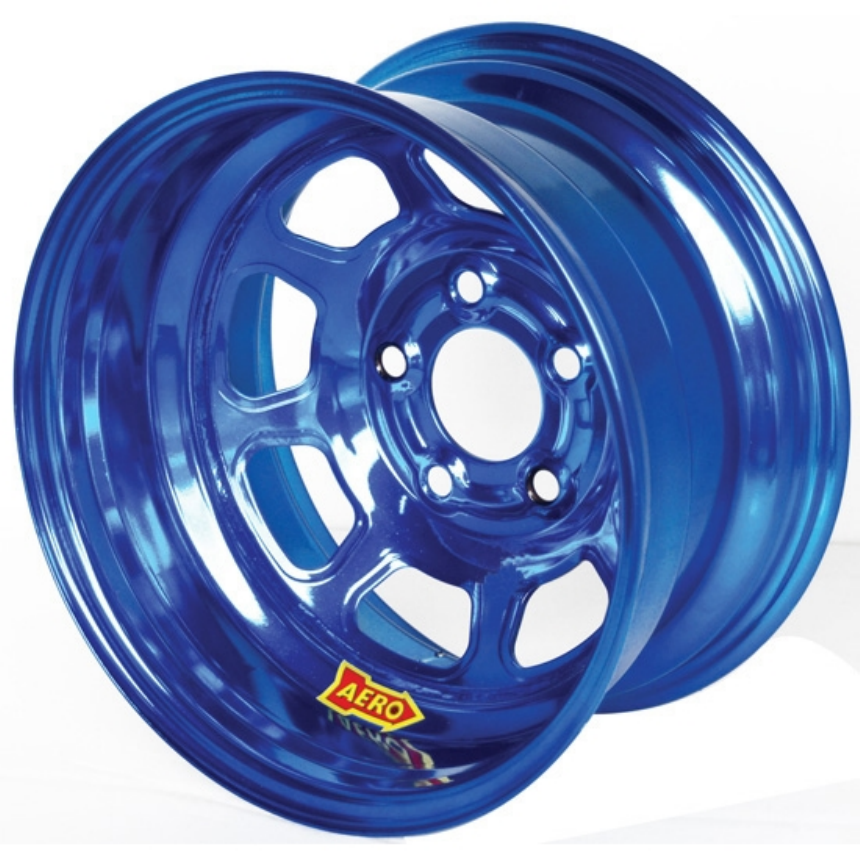} \\

            \includegraphics[width=0.23\linewidth]{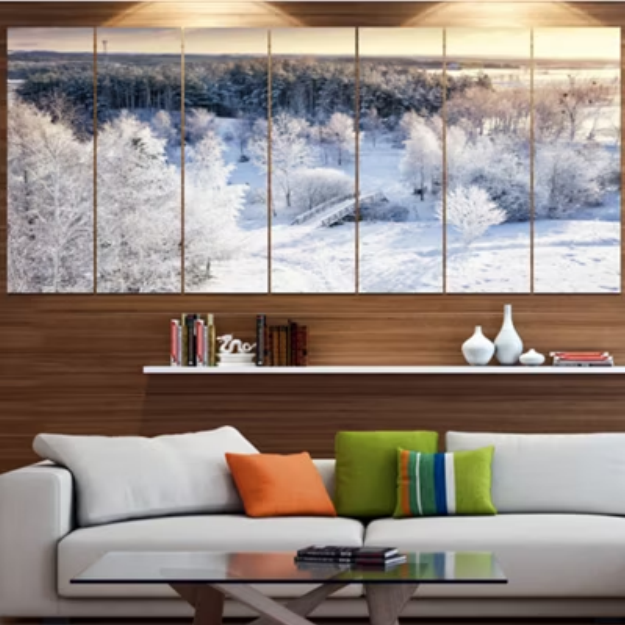} &
            \includegraphics[width=0.23\linewidth]{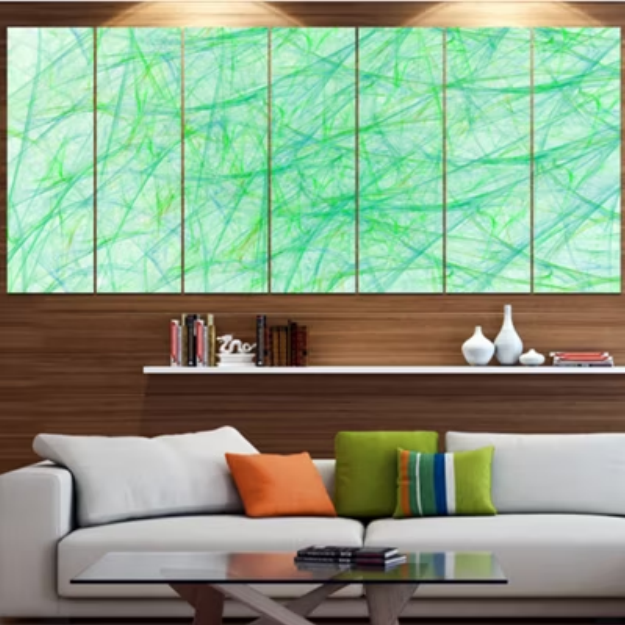} &
            \includegraphics[width=0.23\linewidth]{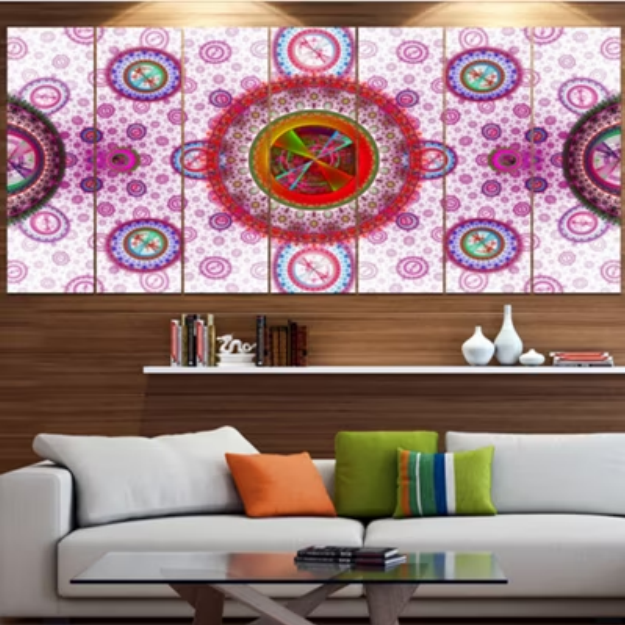} &
            \includegraphics[width=0.23\linewidth]{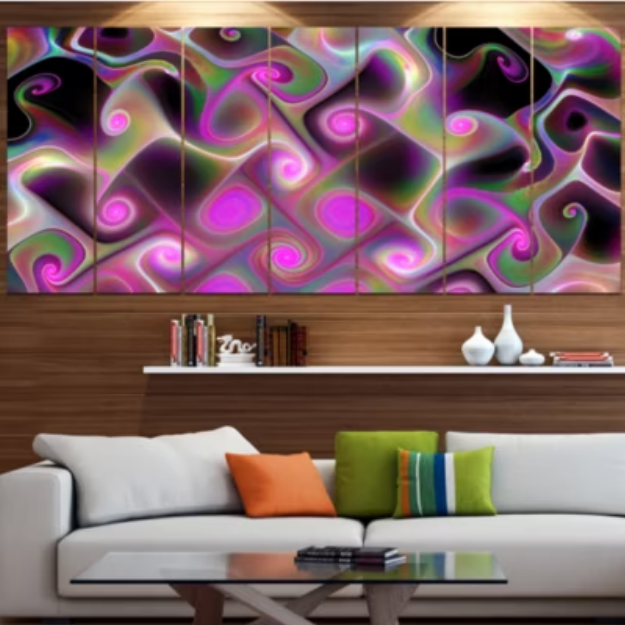}
        \end{tabular}

        \smallskip
        (a) Duplicate examples
    \end{minipage}
    \hfill
    \begin{minipage}[t]{0.48\linewidth}
        \vspace{0pt}
        \centering
        \begin{tabular}{@{}cccc@{}}
            \scriptsize Training & \scriptsize Generation 1 &
            \scriptsize Generation 2 & \scriptsize Generation 3 \\[1pt]

            \includegraphics[width=0.23\linewidth]{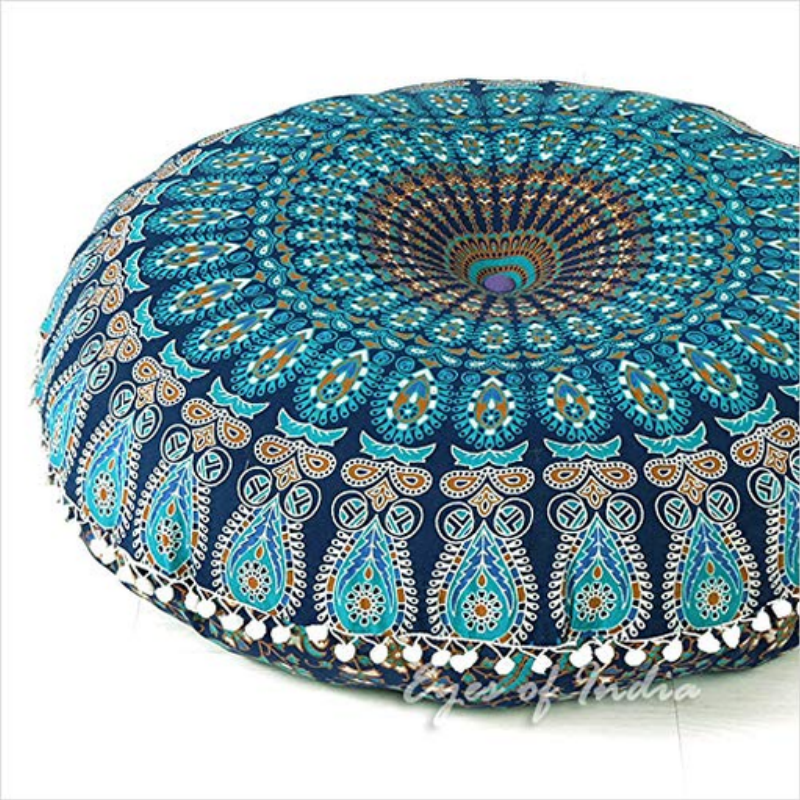} &
            \includegraphics[width=0.23\linewidth]{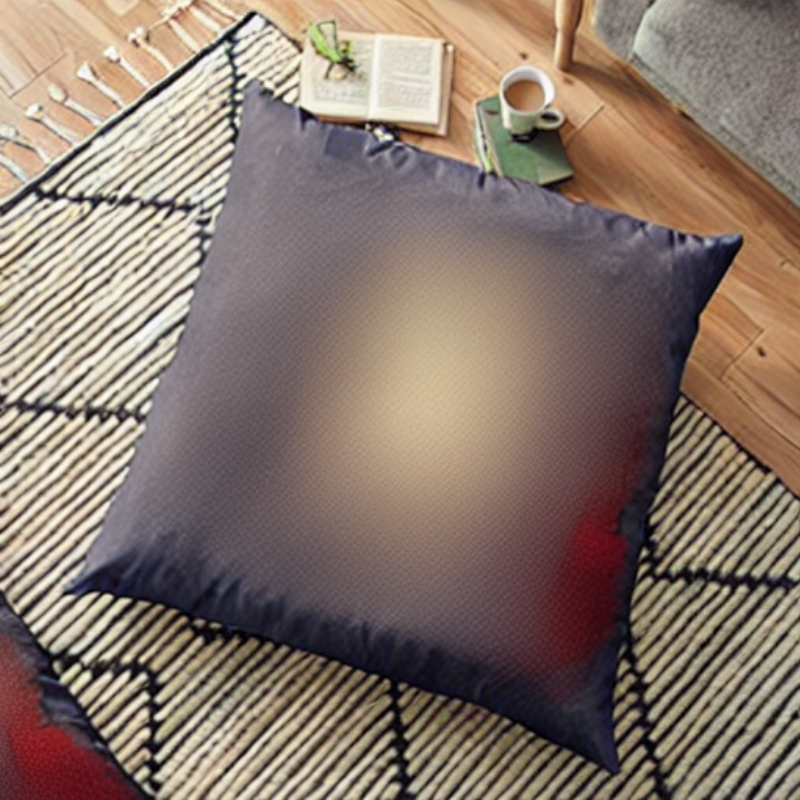} &
            \includegraphics[width=0.23\linewidth]{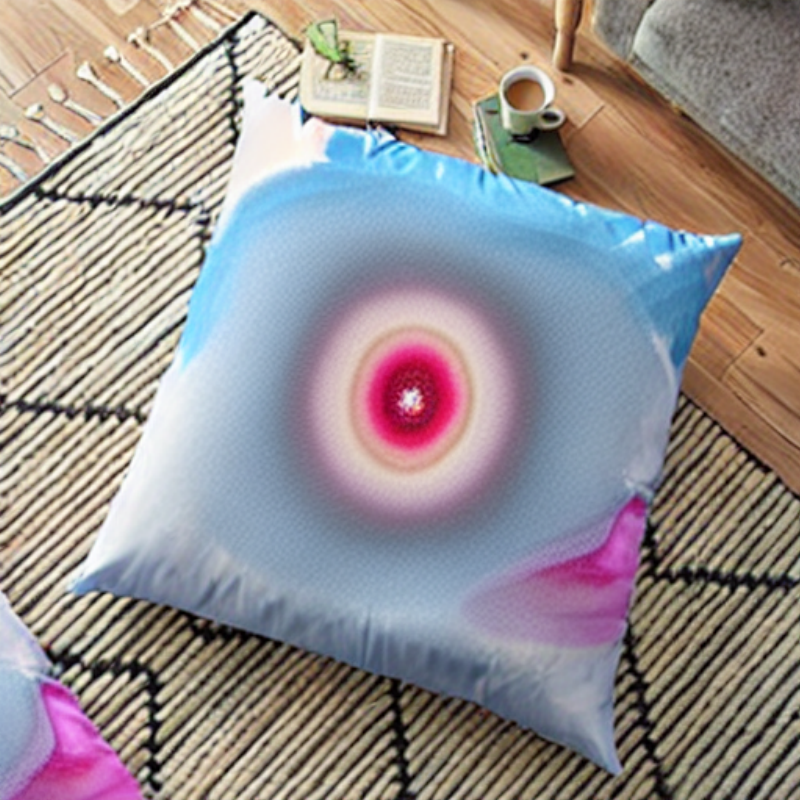} &
            \includegraphics[width=0.23\linewidth]{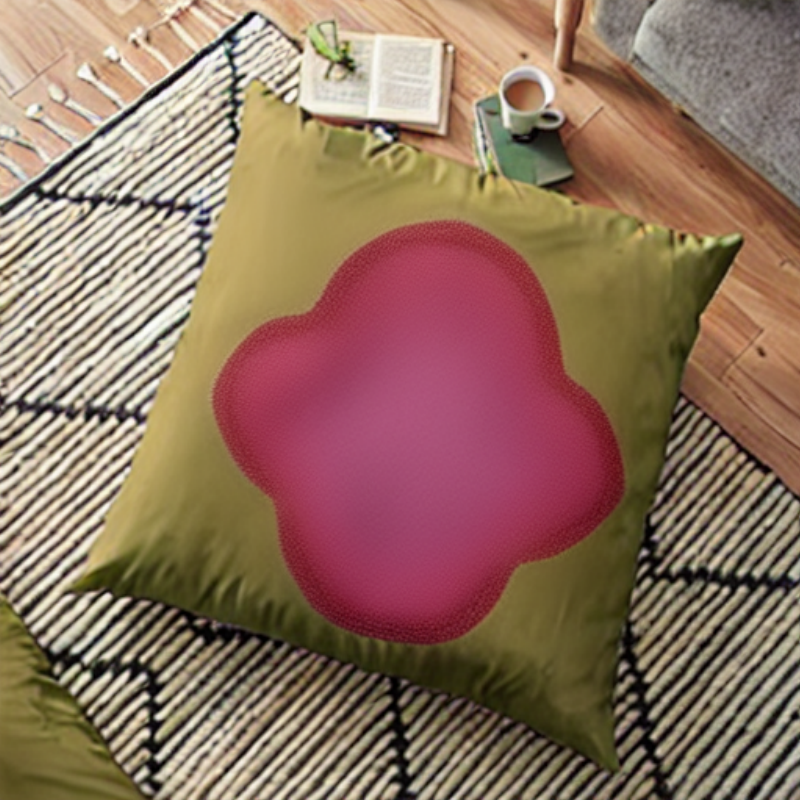} \\

            \includegraphics[width=0.23\linewidth]{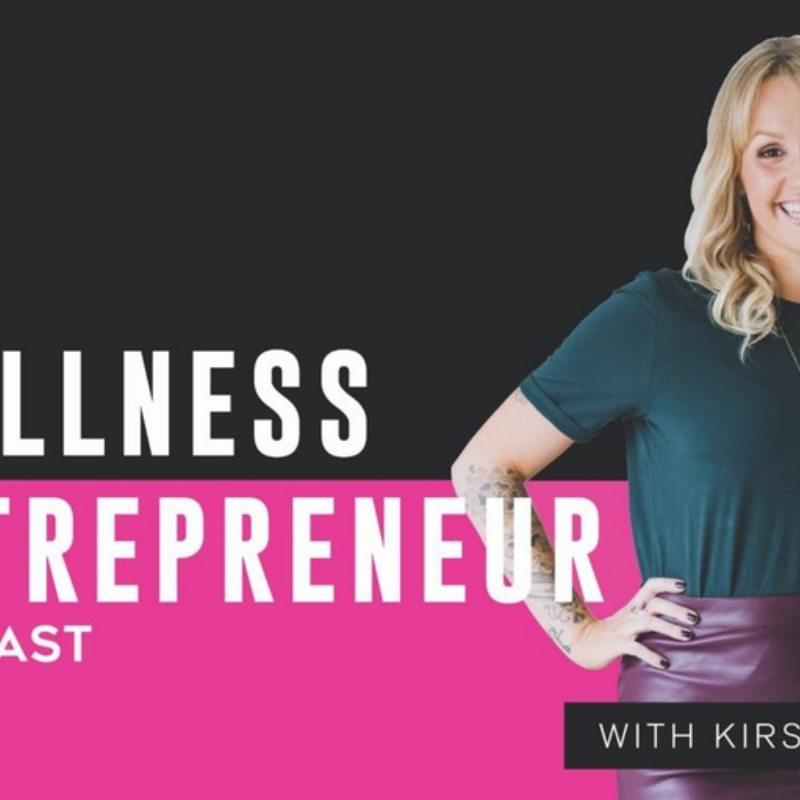} &
            \includegraphics[width=0.23\linewidth]{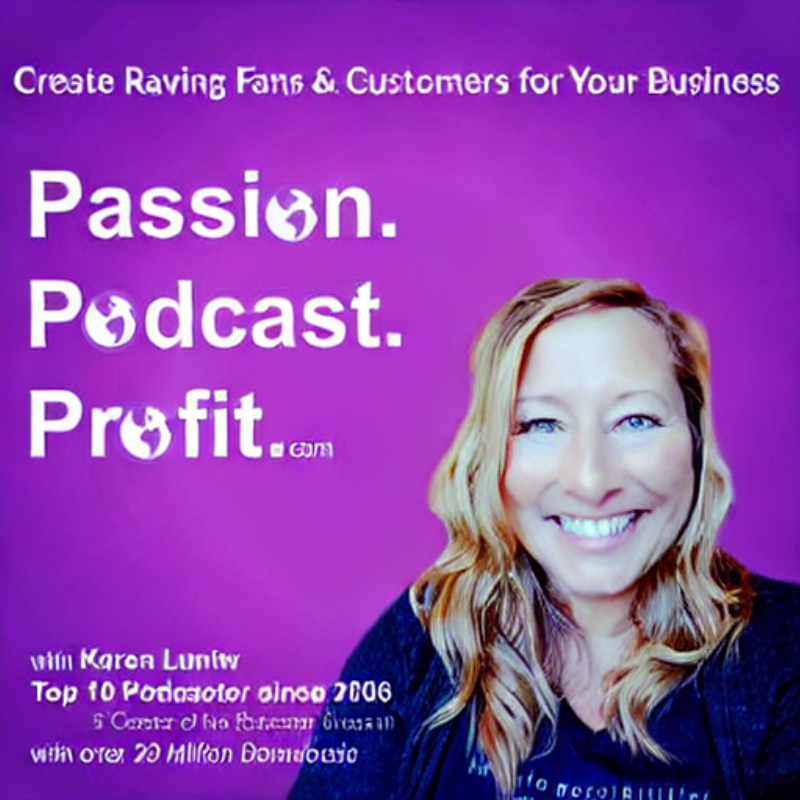} &
            \includegraphics[width=0.23\linewidth]{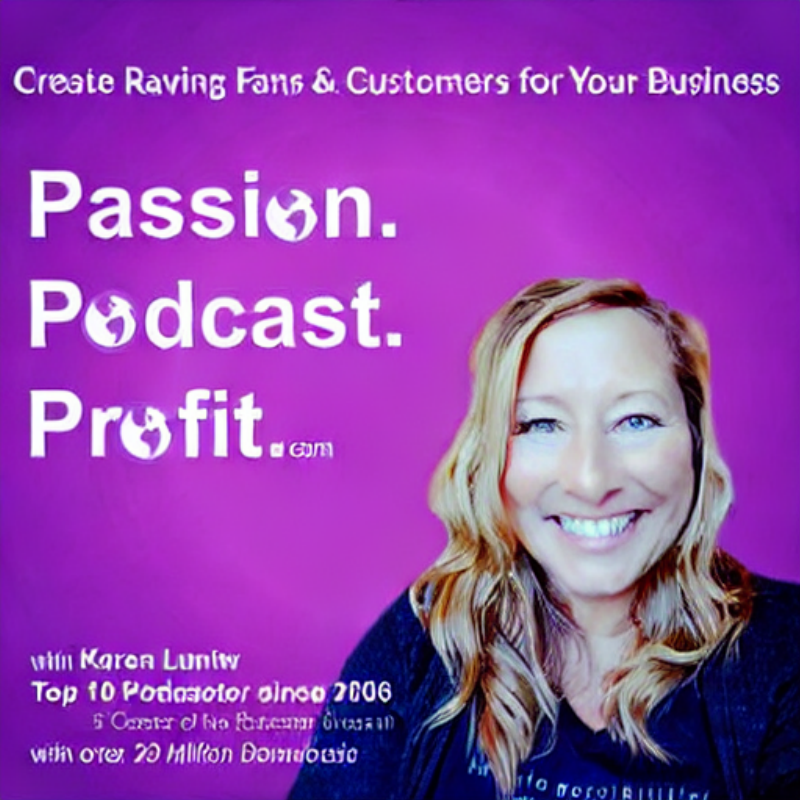} &
            \includegraphics[width=0.23\linewidth]{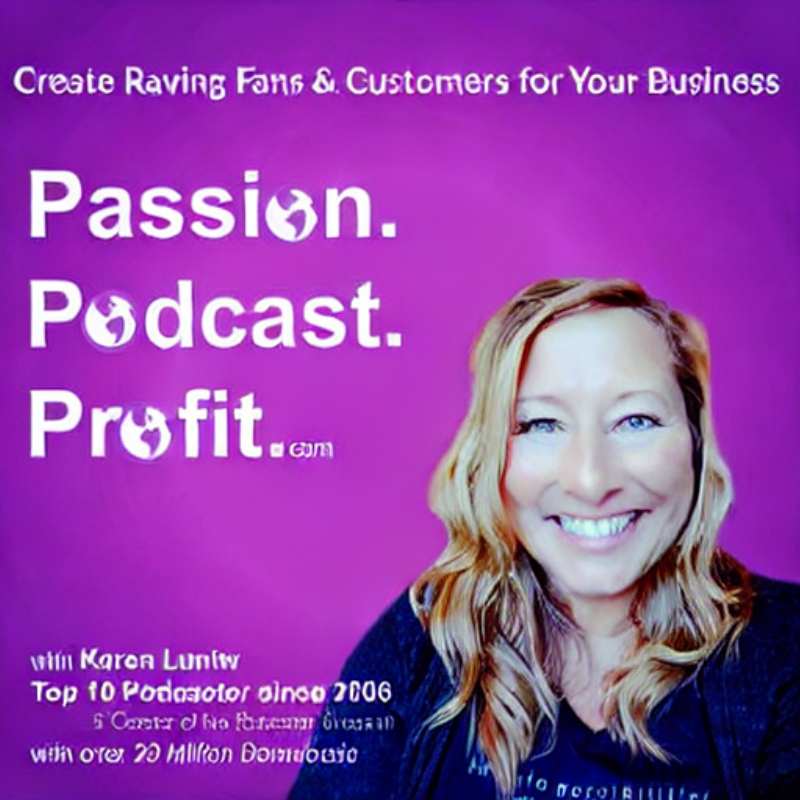} \\

            \includegraphics[width=0.23\linewidth]{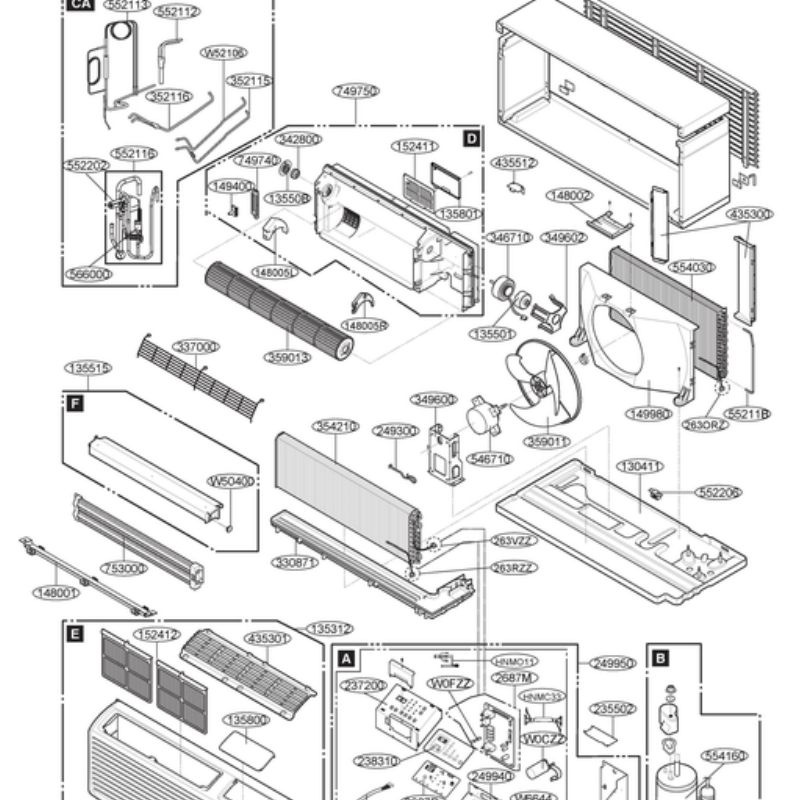} &
            \includegraphics[width=0.23\linewidth]{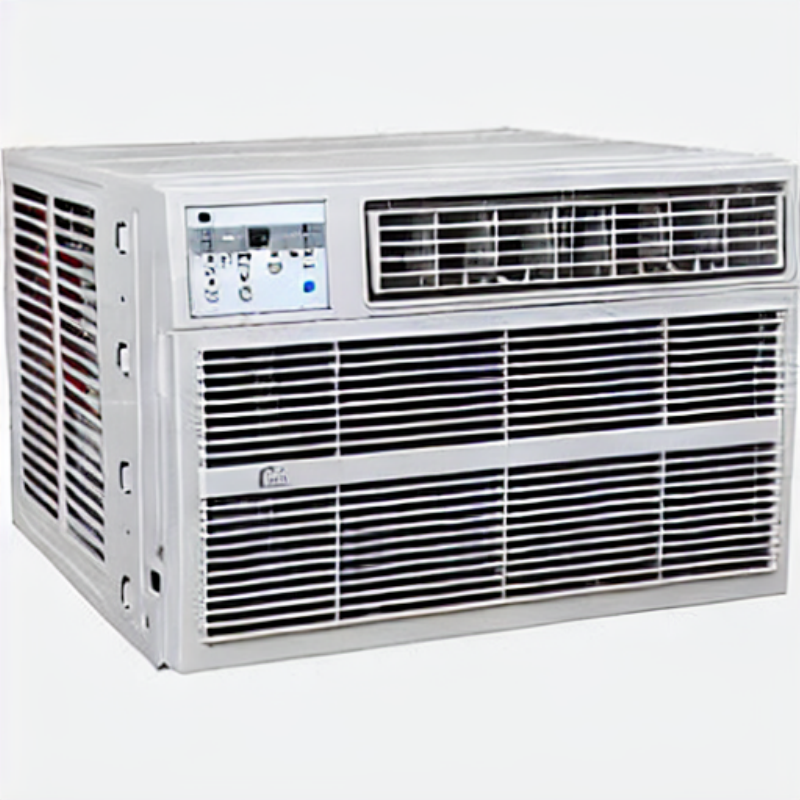} &
            \includegraphics[width=0.23\linewidth]{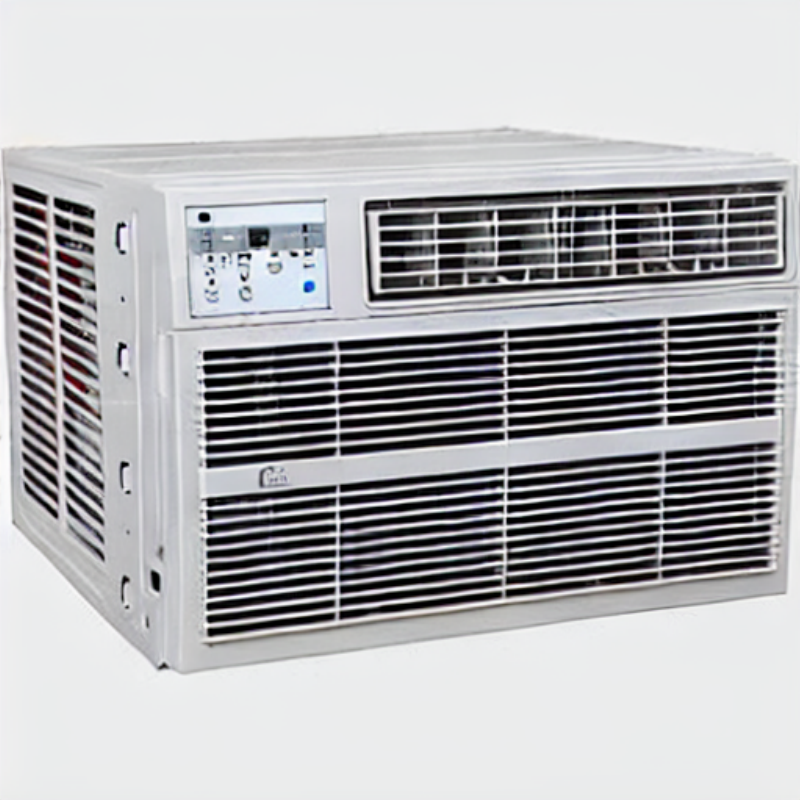} &
            \includegraphics[width=0.23\linewidth]{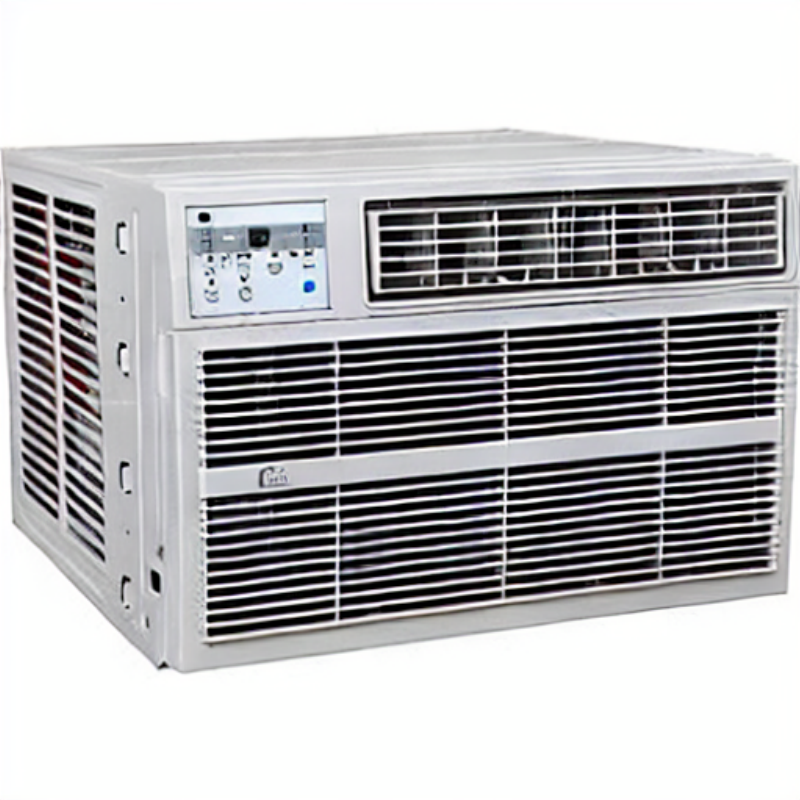}
        \end{tabular}

        \smallskip
        (b) Mismatched image--caption pairs
    \end{minipage}

    \caption{\textbf{Examples removed when filtering the memorization dataset.}
    \textbf{(a)} Each row shows a group of duplicate examples, of which we retain only one.
    \textbf{(b)} The first image in each row is the stored training image, followed by three
    generations from its training caption. We remove cases in which the caption does not reproduce
    the stored image.}
    \label{fig:dataset-filtering}
\end{figure}
\paragraph{Comparison to \cite{wen2024detecting}.} For the results reported in Table~\ref{tab:sd}, we report the best-performing configuration for both variants of \citet{wen2024detecting}: $n=32$ samples with first and first 10 steps for $\|\epsilon_c-\epsilon_\varnothing\| $ and $\|\epsilon_c(x_0)-\epsilon_\varnothing(x_0)\|$, respectively.
\paragraph{Local $\boldsymbol{|\Delta\log\sigma_c|}$.}
In order to retain the image local information, we apply the retention test of Definition~\ref{def:sigc} separately to each 4-dimensional patch $x_u$ of the latent space, $||M_\sigma^k(x_u)-x_u|| \leq r||x_u||$. We first apply the denoising map for $K$ iterations over a logarithmic grid of 24 noise levels. Then, we estimate the local $\sigma^u_c$ as the geometric midpoint between the noise level that first escapes the radius $r$ and the previous value in the grid, and define local $\lvert\Delta\log\sigma_c\rvert$ by averaging over locations. Unless otherwise stated, we use $K=16$ and $r=0.1$ for the local gap experiments.

\section{Stable Diffusion: additional experiments}
\label{app:sd}

\paragraph{Orbits across the $(\sigma,k)$ grid.}
Figure~\ref{fig:sigcpic} expands Figure~\ref{fig:teaser}: for the same memorized and control
pairs, every row is an orbit at one noise scale, so both the orbit (a row) and the end of the orbit
across scales (the last column) can be read off one grid. More qualitative examples are in Appendix \ref{app:examples}.

\begin{figure}[h]
\centering
\includegraphics[width=\linewidth]{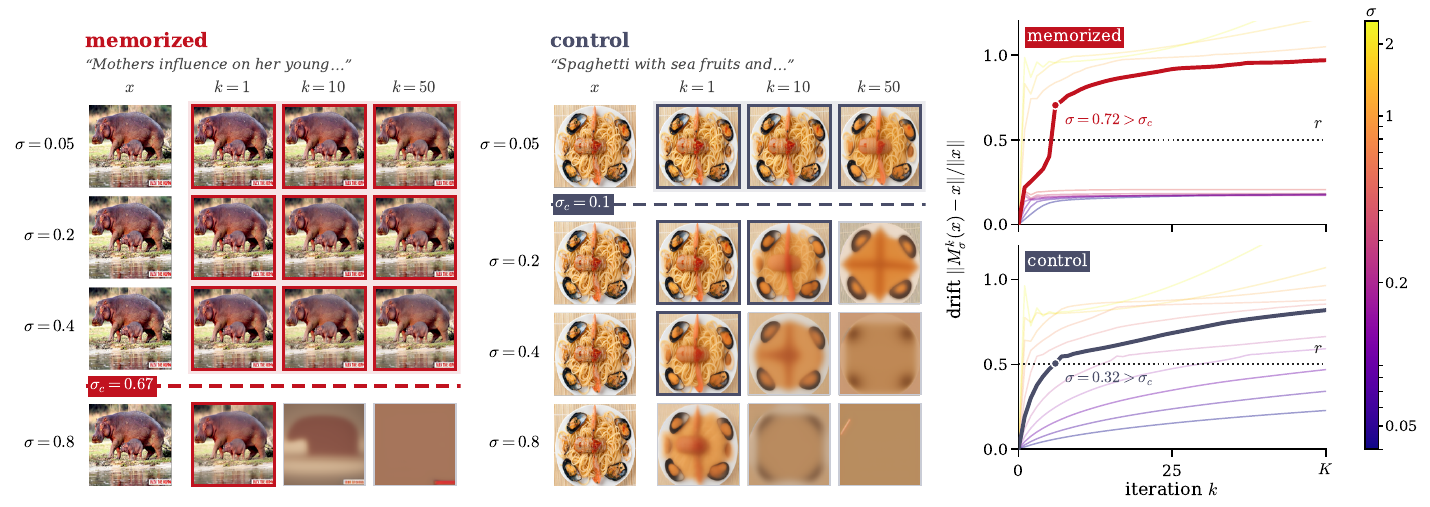}
\caption{The full $(\sigma,k)$ grid behind Figure~\ref{fig:teaser}: the critical scale on Stable Diffusion v1.4, for the same memorized caption--image pair (left)
and control pair (middle), each run with its own caption at guidance scale $1$. Each row iterates
$\M$ from the image $x$ at one noise scale $\sigma$; a framed iterate $\M^{k}(x)$ is still within
$r\|x\|$ of $x$. At small $\sigma$ the model keeps returning $x$, which therefore sits in a basin of
its own; at large $\sigma$ the orbit drifts away towards content shared with other data. The critical
scale $\sigc$ (dashed line) separates the two regimes: the memorized image withstands far more noise
than the control before it escapes. Right: the drift $\|\M^{k}(x)-x\|/\|x\|$ along the orbit, one
curve per scale coloured by $\sigma$; the thick curve is a scale above $\sigc$ and the dot marks
where it crosses $r$. Here $K=50$ and $r=0.5$.}
\label{fig:sigcpic}
\end{figure}

\subsection{Ablation}\label{app:sd-ablation} 
Figures~\ref{fig:ablation-K} and~\ref{fig:ablation-radius} assess sensitivity to the number of
iterations $K$ and escape radius $r$. Performance is stable around our primary choices:
$K=16$, with $r=0.25$ for global scores and $r=0.1$ for the local gap.

For $r=0.25$, the global scores remain broadly stable through $K=64$ and deteriorate at larger
values. With $r=0.1$, their performance is lower and begins to decline earlier; this sensitivity is
particularly pronounced for $\sigma_c$ on TV examples. By contrast, the local gap performs
similarly for both radii across most values of $K$, with only a modest decline beyond $K=128$.
Noticeably, mean $\sigma_c$ decreases with $K$, reducing separation between memorized and control
examples and causing the drop in detection performance at large $K$.

The radius ablation shows a similar pattern: global scores are more sensitive to $r$, especially on TV examples, whereas the local gap maintains strong performance as $r$ varies by an order of magnitude.

The main results and the $K$-ablation estimate $\sigma_c$ by log-space bisection. To evaluate many
radii efficiently, the radius ablation instead uses a logarithmic grid of 24 noise levels and takes
the geometric midpoint between the last retained level and the first escaped level. Overall, these results show that the reported performance is not specific to the selected
hyperparameters and that the local gap is particularly robust.
\begin{figure*}[t]
    \centering
    \setlength{\tabcolsep}{2pt}
    \renewcommand{\arraystretch}{1.05}
    \begin{tabular}{@{}>{\centering\arraybackslash}m{0.018\linewidth} c c c c c@{}}
        &
        \scriptsize $|\Delta\log\sigma_c|$ &
        \scriptsize Local gap &
        \scriptsize $\sigma_c$ &
        \scriptsize $\Delta\log\sigma_c$ &
        \scriptsize Mean $\sigma_c$ \\

        \rotatebox{90}{\scriptsize $r=0.1$} &
        \includegraphics[width=0.18\linewidth]
            {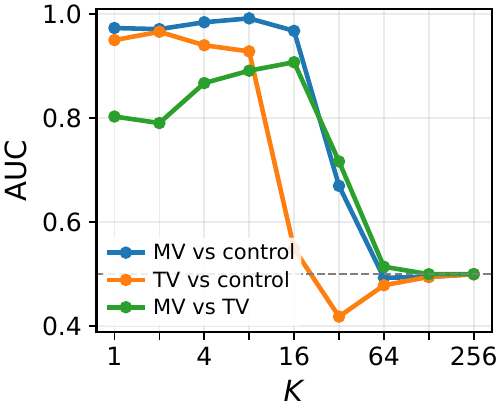} &
        \includegraphics[width=0.18\linewidth]
            {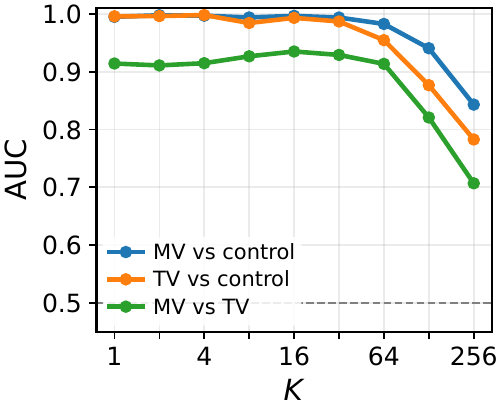} &
        \includegraphics[width=0.18\linewidth]
            {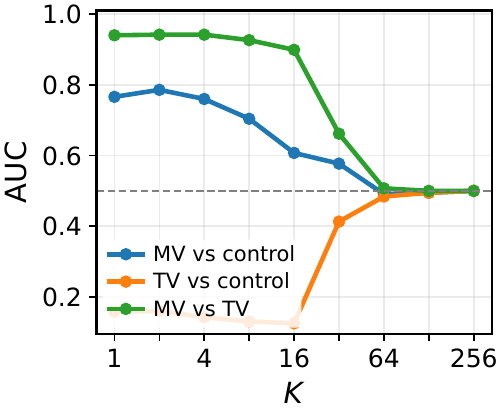} &
        \includegraphics[width=0.18\linewidth]
            {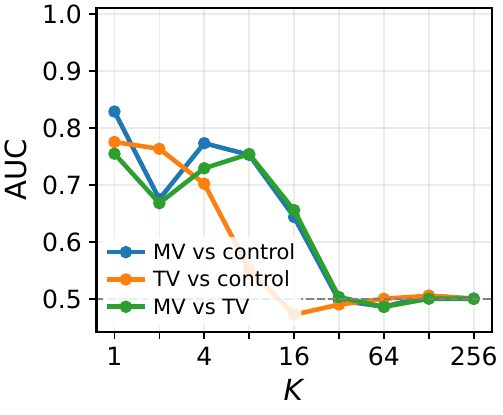} &
        \includegraphics[width=0.18\linewidth]
            {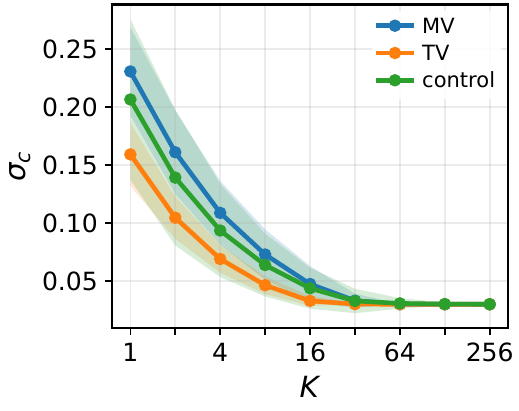} \\

        \rotatebox{90}{\scriptsize $r=0.25$} &
        \includegraphics[width=0.18\linewidth]
            {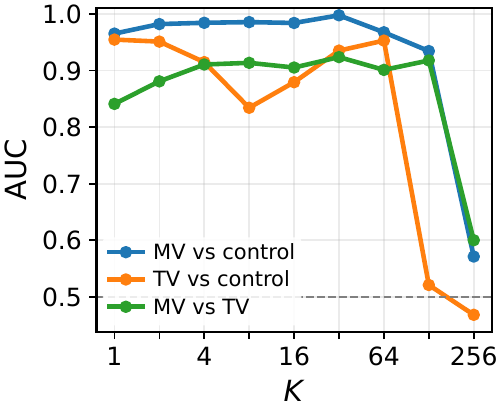} &
        \includegraphics[width=0.18\linewidth]
            {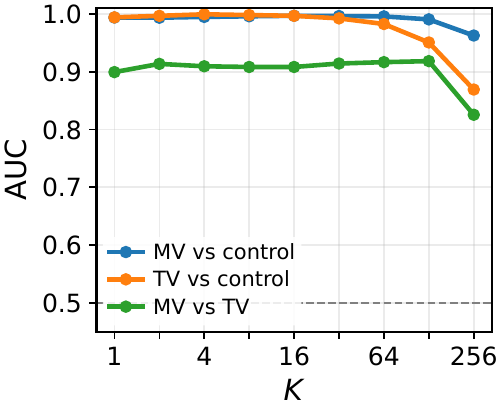} &
        \includegraphics[width=0.18\linewidth]
            {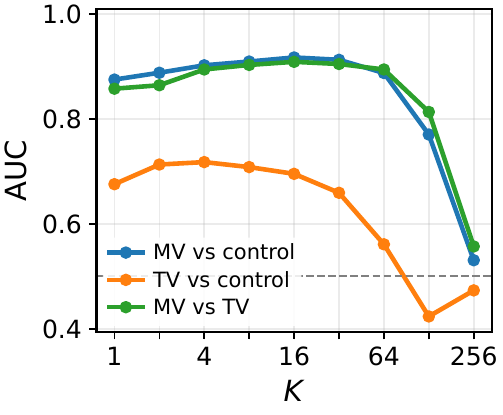} &
        \includegraphics[width=0.18\linewidth]
            {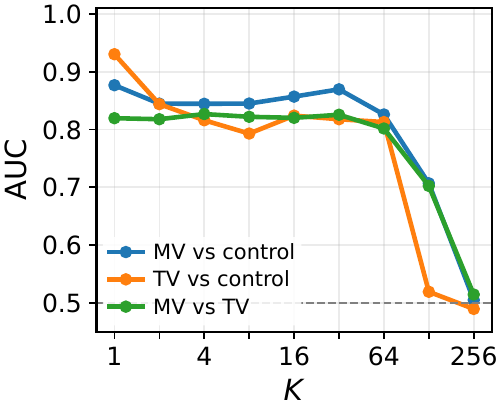} &
        \includegraphics[width=0.18\linewidth]
            {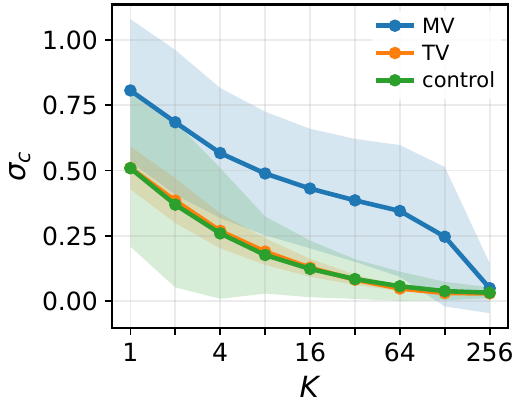}
    \end{tabular}

    \caption{\textbf{Ablation on the number of iterations $\boldsymbol{K}$.}
    Detection performance as a function of $K$ for escape radii $r=0.1$ and $r=0.25$.
    The first four columns report AUC for the different critical-scale scores; the final
    column reports the mean critical scale (solid line) $\pm$ one standard deviation (shaded region). Performance is stable around the primary choice $K=16, r=0.25$.}
    \label{fig:ablation-K}
\end{figure*}
\begin{figure*}[t]
    \centering
\begin{tabular}{@{}c c@{}}
    \scriptsize $\sigma_c$ &
    \scriptsize $\Delta\log\sigma_c$ \\

    \includegraphics[width=0.45\linewidth]
        {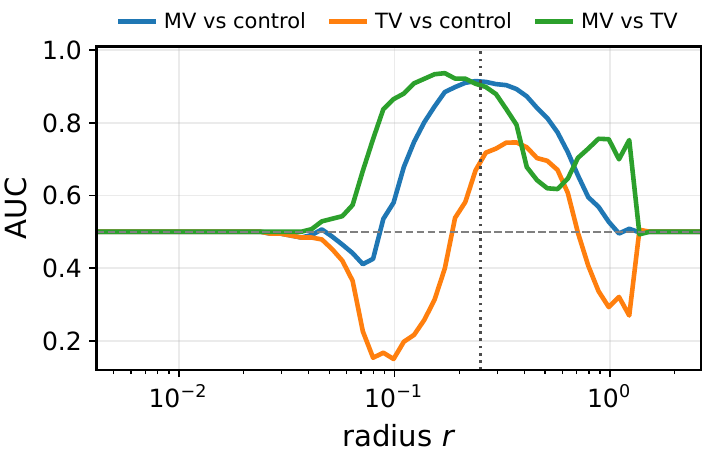} &
    \includegraphics[width=0.45\linewidth]
        {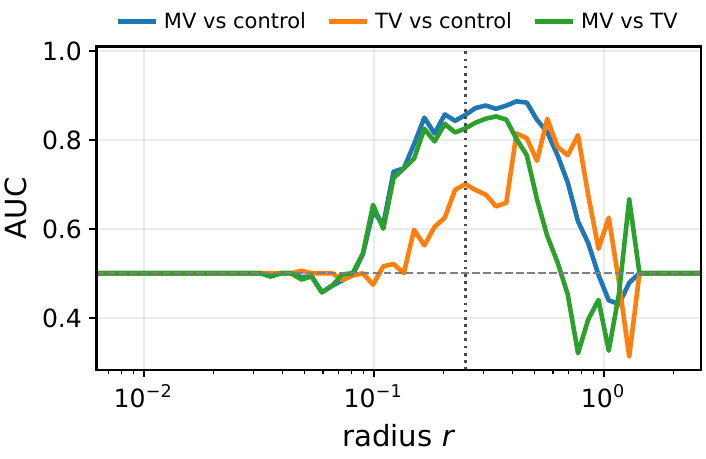} \\

    \scriptsize $|\Delta\log\sigma_c|$ &
    \scriptsize Local $|\Delta\log\sigma_c|$ \\

    \includegraphics[width=0.45\linewidth]
        {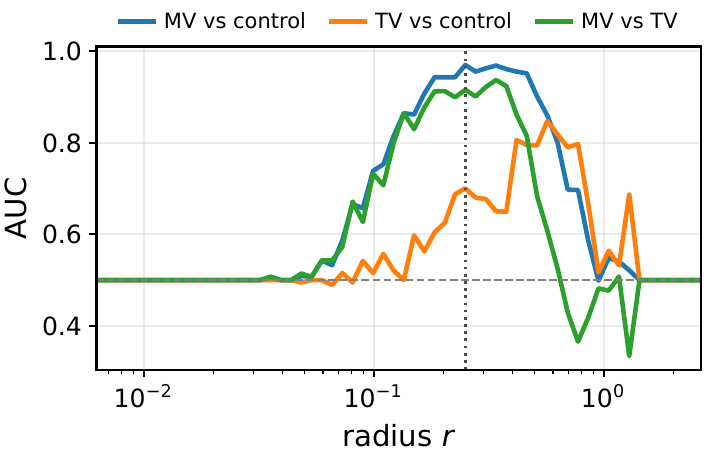} &
    \includegraphics[width=0.45\linewidth]
        {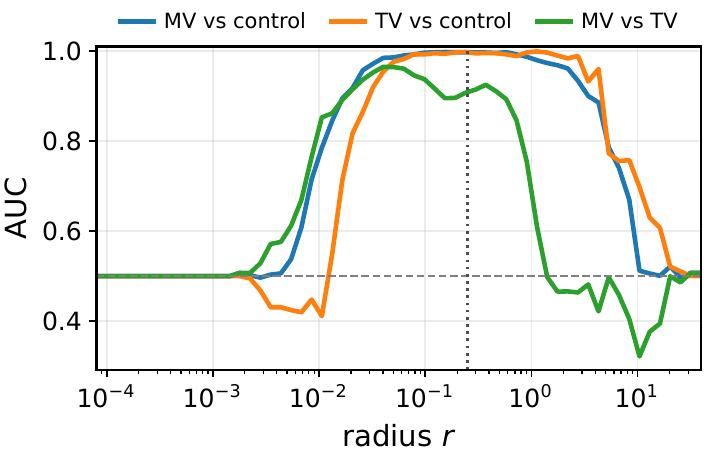}
\end{tabular}
    \caption{\textbf{Ablation on the escape radius \boldsymbol{$r$}.}
    Detection performance (AUC) as a function of $r$ using $K=16$. For this ablation,
    $\sigma_c$ is approximated over a logarithmic grid of 24 noise levels rather than
    estimated by log-space bisection. Performance is stable around the primary choice $K=16, r=0.25$ (global measures), $r=0.1$ (local $|\Delta\log\sigma_c|$).}
    \label{fig:ablation-radius}
\end{figure*}
\subsection{Further experimental results}\label{app:hist}
Figure~\ref{fig:sd-histograms} shows the score distributions for the main detection tasks using the
best critical-scale metric in each setting. For separating memorized examples from controls, we use
the local $|\Delta\log\sigma_c|$ with $K=16$ and $r=0.1$. For distinguishing MV and TV pairs from
their caption-swapped counterparts, we use the signed $\Delta\log\sigma_c$ with $K=16$ and
$r=0.25$.

The caption-swap distributions illustrate why the sign of the gap matters. Most swapped pairs have
a negative gap, indicating that an unrelated caption weakens retention relative to the unconditional
branch; this behavior is especially frequent for TV examples. Taking the absolute value folds these
negative scores onto the positive axis, obscuring the separation between original and swapped pairs.

\begin{figure*}[t]
    \centering

    \begin{subfigure}[t]{0.32\linewidth}
        \centering
        \includegraphics[width=\linewidth]
            {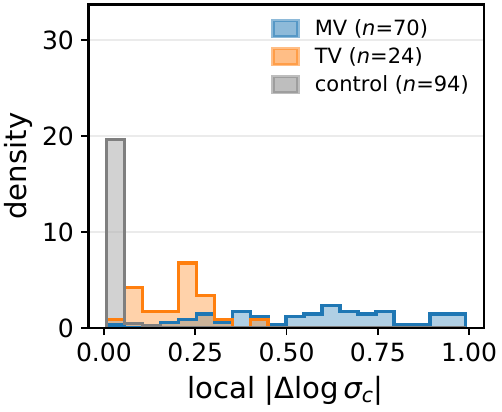}
        \label{fig:hist-mem-control}
    \end{subfigure}
    \hfill
    \begin{subfigure}[t]{0.32\linewidth}
        \centering
        \includegraphics[width=\linewidth]
            {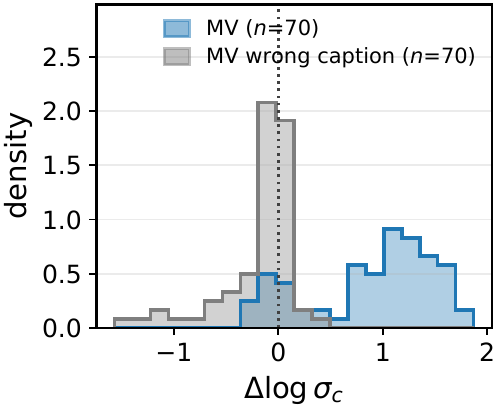}
        \label{fig:hist-mv-swap}
    \end{subfigure}
    \hfill
    \begin{subfigure}[t]{0.32\linewidth}
        \centering
        \includegraphics[width=\linewidth]
            {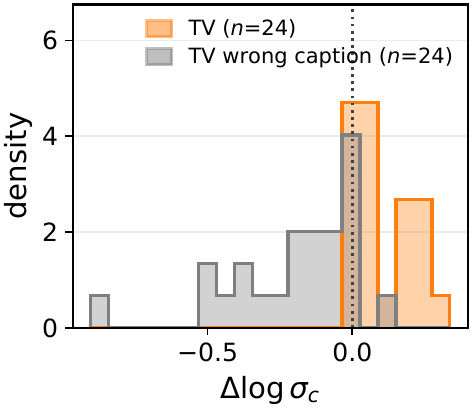}
        \label{fig:hist-tv-swap}
    \end{subfigure}
    \caption{\textbf{Critical-scale score distributions across evaluation groups.} Histograms of the scores used to separate the different groups.
    \textit{Left:} Local $|\Delta\log\sigma_c|$ for MV, TV, and control examples
    ($K=16$, $r=0.1$).
    \textit{Center-Right:} Signed $\Delta\log\sigma_c$ for MV and TV pairs with their
    original and swapped captions ($K=16$, $r=0.25$). Swapped pairs frequently have negative gaps.}
    \label{fig:sd-histograms}
\end{figure*}
\subsection{Detecting memorized regions: additional qualitative examples}\label{app:local-sigma-gap}
Figure~\ref{fig:additional-local-gap} presents four additional TV examples. As in
Figure~\ref{fig:sd-interpretability} (left), large local gaps align with content reproduced consistently across generations, whereas regions that vary tend to have gaps near zero. These examples further show that the local critical-scale gap captures the spatial extent of memorized content across different images.

\begin{figure*}[t]
    \centering
    \begin{subfigure}[t]{0.48\linewidth}
        \centering
        \includegraphics[width=\linewidth]
            {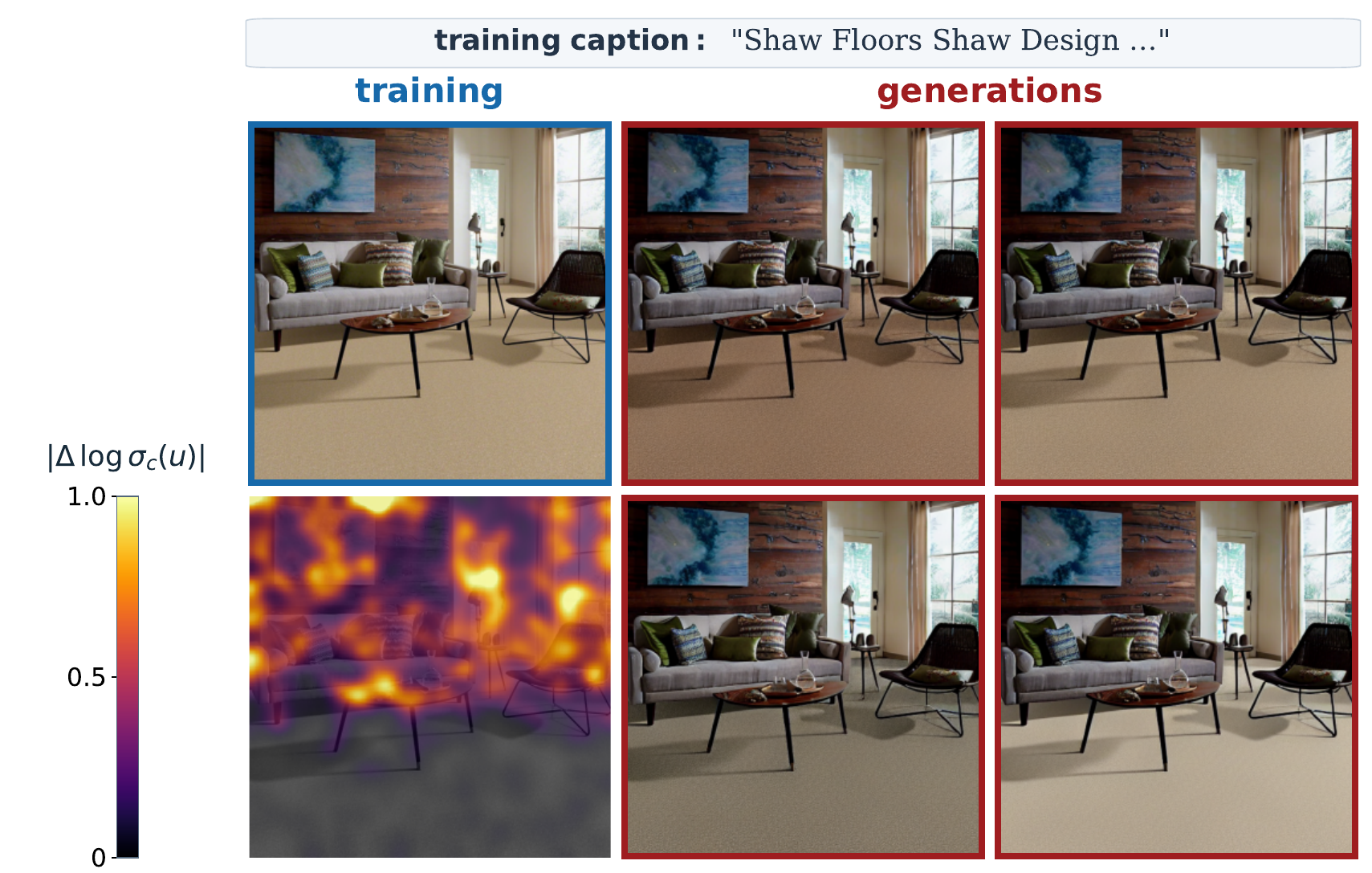}
    \end{subfigure}
    \hfill
    \begin{subfigure}[t]{0.48\linewidth}
        \centering
        \includegraphics[width=\linewidth]
            {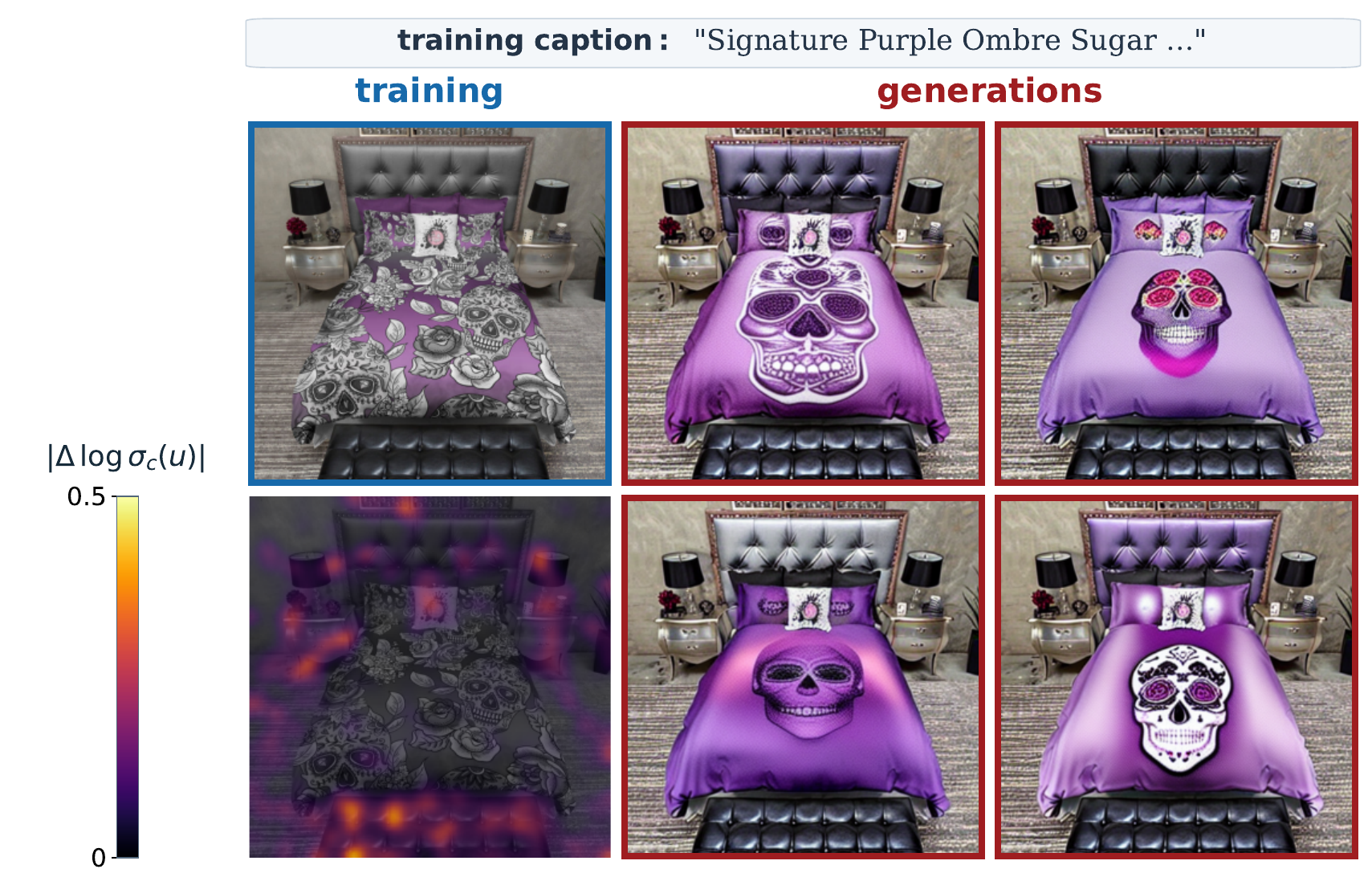}
    \end{subfigure}

    \vspace{4pt}

    \begin{subfigure}[t]{0.48\linewidth}
        \centering
        \includegraphics[width=\linewidth]
            {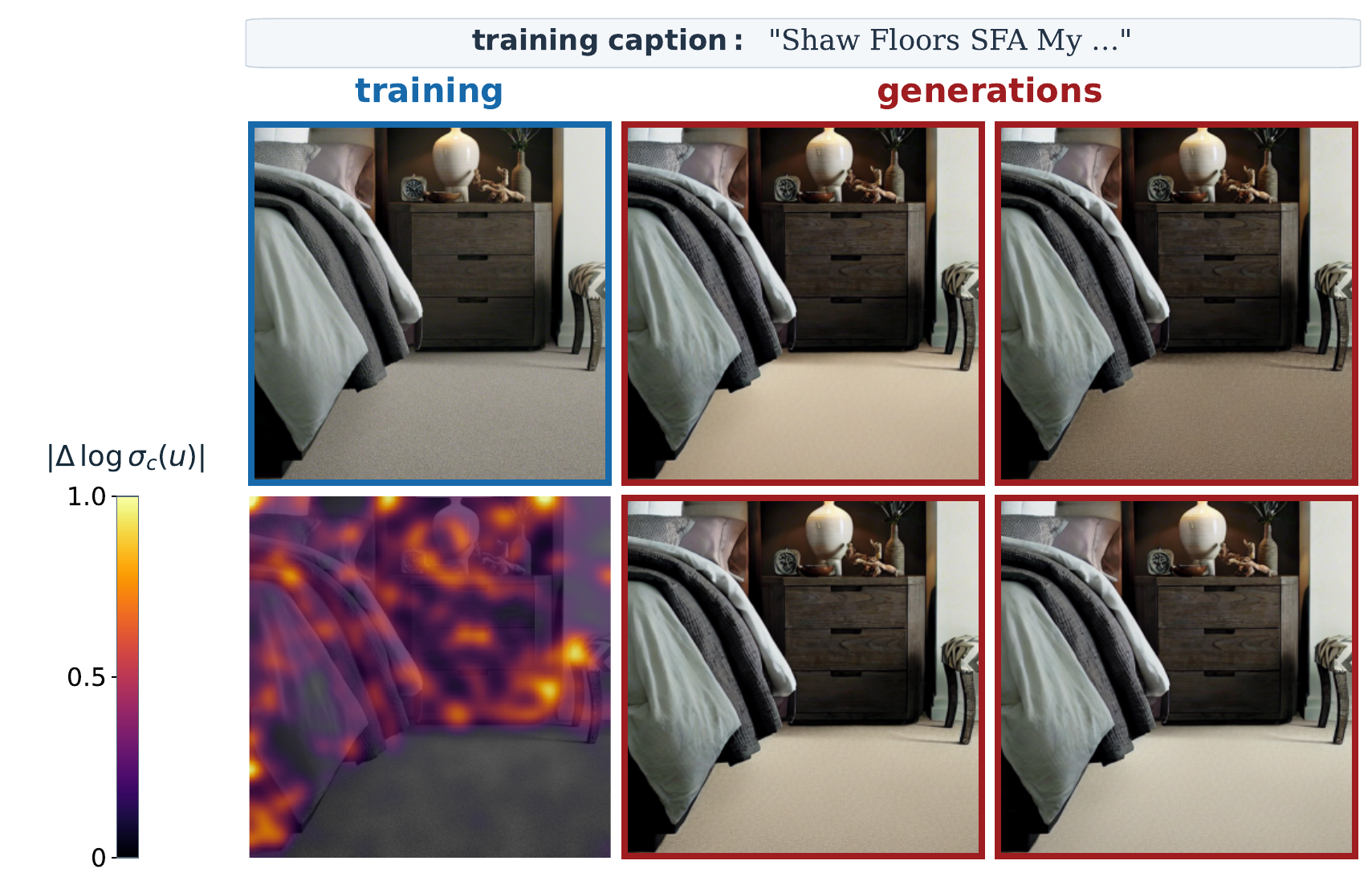}
    \end{subfigure}
    \hfill
    \begin{subfigure}[t]{0.48\linewidth}
        \centering
        \includegraphics[width=\linewidth]
            {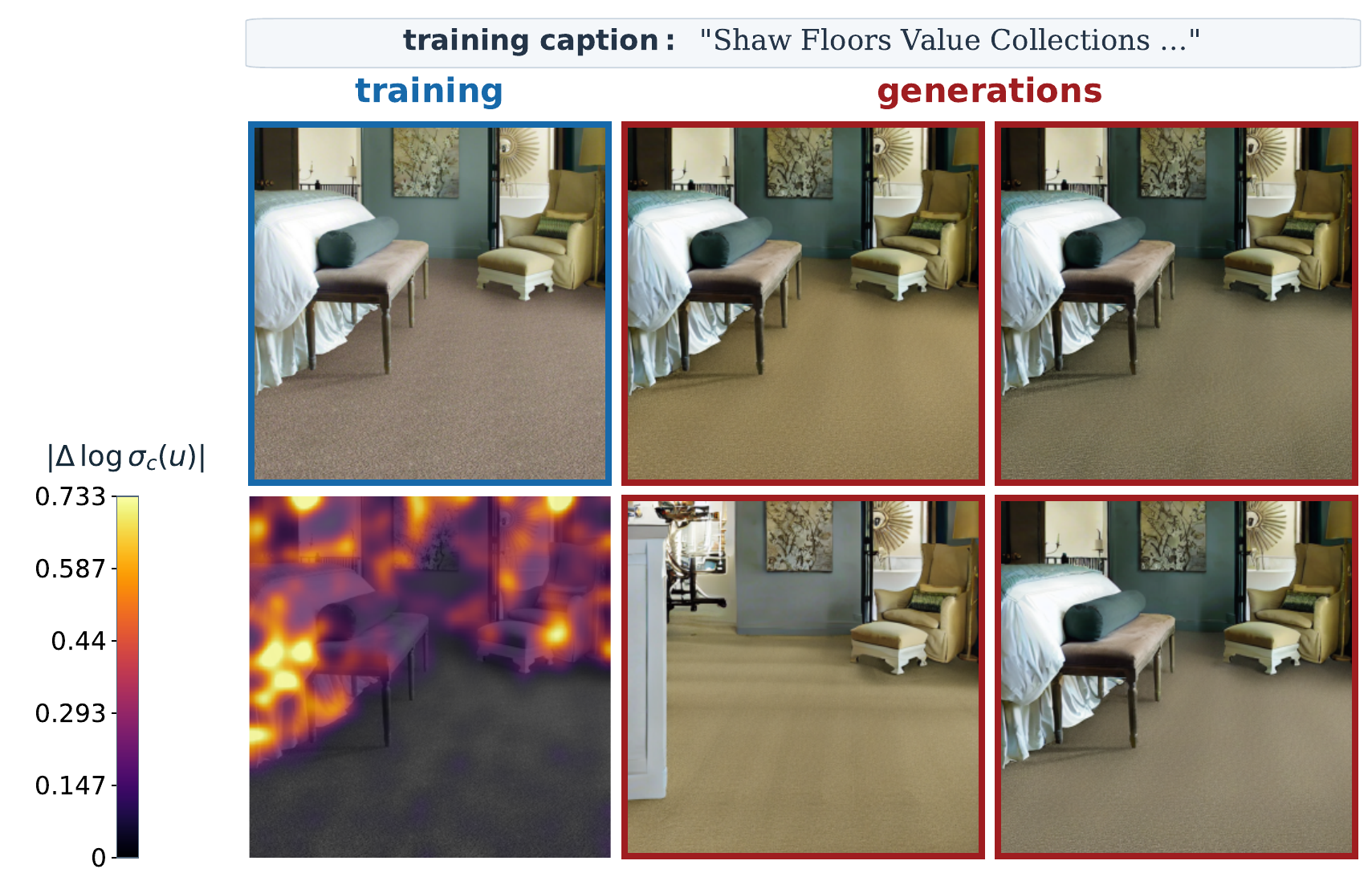}
    \end{subfigure}

    \caption{\textbf{Additional examples of localized memorization.}
    Each panel shows a TV training image, generations from its caption, and the corresponding
    local $|\Delta\log\sigma_c^u|$ map. Large gaps align with regions reproduced consistently
    across generations, whereas variable regions generally have gaps near zero.}
    \label{fig:additional-local-gap}
\end{figure*}
\subsection{Further analysis of mismatched caption gap}\label{app:capt-swap}
\paragraph{Ambiguity near zero.}
As shown in Figure~\ref{fig:sd-interpretability} (right), the gap between mismatched-caption and
unconditional critical scales is most informative when strongly negative. Values near zero,
however, can represent two different regimes. Both branches may retain the image, indicating unconditional memorization and therefore yielding successful inpainting, or neither branch may retain it, indicating the absence of unconditional memorization. Because the gap measures relative retention, it cannot distinguish between these cases. Figure~\ref{fig:near-zero-gap} (top) illustrates the first regime and Figure~\ref{fig:near-zero-gap} (bottom) the second: despite having similar gaps, the examples have markedly
different critical scales and inpainting performance.
\paragraph{Further examples with negative gap.}
Figure~\ref{fig:negative-gap-examples} provides two further examples with negative gaps. Both show
evidence of unconditional memorization, but the example with the more negative gap yields more
faithful reconstructions, consistent with stronger relative retention by the unconditional branch.

\begin{figure*}[t]
    \centering
    \begin{subfigure}[t]{0.98\linewidth}
        \centering
        \includegraphics[width=\linewidth]
            {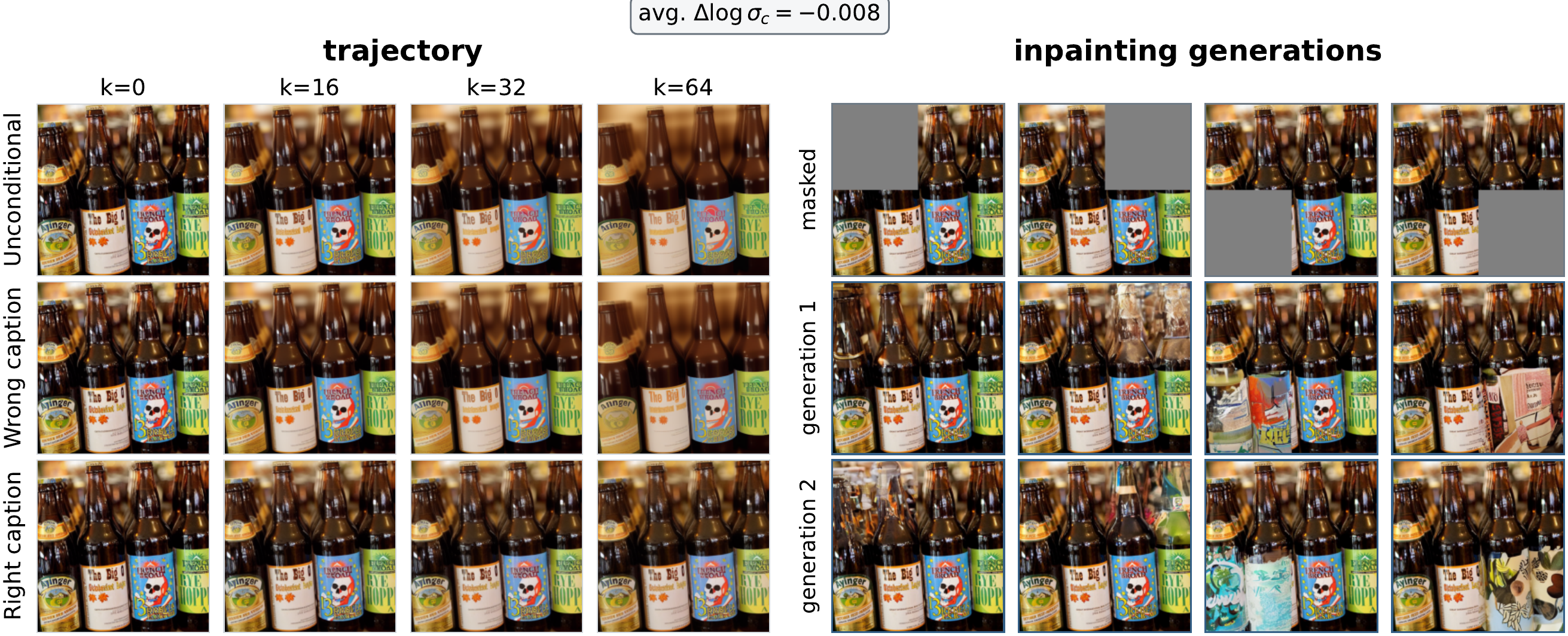}
    \end{subfigure}
    \hfill
    \vspace{1cm}
    \begin{subfigure}[t]{0.98\linewidth}
        \centering
        \includegraphics[width=\linewidth]
            {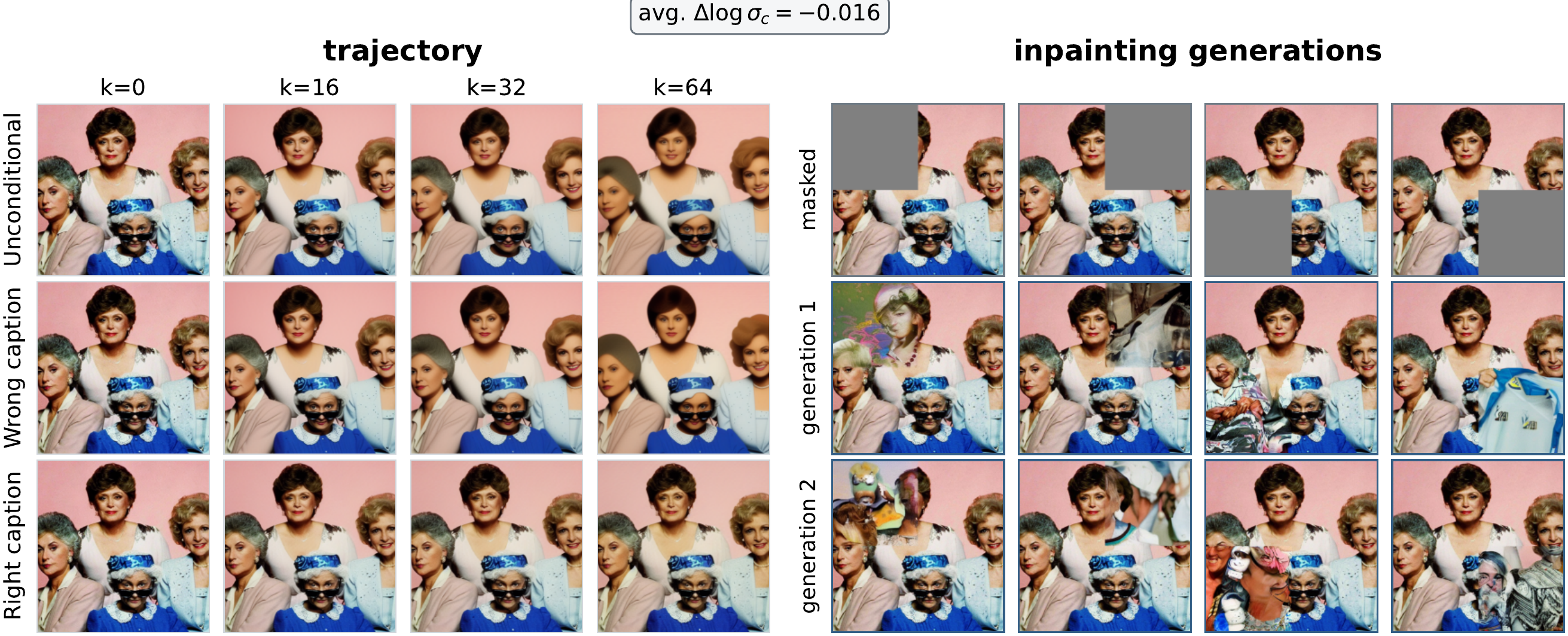}
    \end{subfigure}

    \caption{\textbf{Near-zero gaps are ambiguous.}
    The two examples have similar mismatched-caption gaps but different absolute retention.
    \textit{Top:} Both branches retain the image, and unconditional reconstruction succeeds.
    \textit{Bottom:} Neither branch retains the image, and unconditional reconstruction fails.}
    \label{fig:near-zero-gap}
\end{figure*}
\begin{figure*}[t]
    \centering
    \begin{subfigure}[t]{0.98\linewidth}
        \centering
        \includegraphics[width=\linewidth]
            {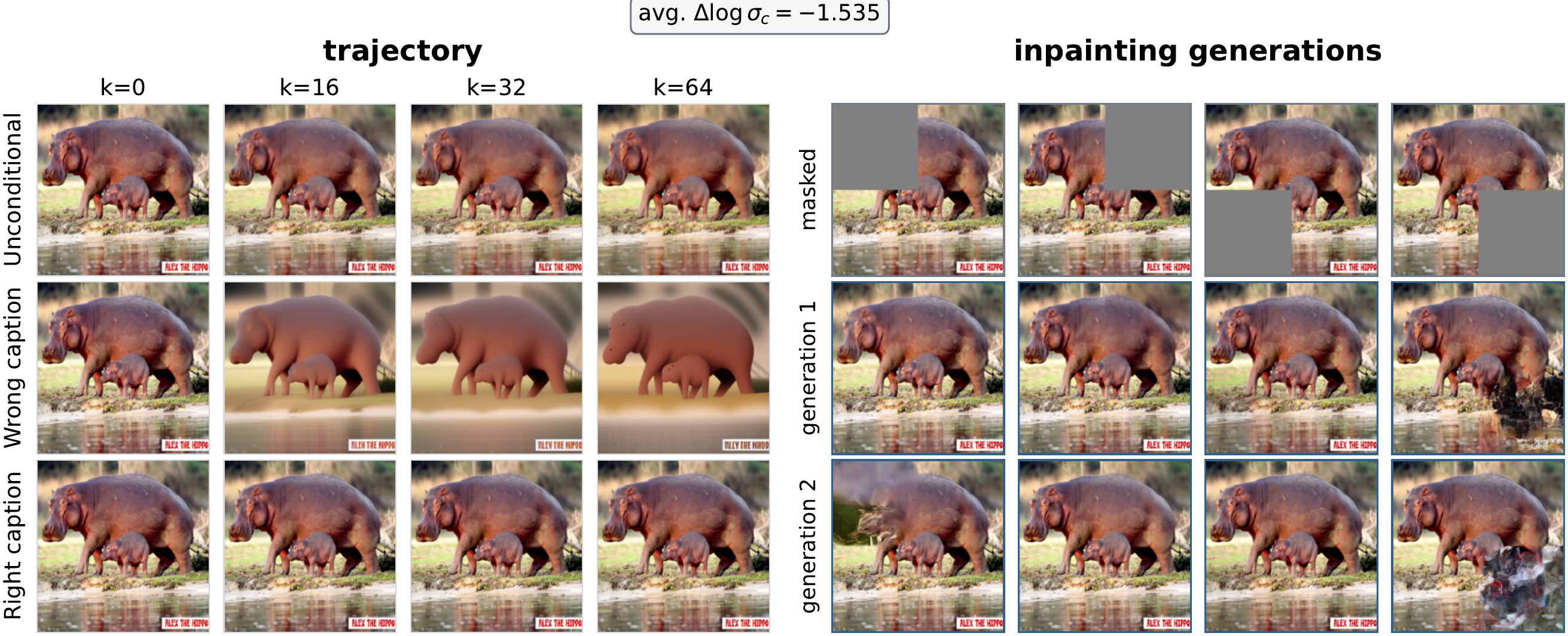}
    \end{subfigure}
    \hfill
    \begin{subfigure}[t]{0.98\linewidth}
        \centering
        \includegraphics[width=\linewidth]
            {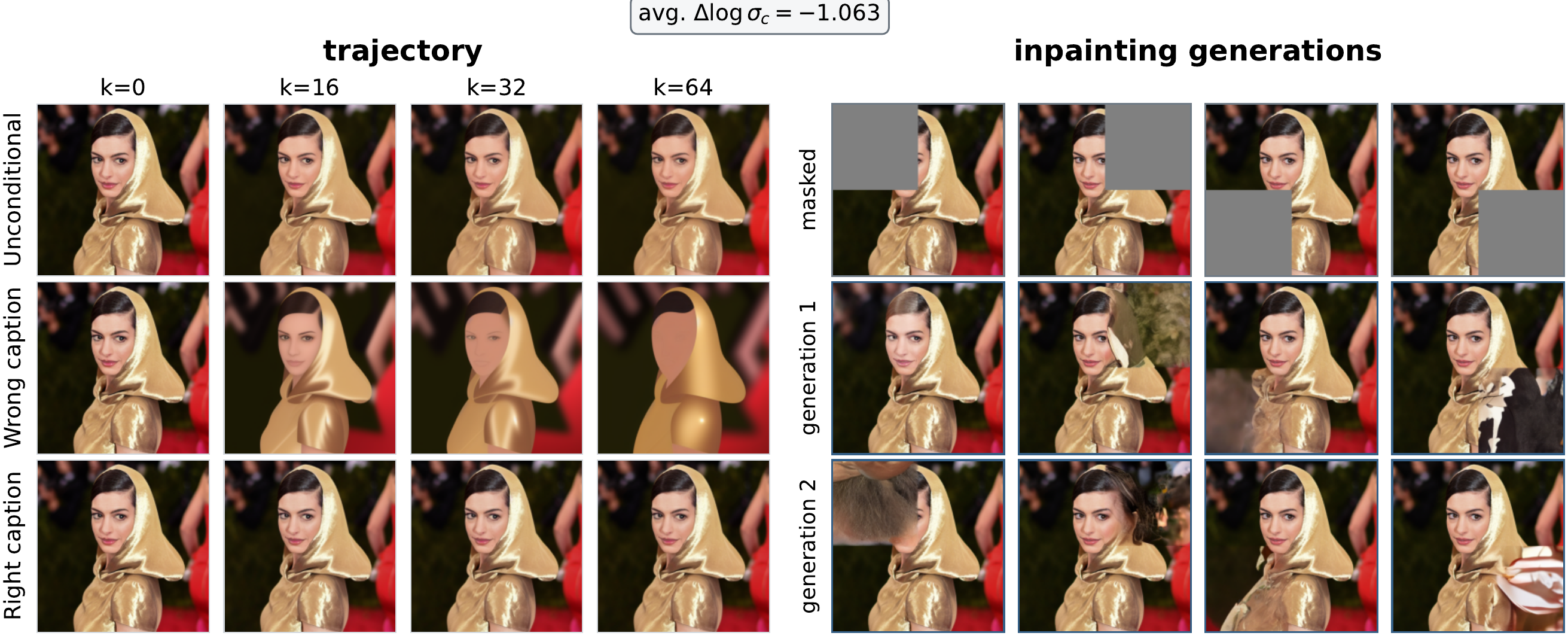}
    \end{subfigure}

    \caption{\textbf{Additional examples with negative mismatched-caption gaps.}
    Both examples exhibit unconditional retention and successful inpainting. The more negative gap in
the top example corresponds to more faithful reconstructions than the less negative gap in the
bottom example.}
    \label{fig:negative-gap-examples}
\end{figure*}
\subsection{Further qualitative examples}\label{app:examples}
Figures~\ref{fig:traj-mv-control}--\ref{fig:traj-mv-tv} provide additional comparisons of fixed-scale
trajectories and critical scales for matching-verbatim (MV), template-verbatim (TV), and control
training images. Overall, retention is strongest for MV examples, followed by TV examples and then
randomly selected controls: MV images remain close to their initial states for more iterations and
at larger noise scales than TV images, which in turn persist longer than controls.

The TV trajectories also reveal spatial differences in retention. In the bedroom example of
Figures~\ref{fig:traj-mv-tv} and \ref{fig:sweep}, the surrounding room remains stable after the bedsheet pattern has been smoothed, consistent with the memorized and variable regions identified in
Figure~\ref{fig:additional-local-gap}.  Similarly, in Figure~\ref{fig:traj-tv-control}, the
non-memorized curtain pattern disappears first, whereas the memorized plant and chair persist
longer. These examples provide qualitative evidence that fixed-scale dynamics retain memorized
regions longer than non-memorized ones, further motivating a local measure for partially memorized
images.

Figure~\ref{fig:caption-gap} illustrates how the caption gap distinguishes memorized image--caption
pairs from controls. For the MV example, conditioning on the training caption retains the image to
larger noise scales than the unconditional branch, producing a large positive
$\Delta\log\sigma_c$. For the control, the two branches escape at similar scales, and the gap is
close to zero.

Figure~\ref{fig:sweep} further illustrates how retention varies across memorization types and image
regions. In the upper comparison, the MV example has the largest $\sigma_c$, followed by the TV
example and then the control. The lower comparison shows two TV examples: the image with memorized
content occupying a larger region has the larger $\sigma_c$. Within both TV images, non-memorized
content---the floor in the first and the bedsheet pattern in the second---disappears at smaller noise
scales than the memorized regions.

\begin{figure*}[t]
    \centering
    \begin{subfigure}[t]{0.98\linewidth}
        \centering
        \includegraphics[width=\linewidth]
            {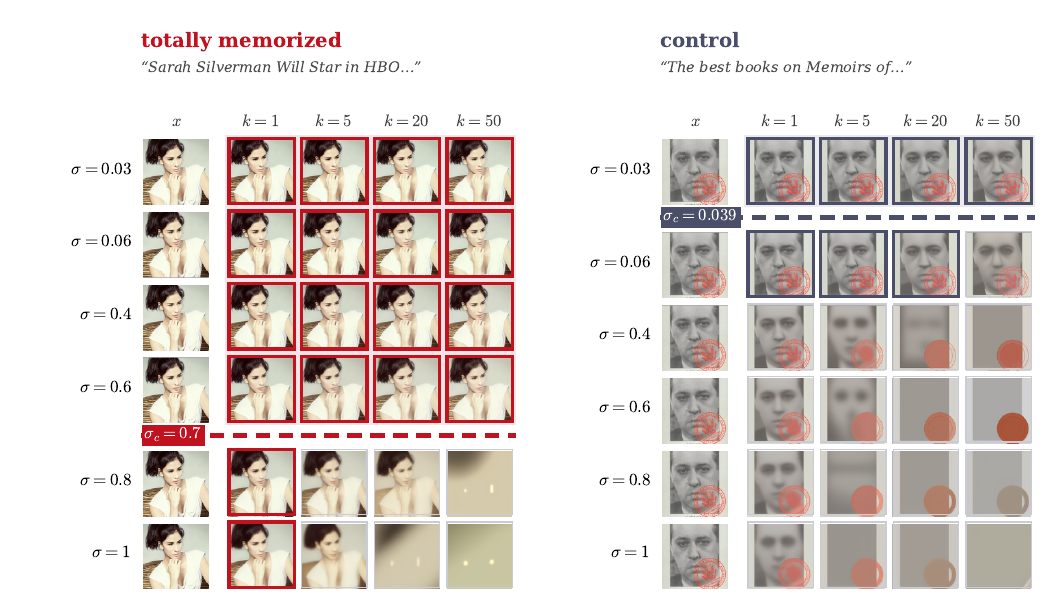}
    \end{subfigure}

   \begin{subfigure}[t]{\linewidth}
        \hspace*{0.1\linewidth}%
        \includegraphics[width=0.9\linewidth]
            {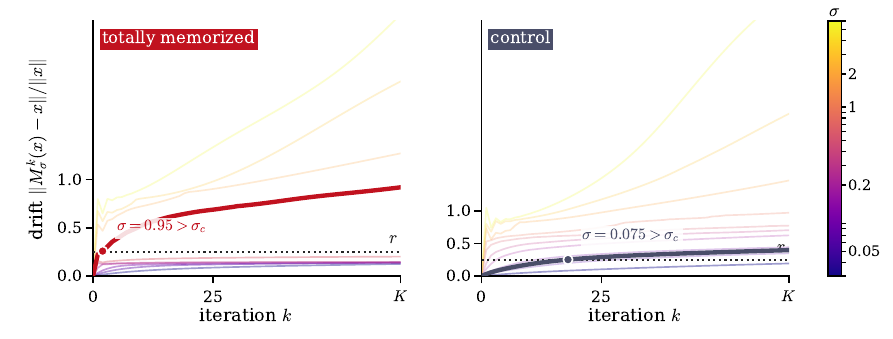}
    \end{subfigure}

    \caption{\textbf{Trajectory and critical scale MV vs control.} \textit{Top:} critical scale on for a totally memorized caption--image pair (\textit{left}) and a control pair (\textit{right}), each run with its own caption at guidance scale $1$. Each row iterates
$\M$ from the image $x$ at one noise scale $\sigma$; a framed iterate $\M^{k}(x)$ is still within
$r\|x\|$ of $x$. At small $\sigma$ the model keeps returning $x$, which therefore sits in a basin of
its own; at large $\sigma$ the orbit drifts away towards content shared with other data. The critical
scale $\sigc$ (dashed colored line) separates the two regimes: the memorized image withstands far more noise than the control before it escapes. \textit{Bottom:} the drift $\|\M^{k}(x)-x\|/\|x\|$ along the orbit, one curve per scale coloured by $\sigma$; the thick curve is a scale above $\sigc$ and the dot marks where it crosses $r$.}
    \label{fig:traj-mv-control}
\end{figure*}
\begin{figure*}[t]
    \centering
    \begin{subfigure}[t]{0.98\linewidth}
        \centering
        \includegraphics[width=\linewidth]
            {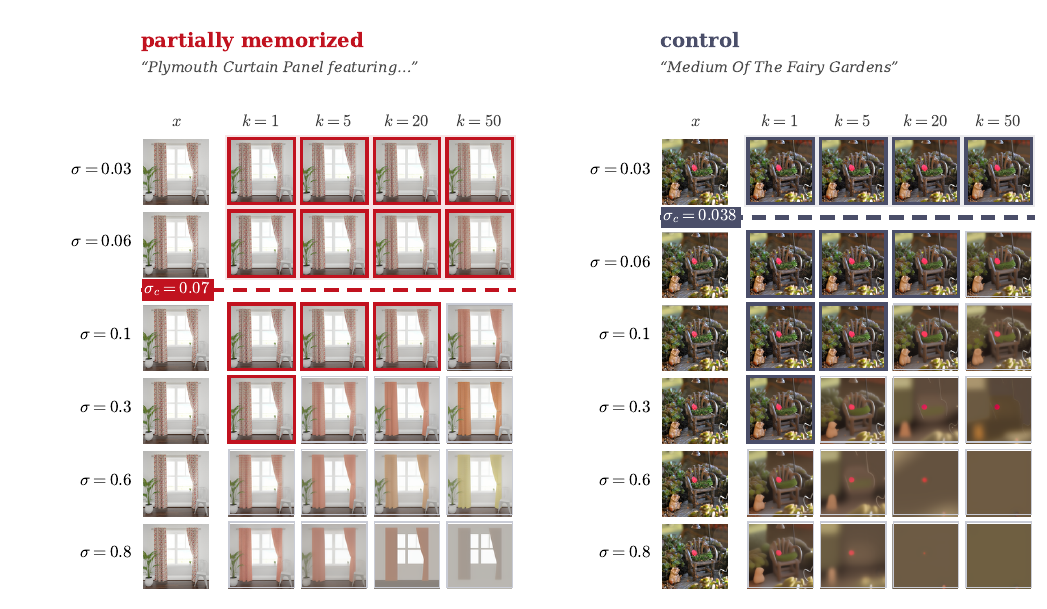}
    \end{subfigure}

   \begin{subfigure}[t]{\linewidth}
        \hspace*{0.1\linewidth}%
        \includegraphics[width=0.9\linewidth]
            {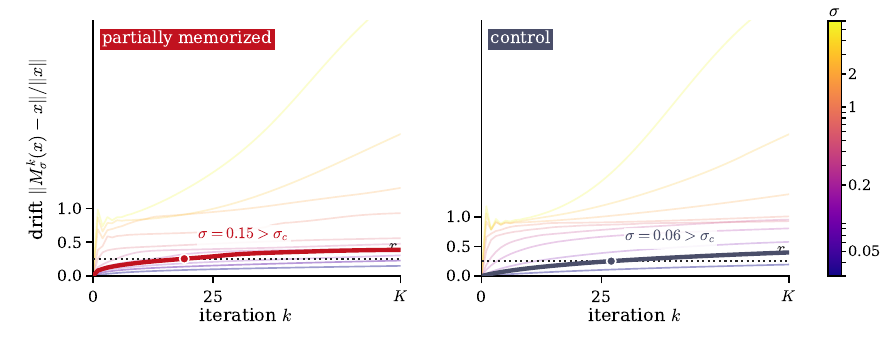}
    \end{subfigure}

    \caption{\textbf{Trajectory and critical scale TV vs control.} \textit{Top:} critical scale on for a partially memorized caption--image pair (\textit{left}) and a control pair (\textit{right}), each run with its own caption at guidance scale $1$. Each row iterates
$\M$ from the image $x$ at one noise scale $\sigma$; a framed iterate $\M^{k}(x)$ is still within
$r\|x\|$ of $x$. At small $\sigma$ the model keeps returning $x$, which therefore sits in a basin of
its own; at large $\sigma$ the orbit drifts away towards content shared with other data. The critical
scale $\sigc$ (dashed colored line) separates the two regimes: the memorized image withstands far more noise than the control before it escapes. \textit{Bottom:} the drift $\|\M^{k}(x)-x\|/\|x\|$ along the orbit, one curve per scale coloured by $\sigma$; the thick curve is a scale above $\sigc$ and the dot marks where it crosses $r$.}
    \label{fig:traj-tv-control}
\end{figure*}

\begin{figure*}[t]
    \centering
    \begin{subfigure}[t]{0.98\linewidth}
        \centering
        \includegraphics[width=\linewidth]
            {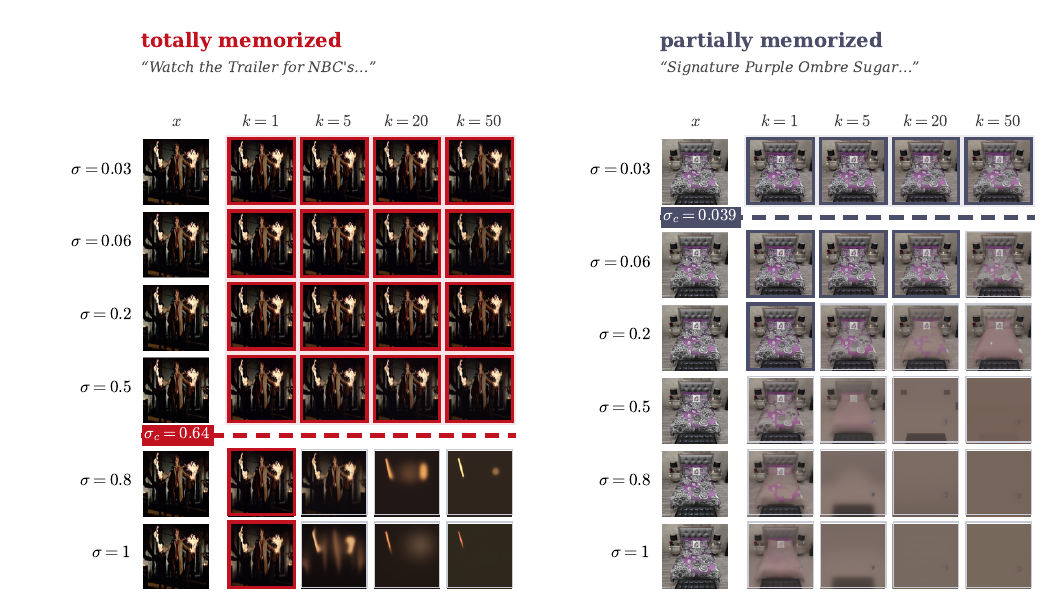}
    \end{subfigure}

   \begin{subfigure}[t]{\linewidth}
        \hspace*{0.1\linewidth}%
        \includegraphics[width=0.9\linewidth]
            {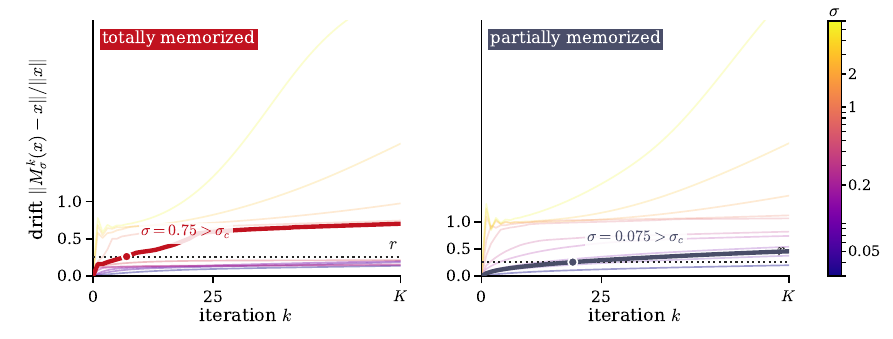}
    \end{subfigure}

    \caption{\textbf{Trajectory and critical scale MV vs TV.} \textit{Top:} critical scale on for a totally memorized caption--image pair (\textit{left}) and a partially memorized pair (\textit{right}), each run with its own caption at guidance scale $1$. Each row iterates
$\M$ from the image $x$ at one noise scale $\sigma$; a framed iterate $\M^{k}(x)$ is still within
$r\|x\|$ of $x$. At small $\sigma$ the model keeps returning $x$, which therefore sits in a basin of
its own; at large $\sigma$ the orbit drifts away towards content shared with other data. The critical
scale $\sigc$ (dashed colored line) separates the two regimes: the memorized image withstands far more noise than the control before it escapes. \textit{Bottom:} the drift $\|\M^{k}(x)-x\|/\|x\|$ along the orbit, one curve per scale coloured by $\sigma$; the thick curve is a scale above $\sigc$ and the dot marks where it crosses $r$.}
    \label{fig:traj-mv-tv}
\end{figure*}

\begin{figure*}[t]
    \centering
    \begin{subfigure}[t]{0.98\linewidth}
        \centering
        \includegraphics[width=\linewidth]
            {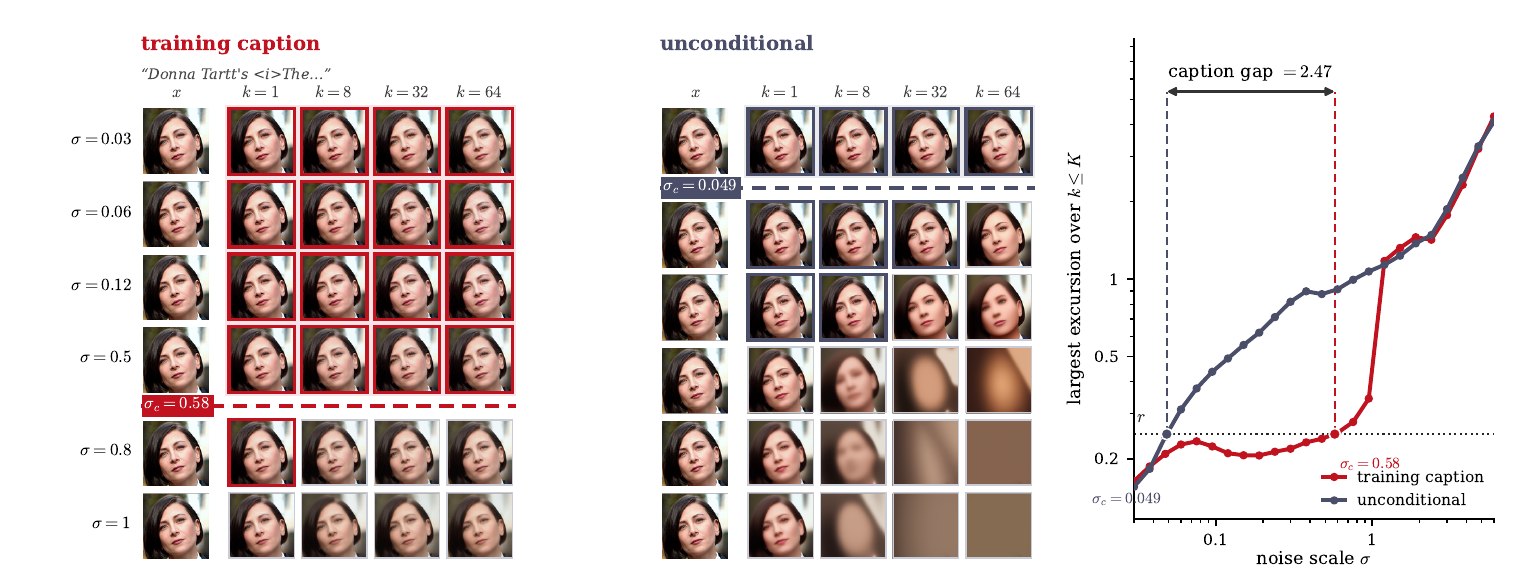}
    \end{subfigure}

  \begin{subfigure}[t]{0.98\linewidth}
        \centering
        \includegraphics[width=\linewidth]
            {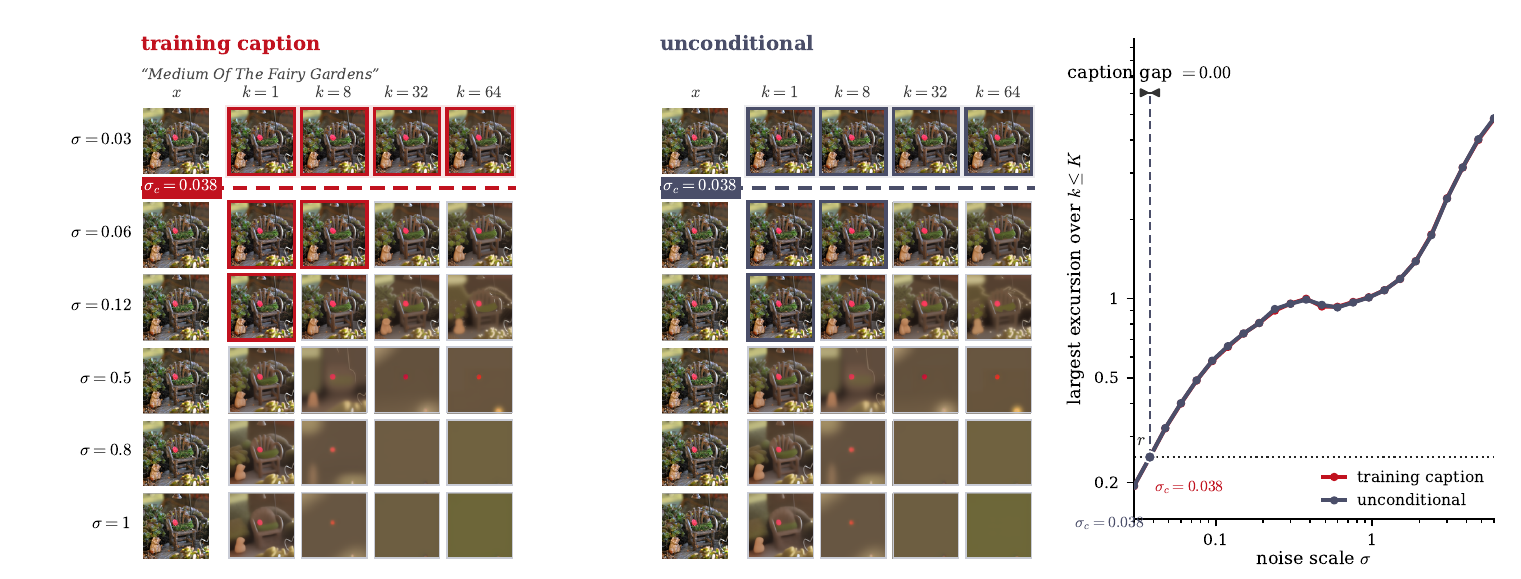}
    \end{subfigure}

    \caption{\textbf{$\boldsymbol{\Delta\log\sigma_c}$ for MV and control images.} The top row shows an MV image--caption pair and the bottom row a control. The left and center panels
show fixed-scale trajectories across noise levels under training-caption and unconditional
conditioning, respectively. The right panels plot
$\max_{k\leq K}||M_\sigma^k(x)-x||/||x||$ for both branches. Their crossings with the dotted
threshold $r$ define the conditional and unconditional critical scales, whose log difference is
$\Delta\log\sigma_c$. The MV pair has a large gap because its caption substantially increases
retention, whereas the control has nearly equal critical scales and a gap close to zero.}
    \label{fig:caption-gap}
\end{figure*}

\begin{figure*}[t]
    \centering
    \begin{subfigure}[t]{0.98\linewidth}
        \centering
        \includegraphics[width=\linewidth]
            {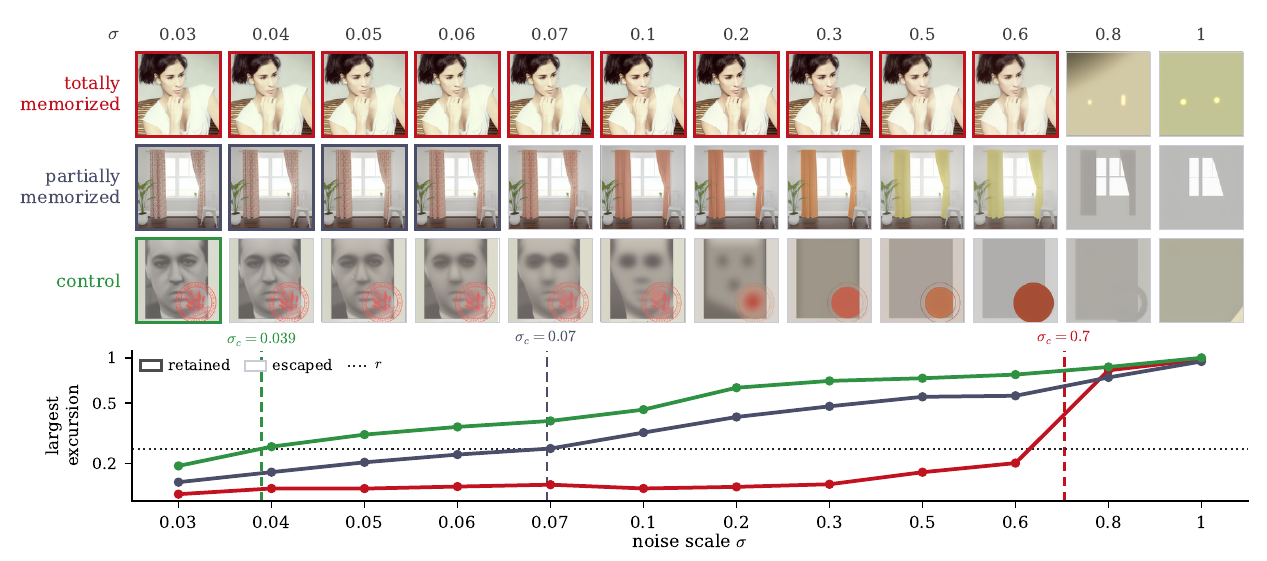}
    \end{subfigure}

  \begin{subfigure}[t]{0.98\linewidth}
        \centering
        \includegraphics[width=\linewidth]
            {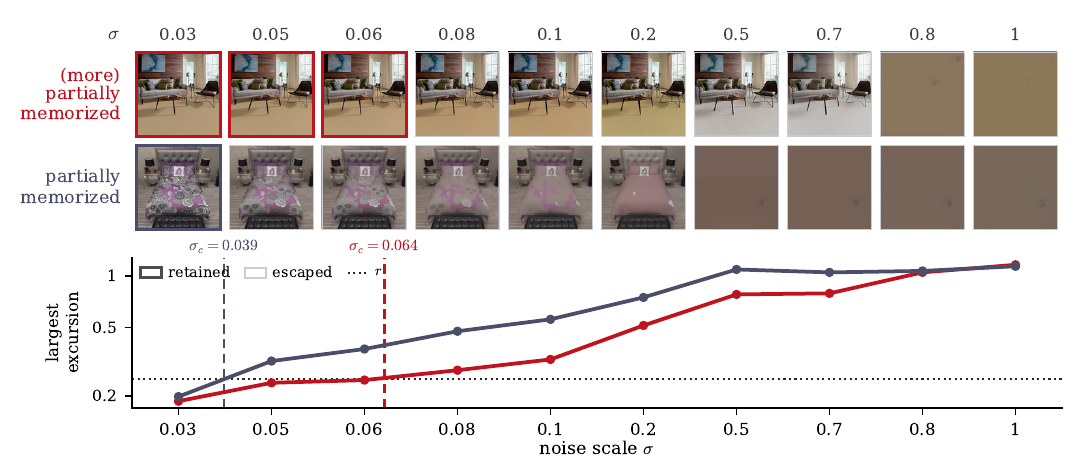}
    \end{subfigure}

    \caption{\textbf{Fixed-scale trajectories across noise levels.}
Each comparison shows the iterate $\M^K(x)$ across noise scales (top) and the maximum
normalized displacement $\max_{k\leq K}|\M^k(x)-x|/|x|$ as a function of $\sigma$
(bottom). The dotted line marks the escape radius $r$; its crossing determines $\sigma_c$.
The upper panel compares an MV, a TV, and a control image, ordered from longest to shortest
retention. The lower panel compares two TV images with different amounts of memorized content. The image with more memorized content is retained longer.}
    \label{fig:sweep}
\end{figure*}

\end{document}